\documentclass[11pt]{article}
\usepackage[margin=1in]{geometry}
\usepackage[hidelinks]{hyperref}
\usepackage{amsmath}
\usepackage[capitalize,noabbrev]{cleveref}

\usepackage{graphicx} 
\PassOptionsToPackage{prologue,dvipsnames}{xcolor}
\usepackage[dvipsnames]{xcolor}
\usepackage{url,mathtools}
\usepackage{caption}
\usepackage{subcaption}
\usepackage[innercaption]{sidecap}
\mathtoolsset{showonlyrefs}
\usepackage{setspace}
\sidecaptionvpos{figure}{t}

\usepackage{lipsum}
\usepackage{amsfonts}
\usepackage{graphicx}
\usepackage{epstopdf}
\usepackage{algorithmic}
\ifpdf
  \DeclareGraphicsExtensions{.eps,.pdf,.png,.jpg}
\else
  \DeclareGraphicsExtensions{.eps}
\fi

\usepackage{amsthm}   
\theoremstyle{plain}
\newtheorem{theorem}{Theorem}[section]
\newtheorem{lemma}[theorem]{Lemma}
\newtheorem{proposition}[theorem]{Proposition}
\newtheorem{corollary}[theorem]{Corollary}

\theoremstyle{definition}
\newtheorem{definition}[theorem]{Definition}

\theoremstyle{remark}
\newtheorem{remark}[theorem]{Remark}

\usepackage{amssymb}
\usepackage{tikz}
\usepackage{wrapfig}
\usepackage{enumitem}
\usepackage{subcaption}
\usepackage{mathrsfs}

\usepackage{bm}
\usepackage{bbm}
\usepackage{mathtools}
\usepackage{booktabs}
\usepackage{xspace}
\usepackage{tcolorbox}
\usepackage{xcolor}
\usepackage{comment}

\usepackage{mathalfa}

\newcommand{\dutchp}{\text{\usefont{U}{dutchcal}{m}{n}p}}

\usepackage{tikz,pgf}
\usepackage{scalerel}
\usepackage{pict2e}
\usepackage{tkz-euclide}
\usetikzlibrary{calc}
\usetikzlibrary{shapes.geometric, fit}

\usetikzlibrary{arrows,positioning}
\usepackage{xcolor}

\newtcolorbox[auto counter, number within=section]{definitionbox}[2][]{colframe=blue!40!black, colback=blue!10!white, coltitle=black, title=Definition~\thetcbcounter: #2,#1}

\newcommand{\wick}[1]{\left \langle #1 \right \rangle}

\newcommand{\abs}[1]{\left| {#1} \right|}
\newcommand{\norm}[1]{\left\| {#1} \right\|}

\renewcommand{\Re}{\mathrm{Re}\,}

\let\emptyset\varnothing

\newcommand{\Id}{\m{I}}

\newcommand{\tr}{\operatorname{tr}}

\renewcommand{\phi}{\varphi}
\DeclareMathAlphabet{\dutchcal}{U}{dutchcal}{m}{n}
\SetMathAlphabet{\dutchcal}{bold}{U}{dutchcal}{b}{n}
\DeclareMathAlphabet{\dutchbcal} {U}{dutchcal}{b}{n}

\newcommand{\E}{\mathbb{E}}
\newcommand{\prob}{\mathbb{P}}
\renewcommand{\Pr}{\prob}

\newcommand{\Var}{\operatorname{Var}}

\DeclareDocumentCommand{\Asto} {o} {
\IfNoValueTF {#1}
 {\overset{\operatorname{a.s.}}{\longrightarrow}}
 {
 \xrightarrow[ #1 \to \infty]{\operatorname{a.s.} }
 }
}

\DeclareDocumentCommand{\Probto} {o} {
\IfNoValueTF {#1}
 {\overset{\prob}{\longrightarrow}}
 {
 \xrightarrow[ #1 \to \infty]{\prob}
 }
}

\newcommand{\V}{\m{V}} 
\newcommand{\x}{\m{x}} 
\newcommand{\K}{\m{K}} 

\newcommand{\wc}{\boldsymbol{\omega}}
\newcommand{\W}{\m{W}}

\newcommand{\z}{\m{z}}
\newcommand{\D}{\m{D}}
\newcommand{\pwrco}{\alpha}
\newcommand{\B}{\m{B}}
\newcommand{\y}{\m{y}}

\newcommand{\col}{\text{col}}

\newcommand{\p}{\dutchp}

\newcommand{\cov}{\m{H}}

\newcommand{\ti}{\tilde{\m{i}}}
\newcommand{\ts}{\tilde{\m{s}}}

\newtheorem{exmp}{Example}

\definecolor{blue}{RGB}{68,119,170}
\definecolor{red}{RGB}{238,102,119}
\definecolor{green}{RGB}{0,158,115}

\definecolor{c1}{RGB}{174,118,163}

\definecolor{c2}{RGB}{136,46,163}

\definecolor{c3}{RGB}{136,46,114}

\definecolor{c4}{RGB}{25,101,176}

\definecolor{c5}{RGB}{82,137,199}

\definecolor{c6}{RGB}{123,175,222}

\definecolor{c7}{RGB}{78,178,101}

\definecolor{c8}{RGB}{144,201,135}

\definecolor{c9}{RGB}{246,193,65}

\definecolor{c10}{RGB}{241,147,45}

\definecolor{c11}{RGB}{238,102,119}

\definecolor{c12}{RGB}{220,5,12}

\newcommand{\m}[1]{\mathbf{#1}}

\newcommand{\R}{\mathbb{R}}
\newcommand{\C}{\mathbb{C}}

\newcommand{\s}{\mathcal{S}}

\newcommand{\N}{\mathbb{N}}

\newcommand{\bt}{\mathbf{t}}
\newcommand{\bT}{\mathbf{T}}
\newcommand{\bg}{\mathbf{g}}
\newcommand{\bz}{\mathbf{z}}
\newcommand{\bZ}{\mathbf{Z}}
\newcommand{\bD}{\mathbf{D}}
\newcommand{\bB}{\mathbf{B}}

\newcommand{\dif}{\mathrm{d}}
 
\newcommand\inner[2]{ \left \langle #1, #2 \right \rangle }

\newcommand{\email}[1]{\texttt{#1}}
\title{Power-Law Spectrum of the Random Feature Model}
\author{%
  Elliot Paquette\thanks{Department of Mathematics and Statistics, McGill University (\email{elliot.paquette@mcgill.ca})}
  \and Ke Liang Xiao\thanks{Department of Mathematics and Statistics, McGill University (\email{keliang.xiao@mail.mcgill.ca})}
  \and Yizhe Zhu\thanks{Department of Mathematics, University of Southern California (\email{yizhezhu@usc.edu})}}
\date{}

\usepackage{amsopn}

\newif\ifrebuttal
\rebuttalfalse 

\ifrebuttal
\newcommand{\reb}[1]{{\color{blue}#1}}
\else
\newcommand{\reb}[1]{#1}
\fi

\ifrebuttal
\newcommand{\ep}[1]{{\color{blue}#1}}
\else
\newcommand{\ep}[1]{#1}
\fi

\hypersetup{pdftitle={Power-Law Spectrum of the Random Feature Model},
            pdfauthor={E. Paquette, K. Xiao, and Y. Zhu}}

\begin{document}

\maketitle

\begin{abstract}
Scaling laws for neural networks, in which the loss decays as a power-law in the number of parameters, data, and compute, depend fundamentally on the spectral structure of the data covariance, with power-law eigenvalue decay appearing ubiquitously in vision and language tasks.
A central question is whether this spectral structure is preserved or destroyed when data passes through the basic building block of a neural network: a random linear projection followed by a nonlinear activation.
We study this question for the random feature model: given data $\x \sim N(0,\cov)\in \R^v$ where $\cov$ has $\pwrco$-power-law spectrum ($\lambda_j(\cov) \asymp j^{-\pwrco}$, $\pwrco > 1$), a Gaussian sketch matrix $\W\in \R^{v\times d}$, and an entrywise polynomial $f$ of degree $\p$, we characterize the eigenvalues of the population random-feature covariance $\E_{\x}[\frac{1}{d}f(\W^\top \x)^{\otimes 2}]$.
We prove matching upper and lower bounds: for all $1 \leq j \leq c_1 d \log^{-(\p+1)}(d)$, the $j$-th eigenvalue is of order $\left(\log^{\p-1}(j+1)/j\right)^{\pwrco}$. For $ c_1 d \log^{-(\p+1)}(d)\leq j\leq d$,  the $j$-th eigenvalue is of order $j^{-\alpha}$ up to a polylog factor.
That is, the power-law exponent $\pwrco$ is inherited exactly from the input covariance, modified only by a logarithmic correction that depends on the \ep{degree $\p$ of the activation}.
The proof combines a dyadic head-tail decomposition with Wick chaos expansions for higher-order monomials and random matrix concentration inequalities\ep{; a Gaussian integration-by-parts argument then extends the result from monomials to polynomials}.
\end{abstract}


\noindent\textbf{Keywords.} random feature model, random matrix theory, scaling law

\section{Introduction}

When training large neural networks, the loss typically scales as a power-law in the number of parameters, amount of data, and total compute \cite{kaplan2020scaling,hoffmann2022chinchilla}.
The exponents of these scaling laws are of central importance: they determine how much additional data or compute is needed to achieve a given reduction in loss, and they govern compute-optimal training strategies.
A growing body of theoretical work has identified the spectral decay of the data covariance as \emph{one} component controlling these exponents. Specifically, power-law eigenvalue decay of the form $\lambda_j \asymp j^{-\pwrco}$ appears as a key ingredient in scaling law models \cite{bahri2021explaining,maloney2022solvable,paquette20244+,reventos2025understanding}.
Of course, the data covariance spectrum is not the only mechanism at play: other factors such as the structure of the learning algorithm, the target function complexity, and feature learning dynamics also shape scaling behavior \cite{liu2026universal,ren2025emergence,benarous2025learning}\ep{; see also \cite{defilippis2026scaling}}.

Nonetheless, power-law spectral structure is ubiquitous. It appears in the covariance of natural images (see \cref{fig:cifar10}), language model representations, and intermediate features of neural networks.
Understanding what controls these spectral exponents is therefore a path toward understanding, and eventually improving, scaling laws.
\ep{Power-law structure brings its own difficulties: in phase retrieval with power-law data, for instance, gradient flow converges slowly \cite{braun2026fast}. But once the structure is known to be there, it can be exploited: accelerated algorithms designed for power-law spectra provably achieve better scaling-law exponents \cite{ferbach2025dana,ferbach2026logarithmic,bordelon2026optimal,varre2022accelerated,liu2018accelerating,bach2025ztransform}.}

\begin{figure}[t]
  \centering
  \includegraphics[width=\textwidth]{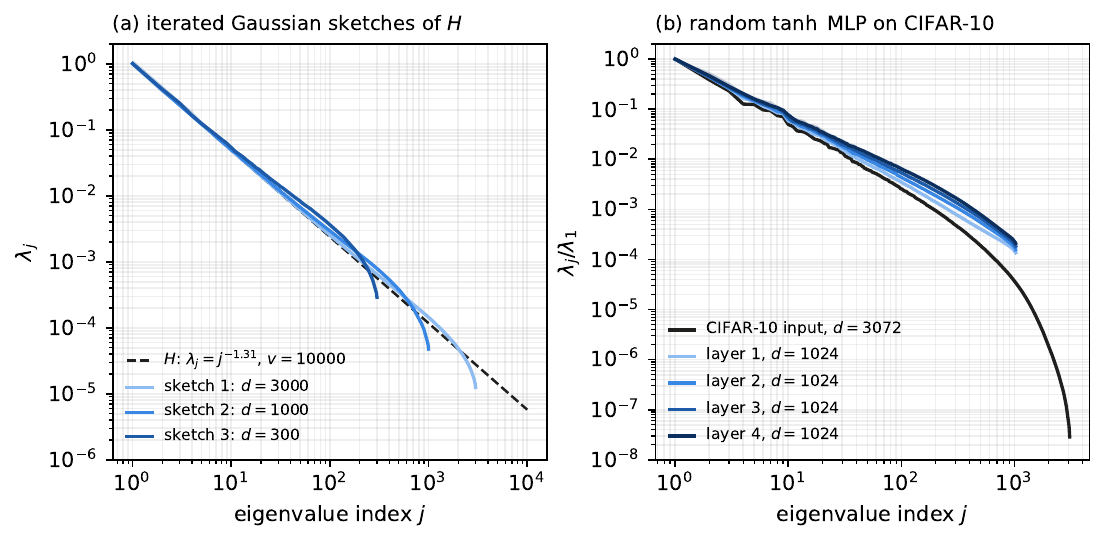}
  \caption{\ep{\textbf{Power-law spectra under random sketches and random layers.} \textbf{(a)} Starting from $\mathbf{H} = \mathrm{diag}(j^{-1.31})$ with $v = 10^4$ (dashed), three successive Gaussian sketches $10^4 \to 3000 \to 1000 \to 300$. \textbf{(b)} CIFAR-10 (black) and the four hidden layers of a random $\tanh$ MLP of width $1024$, with pre-activations of mean square $1$; covariance spectra (centered), each divided by its largest eigenvalue. Every layer is close to a power law, and the decay flattens slowly with depth.}}
  \label{fig:cifar10}
\end{figure}

The most basic operation in a neural network is to take input data, multiply it by a weight matrix, and apply a nonlinear activation function. This is the fundamental multilayer perceptron (MLP) layer.
The \emph{power-law random features} (PLRF) model formalizes this construction: data $\x\in \R^v$ with power-law covariance $\cov$ is projected through a random sketch $\W\in \R^{v\times d}$ and passed through a nonlinearity $f$ \cite{maloney2022solvable,paquette20244+,rahimi2008random}.
In the linear case ($f$ is the identity), the spectrum of the sketched covariance $\frac{1}{d}\W^\top \cov \W$ is well understood: it inherits the power-law decay from $\cov$ \cite{paquette20244+,lin2024scaling}.
This raises a fundamental question: \emph{does the power-law spectral structure survive the nonlinearity?}
If the exponent $\pwrco$ persists through each layer, the \ep{power-law} structure is maintained as data propagates through the network, \ep{so optimizers designed to exploit it remain applicable at every layer, and the structure can be used to understand scaling laws beyond the linear setting}.

We answer this question affirmatively for \ep{polynomial} activations.
Our main result (\cref{theorem:MAIN}) shows that for \ep{a polynomial activation $f$ of degree $\p$ (for example the monomial $f(y) = y^{\p}$)}, the eigenvalues of the population random-feature covariance satisfy
\[
\lambda_j\!\left(\E_{\x}\!\left[\tfrac{1}{d}f(\W^\top \x)^{\otimes 2}\right]\right) \asymp \left(\frac{\log^{\p-1}(j+1)}{j}\right)^{\!\pwrco},
\]
with matching upper and lower bounds for all $1 \leq j \leq c_1 d \log^{-(\p+1)}(d)$, holding with probability $1 - O(\log^{-4}(d))$ \ep{ whenever $v \geq d^{\frac{1}{\p}+\epsilon}$ ($v\geq d$ if $\p=1$)}.
The power-law exponent $\pwrco$ is inherited exactly from the input covariance.
The only modification is a logarithmic correction $\log^{\p-1}(j+1)$ that depends on the \ep{degree}: for $\p = 1$ (the linear case) there is no correction, recovering the known result \cite{lin2024scaling,bartlett2020benign,paquette20244+}, while for $\p > 1$ the correction is mild and does not change the polynomial rate of decay. For $ c_1 d \log^{-(\p+1)}(d)\leq j\leq d$, \cref{theorem:MAIN} also shows that the $j$-th eigenvalue is of order $j^{-\alpha}$ up to a polylog factor.

\ep{Numerical experiments (Section~\ref{sec:discussion}) illustrate \cref{theorem:MAIN} and suggest that power-law decay persists beyond its hypotheses, for activations such as ReLU and for non-Gaussian data.}
This spectral characterization provides a building block for rigorous scaling law theory in nonlinear random feature models.

   \paragraph{\textbf{Related work}}  Power-law spectral decay has been used as a principled model for describing effective dimension in modern high-dimensional learning problems, and for predicting how risk scales with sample size, parameters, and compute.
In linear models, this perspective is closely tied to benign overfitting and scaling-law phenomena; see, e.g.,
\cite{bartlett2020benign,lin2024scaling,tsigler2023benign}.
Our results can be viewed as identifying how such a power-law structure propagates through \emph{random sketching} and nonlinear feature maps, by characterizing the eigenvalue decay of the population random-feature covariance.

There is also a substantial literature on kernel methods and kernel random matrices beyond the isotropic setting.
Classical work analyzes the spectrum of kernel random matrices under fairly general distributions \cite{el2010spectrum}, and more recent works study approximation and universality phenomena for empirical kernel matrices and kernel regression in high dimensions, including anisotropic regimes; see, e.g., \cite{donhauser2021rotational,bartlett2021deep,tsigler2023benign,liang2020multiple,wang2021deformed,pandit2025universality,kaushik2025general,hu2024asymptotics}.
Many kernel-regression analyses in power-law settings can be interpreted as assuming (explicitly or implicitly) a power-law spectrum for a feature representation $\varphi(\x)$ together with sufficient concentration of $\varphi(\x)$ \cite{defilippis2024dimension,schroder2024deterministic,schroder2024asymptotics,cheng2024dimension}.
In contrast, we start from an explicit random-feature construction and derive the resulting spectral decay from first principles.

Finally, there is a connection to asymptotic geometric analysis: random projections of ellipsoids (and the geometry induced by heavy spectral decay) have been studied from a geometric and information-based viewpoint, e.g. \cite{hinrichs2021random}. Our approach can be viewed as a random matrix counterpart in this direction: the power-law structure is encoded by an ellipsoidal covariance $\cov$, and we track how random sketching and nonlinear feature maps reshape its spectrum, requiring a blend of random matrix concentration \cite{jirak2025concentration,koltchinskii2015bounding} and Wick chaos expansion techniques  \cite{janson1997gaussian} adapted to the power-law setting. 

 \paragraph{\textbf{Comparison with } \cite{wortsman2025kernel}}
The recent work~\cite{wortsman2025kernel} studies \emph{kernel ridge regression} (KRR)  under anisotropic Gaussian covariates whose covariance matrix exhibits a power-law spectrum.
A key object in their analysis is the associated \emph{kernel integral (population) operator} $T$,
defined for $\nu= N(0,\boldsymbol{\Sigma})$ by
\[
(Tf)(\x)
\;=\;
\mathbb E_{\m{Y}\sim \nu}\!\bigl[k(\x,\m{Y})\,f(\m{Y})\bigr]
\;=\;
\int_{\mathbb R^{v}} k(\x,\m{y})\,f(\m{y})\,d\nu(\m{y}),
\qquad f\in L^2(\nu).
\]
This operator depends only on the kernel $k$ and the data distribution $\nu$ (hence contains no finite-sample randomness),
and generalization bounds for KRR can be expressed in terms of the spectral decay of $T$.
When the kernel function is $k(\x,\m{y})=\langle \x,\m{y}\rangle^{p}$, the spectrum of $T$ is characterized in \cite[Corollary 2]{wortsman2025kernel}. \reb{However, because their analysis treats the kernel as fixed, with neither feature learning nor a model width, it does not describe how the spectrum of the data evolves as it passes through the layers of a neural net.
} \ep{Apart from \cite{wortsman2025kernel}, the high-dimensional behavior of a nonlinear model under power-law data has also been analyzed in \cite{braun2026fast}, which studies the gradient-flow dynamics of phase retrieval.}

Our contribution is complementary to that of \cite{wortsman2025kernel}: we analyze the \emph{random feature model} obtained by applying a \ep{polynomial} nonlinearity to a random sketch,
and we characterize the eigenvalue decay of the resulting \emph{finite-dimensional} $d\times d$ random feature matrix.
Compared to the operator-level viewpoint in~\cite{wortsman2025kernel}, our results directly quantify how a power-law spectrum in the population covariance
propagates through random sketching and the feature nonlinearity at finite width. \reb{
Moreover, we show that when the random sketch embeds the data into a higher feature dimension $d$, the nonlinearity of the activation may even extend the power-law decay of the data's spectrum to span the full feature dimension.
}
We expect this spectral characterization to be a useful input for studying the test error of random-feature ridge regression under power-law data.

\paragraph{\textbf{Notation}} 
Given two sequences $\{a_n\}$ and $\{b_n\}$, we will write $a_n\lesssim b_n$ if $a_n\leq \reb{C} b_n$ for some constant $\reb{C}$ that is \reb{dimension ($v$ and $d$) independent}. 
We will write $a_n\asymp b_n$ if $a_n \lesssim b_n$ and $b_n \lesssim a_n$. If $k\in \N \coloneqq \{1, 2,\cdots \}$, we may often write $[k]=\{1,\dots,k\}$. For two vectors $\m{u}$ and $\m{v}$, we will write either $\m{u}\otimes \m{v}$ or $\m{u}\m{v}^\top$ to be the outer-product \reb{or equivalently} the $2$-tensor. We will write $\m{u} \odot \m{v}$ to be the standard \reb{Hadamard} product. If $\boldsymbol{\Sigma} \in \R^{n\times n}$, then $\norm{\boldsymbol{\Sigma}}$ will refer to its operator norm and $\norm{\boldsymbol{\Sigma}}_F$, its Frobenius norm. We denote the effective rank to be $r(\boldsymbol{\Sigma}) = \tr(\boldsymbol{\Sigma})/\norm{\boldsymbol{\Sigma}}$.

\section{Main results}

We start with a formal definition of the power-law spectral distribution of a \reb{symmetric} matrix with eigenvalues arranged in non-increasing order. 
\begin{definition}
    We say that a \reb{symmetric positive semidefinite} matrix $\cov\in \R^{v \times v}$ has a power-law spectrum with exponent $\pwrco>1$ if for all $1\leq j \leq v$,
   $  \lambda_j(\cov)  \asymp j^{-\pwrco}$.

\end{definition}
\reb{
As we will see, a non-linear activation function contributes a polylog factor to the eigenvalues. We therefore refer to the resulting spectrum as exhibiting a weak power-law.
\begin{definition}
    We say that a \reb{symmetric positive semidefinite} matrix $\cov\in \R^{v \times v}$ has weak power-law spectrum with exponent $\pwrco>1$, if there exists a non-decreasing function $f:\R^+ \to \R^+$ such that 
    for all $1\leq j \leq v$,
   $  \lambda_j(\cov)  \asymp \left(\frac{f(j)}{j} \right)^{\pwrco}$.
\end{definition}
}

The following theorem describes how the power-law exponent $\pwrco$ propagates through random sketch and \reb{polynomial} activation. 
\begin{theorem}\label{theorem:MAIN}
    Let $\pwrco>1$ and $\x \sim N(0,\cov)$ such that $\cov \in \R^{v \times v}$ has power-law spectrum with exponent $\pwrco$. Consider a sketch matrix $\W \in \R^{v \times d}$ with i.i.d. standard Gaussian entries, along with a \reb{polynomial activation function $f$ of degree $\p \in \N$}. \ep{Assume that $v \geq d$ if $\p = 1$, and that $v \geq d^{\frac{1}{\p}+\epsilon}$ for some fixed $\epsilon>0$ if $\p \geq 2$. Then, for $d$ sufficiently large, there exist constants $c_1,c_2,c_3,c_4>0$, independent of $v$ and $d$ and depending only on $\pwrco$, $\p$, $\epsilon$, the coefficients of $f$ and the implicit constants in $\lambda_j(\cov) \asymp j^{-\pwrco}$,} so that with probability $1-O\left( \log^{-4}(d)\right)$, for all  $1\leq j \leq c_1 d \log^{-(\p+1)}(d)$,
    \begin{equation}\label{eq:main_1}
        c_2 \left(\frac{\log^{\p-1}(j+1)}{j} \right)^{\pwrco} \leq \lambda_j\left(\E_{\x} \left[ \frac{1}{d}f(\W^\top \x)^{\otimes 2}\right] \right) \leq c_3\left(\frac{\log^{\p-1}(j+1)}{j} \right)^{\pwrco}.
    \end{equation}
Moreover, when \reb{$\p\geq 2$}, if $ c_1 d \log^{-(\p+1)}(d) <j \leq d$, then 
\begin{equation}\label{eq:main_min_tail}
    c_2 \left( \frac{\log^{-2}(j+1)}{j}\right)^\pwrco \leq \lambda_j\left(\E_{\x} \left[ \frac{1}{d}f(\W^\top \x)^{\otimes 2}\right] \right) \leq c_4 \left(  \frac{ \log^{2\p}(j+1)}{j}\right)^{\pwrco}.
\end{equation}
\end{theorem}

\reb{
\begin{remark}
    Theorem \ref{theorem:MAIN} naturally extends to the covariance matrix of $\frac{1}{\sqrt{d}} f(\W^\top \x)$, since the second-moment matrix differs from the covariance only by the rank-$1$ perturbation \[\frac{1}{d}\E_\x\left[ f(\W^\top \x)\right]^{\otimes 2}.\] One may check using the law of large numbers that with high probability, this rank-$1$ matrix is $O(1)$ in $\ell_2$-norm. By Weyl's inequality, a rank-$1$ perturbation of a weak power-law spectrum is still a weak power-law with the same exponent. For the complete details, see Appendix \ref{Appendix:cov_vs_2nd}. As such, we will only focus on the second moment matrix in this paper. 
\end{remark}
}
\reb{
\begin{remark}\label{rem:dimension_assumption}
    The assumption that $v \geq d^{\frac{1}{\p}+\epsilon}$ is a notable consequence of polynomial activations. 
    Informally, such activations lift the power-law spectrum from $v$ dimensional space to a higher ${v \choose \ p} \approx v^{\p} $ dimensional space while preserving its exponent. This effectively extends the range of the power-law decay.
    \ep{The assumption on $v$ is needed only for the lower bounds; the upper bounds in \cref{theorem:MAIN} hold without it (see Remark~\ref{rem:upper_any_v}). When $\p = 1$, $v \geq d$ suffices; when $\p \geq 2$, the margin $\epsilon > 0$ cannot be removed entirely.}
\end{remark}}
\subsection{Discussion and numerical experiments}\label{sec:discussion}

We now present numerical experiments that illustrate \cref{theorem:MAIN} and explore the extent to which power-law preservation holds beyond the hypotheses of the theorem. \ep{The code for all of them, with a description of each computation and the commands that reproduce every figure, is available at \url{https://github.com/mlexpos/power-law-activations-are-preserved} \cite{paquette2026code}.}

\paragraph{\textbf{Illustrating the main theorem}} \ep{\Cref{fig:gaussian_monomials}(a) shows, for Gaussian data with $\pwrco = 1.31$ and three polynomial activations, the spectrum of $\K_\W = \E_\x[\frac1d f(\W^\top\x)^{\otimes 2}]$ together with its population spectrum, the limit of the spectrum of $\K_\W$ as $d\to\infty$ with $v$ fixed. Both are computed exactly, with no sampling of $\x$. The reference lines are pure power laws: at these sizes the logarithmic correction of \cref{theorem:MAIN} is only partly developed, and for $\p=3$ its asymptotic form is reached only for $j$ far beyond $10^4$ (Remark~\ref{rem:theory_curve} of the appendix).}

\paragraph{\textbf{Beyond \reb{polynomial} activations}}
\ep{\Cref{fig:gaussian_monomials}(b) shows two non-polynomial activations, ReLU and the Heaviside function. Both spectra decay like $j^{-\pwrco}$ over the head. As for the polynomials in panel (a), the tail lies above this decay at finite $d$ and comes down as $d$ grows. For the Heaviside function this lift is numerically consistent with a floor due to its high Hermite degrees, see Section \ref{sec:across_layers} of the appendix for further discussion. \cref{theorem:MAIN} covers only polynomials; for general activations each Hermite degree brings its own logarithmic correction, and controlling their sum is an important open direction.}

\begin{figure}[h]
  \centering
  \includegraphics[width=\textwidth]{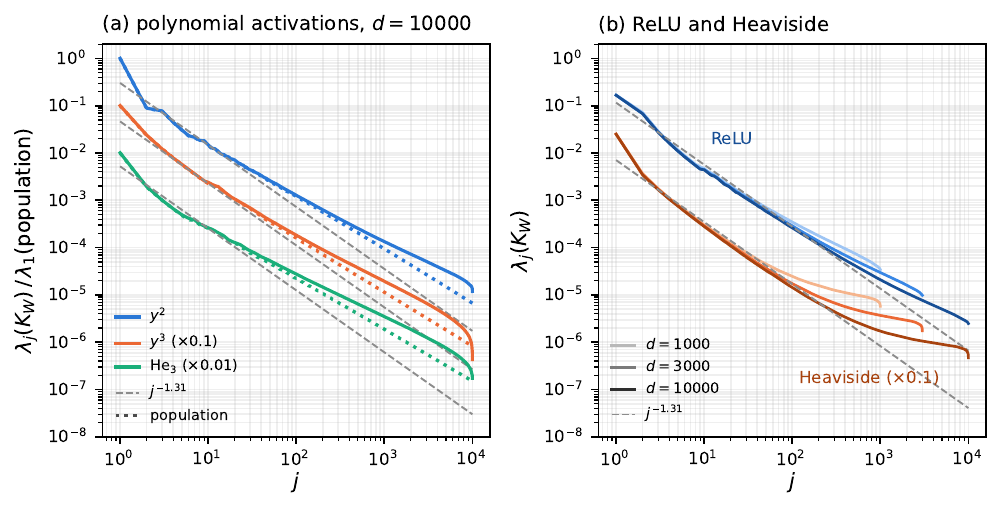}
  \caption{\ep{Spectra of $\K_\W$ for Gaussian data with $\lambda_j(\cov)\propto j^{-1.31}$, $\tr\cov=1$ and $v=2d$, computed exactly; median over 20 draws of $\W$. Dashed: $j^{-1.31}$, one line per curve. \textbf{(a)} Polynomial activations at $d=10^4$, divided by the top population eigenvalue, with the population spectrum dotted; curves are shifted vertically for legibility. The deviations from $j^{-1.31}$ are consistent with the logarithmic correction of \cref{theorem:MAIN}, which develops slowly in $j$; see \cref{fig:local_slope} of the appendix. \textbf{(b)} ReLU and Heaviside at $d=10^3$, $3\cdot10^3$ and $10^4$ (light to dark).}}
  \label{fig:gaussian_monomials}
\end{figure}

\paragraph{\textbf{Universality across data distributions}} \cref{theorem:MAIN} assumes Gaussian data. \ep{\Cref{fig:universality} keeps the covariance of $\x$ fixed and changes its distribution: in panel (a), $\x = \cov^{1/2}\boldsymbol{\xi}$ with i.i.d.\ standardized Gaussian, Rademacher or Student-$t_9$ coordinates; in panel (b), $\x$ is a centered CIFAR-10 image, compared with Gaussian data with the same covariance. In every case the spectra decay alike. For ReLU they agree to within a few percent. For $f(y) = y^3$ the entries of $\K_\W$ depend on the higher cumulants of $\x$, and the spectra differ by a bounded factor, largest at the head: the top eigenvalue for Rademacher data is $0.30$ times the Gaussian one, and on CIFAR-10 the spectrum lies up to $30\%$ above the Gaussian one.}
A related open question is universality in the sketch matrix $\W$: we expect the result to hold for random matrices $\W$ beyond the Gaussian case, provided the entries have matching first and second moments and sufficiently light tails.
Indeed, our \cref{thm:linear} for the linear case already establishes this under moment conditions on $\W$.

\begin{figure}[t]
  \centering
  \includegraphics[width=\textwidth]{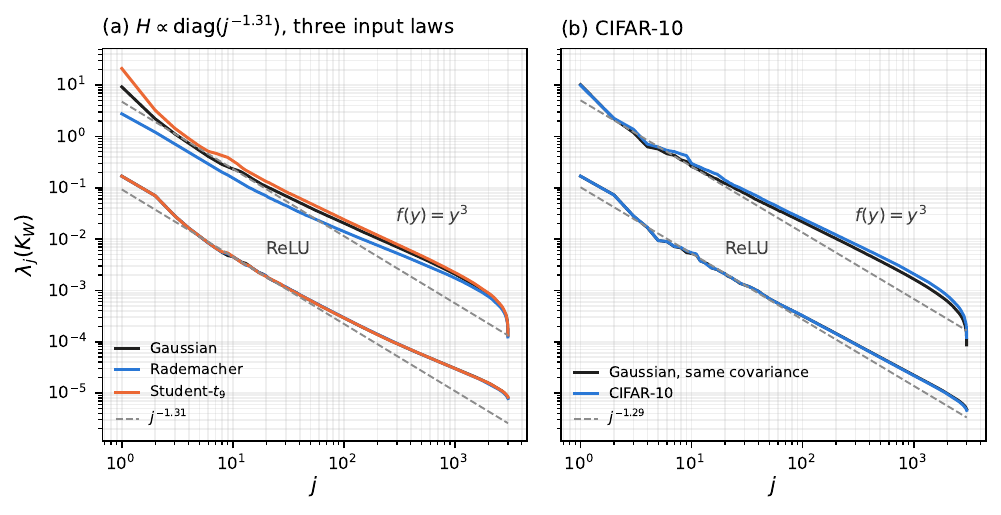}
  \caption{\ep{Spectra of $\K_\W$ for $f(y)=y^3$ (upper curves) and ReLU (lower curves), with the covariance of $\x$ fixed and its distribution varied; $d=3000$, median over 20 draws of $\W$, the same for every distribution. \textbf{(a)} $\x=\cov^{1/2}\boldsymbol{\xi}$ with $\lambda_j(\cov)\propto j^{-1.31}$, $\tr\cov=1$, $v=6000$ and i.i.d.\ standardized coordinates $\xi_k$; exact for $y^3$, Monte Carlo with $60{,}000$ samples for ReLU. \textbf{(b)} The $60{,}000$ CIFAR-10 images, centered and scaled so that $\tr\cov=1$ ($v=3072$), against $60{,}000$ samples of Gaussian data with the same covariance. Dashed: $j^{-1.31}$ in (a) and $j^{-1.29}$, the decay fitted to the CIFAR-10 covariance over $1\le j\le 100$, in (b), through the Gaussian curve at $j=10$.}}
  \label{fig:universality}
\end{figure}

\paragraph{\textbf{Persistence across layers}} \cref{theorem:MAIN} characterizes the spectrum after a single random projection and activation, i.e., one layer of an MLP.
\reb{In the appendix, we present numerical experiments} \ep{(Section \ref{sec:across_layers}) that follow the covariance spectrum of CIFAR-10 through the layers of an MLP, a VGG-style CNN and a ResNet-20, at random initialization and after training. Every layer has a spectrum close to a power law ($R^2\ge 0.97$ over $11\le j\le 200$), but the exponent drifts, in ways that depend on the architecture and on training. Random MLPs flatten it layer by layer, a random ResNet-20 keeps it between $-1.13$ and $-1.34$, and after training the last layer is close to or steeper than the input: $-1.40$ for ResNet-20, $-1.46$ for the CNN and $-1.53$ for the MLP, against $-1.35$. This drift is no larger than what the logarithmic corrections of \cref{theorem:MAIN} already produce in a single layer. For Gaussian data with $\pwrco=1.31$ and $f(y)=y^2$, $y^3$ or $\mathrm{He}_3(y)$, the population spectrum fitted over the same window has slope between $-1.06$ and $-1.17$, that is, $0.14$ to $0.25$ flatter than the input (see also \cref{fig:local_slope} of the appendix); a random ReLU or $\tanh$ layer flattens the fitted slope by at most $0.26$.}

\ep{
\paragraph{\textbf{Scope and limitations}} \cref{theorem:MAIN} assumes Gaussian data, a Gaussian sketch matrix, $\pwrco>1$, a polynomial activation and $v \geq d^{\frac{1}{\p}+\epsilon}$. The Gaussian assumptions enter through Wick calculus and Gaussian hypercontractivity. The condition $\pwrco>1$ keeps $\tr(\cov)$ bounded, which is used throughout to control the sub-leading terms; when $\pwrco<1$, $\tr(\cov)$ grows like $v^{1-\pwrco}$ and a different analysis is needed (see Section \ref{sec:alpha_lt_1} of the supplementary material for numerical experiments). For analytic activations, each Hermite degree contributes its own polylogarithmic correction, and controlling their sum is the main obstacle. Finally, consider the remaining range $v^{\p/(1+\epsilon\p)} < d \leq \binom{v+\p-1}{\p}$, in which, for $f(y)=y^\p$, the matrix $\E_\x[\frac{1}{d}f(\W^\top\x)^{\otimes 2}]$ still has full rank for generic $\W$. There the logarithmic correction in \cref{theorem:MAIN} cannot remain the same, since the number of index tuples in $[v]^\p$ with product at most $X$ is no longer of order $X\log^{\p-1}(X)$. Up to characterizing the corresponding deterministic population spectrum, that is, replacing $\left(\log^{\p-1}(j+1)/j\right)^{\pwrco}$ by the eigenvalues of the population matrices $\cov^{\pi}$ restricted to indices in $[v]$, we expect our arguments to go through in this range as well.
}

\paragraph{\textbf{Open questions}}
Several directions remain open:
\begin{enumerate}[leftmargin=*,itemsep=2pt]
    \item \emph{General activations.} Extending the result from \reb{polynomials} to classical activations (ReLU, $\tanh$, sigmoid) is the most natural next step. The experiments strongly suggest that ``tame'' activations preserve the $j^{-\pwrco}$ behavior.
    \item \emph{Eigenvalue coverage.} The current two-sided sharp bound holds for the $j$-th eigenvalue up to $j=O \left( d \log^{-(\p+1)}(d) \right)$, which is almost all $d$ eigenvalues but with a polylogarithmic gap. This gap is related to the well-known missing logarithm in covariance estimation: the empirical covariance matrix requires $n \gtrsim d \log d$ samples rather than $n \gtrsim d$  in some heavy-tailed examples to achieve operator-norm concentration \cite[Remark 5.43]{vershynin2010introduction}, while for sub-Gaussian random vectors, $n\gtrsim d$ suffices.  Under only a finite moment assumption on the random vectors, the log factors also appear in \cite{mendelson2012generic,tikhomirov2018sample}. For heavier-tailed data (such as the Gaussian chaos variables that arise here, which are not sub-Gaussian), this logarithmic factor can be removed by replacing the empirical covariance with a robust estimator \cite{mendelson2018robust,abdalla2024covariance}. We expect that getting sharp bounds for all eigenvalues is possible, but this would require different techniques.
    \item \emph{Universality in $\W$.} We expect the result to hold for sketch matrices with i.i.d.\ entries satisfying a two-moment matching condition, at least for light-tailed distributions. The reason is that the leading term in the expansion comes from the homogeneous chaos, which depends only on the first two moments of the sketch distribution. This is established for $\p = 1$ in \cref{thm:linear}.
    \item \emph{Test error under power-law data.} Combining this spectral characterization with the analysis of random-feature ridge regression \cite{mei2019generalization,mei2021generalization,wang2023overparameterized,defilippis2024dimension} to obtain scaling laws for the test error is a natural goal.
\end{enumerate}

\section{Warm\reb{-}up: Linear Case}
To illustrate our proof strategy for \cref{theorem:MAIN}, we first present the linear case $f(x)=x$. Compared to \cite{lin2024scaling}, we can relax the distribution of the sketching matrix $\m{W}$ from Gaussian to general distributions under moment conditions. In fact, for the linear case we will assume a weaker moment assumption on $\m{W}$ than our main \cref{theorem:MAIN}. In addition, we assume that $\x\in \R^v$ is any random distributed vector with covariance $\cov$.  It is easy to check for linear activation function,
\begin{equation}\label{eq:linear_cov}
    \E_{\x} \left[ \frac{1}{d}f(\W^\top \x)^{\otimes 2}\right] = \frac{1}{d} \W^\top \cov \W.
\end{equation}
We will show $\frac{1}{d}\W^\top \cov \W$ has power-law spectrum \reb{with exponent $\pwrco$} for the first $O(d/\log d)$ many eigenvalues.

\begin{theorem} \label{thm:linear}
    Let $v \geq d$ and $\W \in \R^{v \times d}$ be a random matrix with independent centered column vectors $\W_{\col(i)}$ with isotropic covariance, 
    $\E\left[\W_{\col(i)}\W_{\col(i)}^\top\right] = \Id$. Additionally, suppose for all $p\geq 1$, there exists $C_p>0$ such that for all $s$-dimensional sub-vector of $\W_{\col(i)}$, denoted by $\W_{s,\col(i)}$, we have for all $\m{A}\in \R^{s \times s}$,
    \begin{equation}\label{eq:quad_form_col}
     \norm{\inner{\W_{s,\col(i)}^{\otimes 2}}{\m{A}} - \tr(\m{A})}_{L^p} \leq C_p \norm{\m{A}}_F.  
    \end{equation}
    If $\cov$ has power-law spectrum \reb{with exponent} $\pwrco>1$ and \reb{$d$ is sufficiently large}, then there exist \reb{dimension independent constants} $c_1,c_2, c_3>0$ such that with probability $1-O\left(\log^{-4}(d) \right)$, for all $1\leq j \leq c_1 d \log^{-1}(d)$, 
    \begin{equation}
 c_2 j^{-\alpha}\leq \lambda_j\left(\frac{1}{d}\W^\top \cov \W  \right) \leq c_3 j^{-\pwrco}.
\end{equation}
\end{theorem}

\begin{proof}
Let $k_* = c_1 d \log^{-1}(d)$ for some $c_1>0$ to be specified later and $\mathcal{K}_0 = [k_*]$. We will write $\W_{0} \in \R^{k_* \times d}$ and $\cov_{0} \in \R^{k_* \times k_*}$ to denote the row restriction of $\W$ and $\cov$ onto $\mathcal{K}_0$ respectively. Similarly, for $k\geq 1$, let $\mathcal{K}_k = [k_*2^{k-1}+1,k_*2^{k}]$ and $\W_k \in \R^{(k_*2^{k-1}) \times d}$ and $\cov_k \in \R^{(k_*2^{k-1}) \times (k_* 2^{k-1})}$ be the row restriction of $\W$ and $\cov$ onto $\mathcal{K}_k$ respectively. By a block-wise decomposition\reb{,} we may write $\frac{1}{d}\W^\top \cov \W$ as
\begin{equation}\label{eq:dyadic_decomp}
    \frac{1}{d}\W^\top \cov \W = \frac{1}{d}\W_{0}^\top \cov_{0} \W_{0} + \sum_{k=1}^\infty \frac{1}{d} \W_k^\top \cov_k \W_k,
\end{equation}
where it is understood that when $k_*2^{k-1}+1>v$, then $\m{W}_k=0$. We will refer to the former term ($k=0)$ of \eqref{eq:dyadic_decomp} as the head matrix, and the latter terms ($k\geq 1$) as the tail matrices. Now observe by Loewner ordering of positive semidefinite matrices that
\begin{align}\label{eq:dyadic_psd_bds}
    \lambda_j\left( \frac{1}{d}\W_0^\top \cov_0 \W_0 \right)&\leq \lambda_j\left( \frac{1}{d}\W^\top \cov \W \right) \\
    &\leq  \reb{ \lambda_j\left( \frac{1}{d}\W_0^\top \cov_0 \W_0 \right) + \sum_{k=1}^\infty 
    \lambda_1 \left( \frac{1}{d}\W_k^\top \cov_k \W_k \right).}
\end{align}
Thus, if we can show that the spectrum of the tail matrices are dominated by the spectrum of the head matrix, then \eqref{eq:dyadic_psd_bds} would imply that
\begin{equation}
    \lambda_j\left( \frac{1}{d}\W^\top \cov \W \right) \asymp \lambda_j\left( \frac{1}{d}\W_0^\top \cov_0 \W_0 \right).
\end{equation}
As such, the goal for the remainder of this proof will be to obtain two-sided bounds for the head matrix and one-sided upper bounds for the tail matrices. 

For a fixed $k\geq 0$, we observe that $\W_k^\top \cov_k \W_k$ has the same non-zero spectrum as 
\\
\noindent
$\sqrt{\cov_k}\W_k\W_k^\top \sqrt{\cov_k}$. \reb{Hence}, it suffices to bound the latter. To do so, we first apply a rank-one decomposition to $\sqrt{\cov_k}\W_k\W_k^\top \sqrt{\cov_k}$ to obtain
\begin{equation} \label{eq:linear_k_rank_1_decomp}
    \frac{1}{d}\sqrt{\cov_k}\W_k\W_k^\top  \sqrt{\cov_k} = \frac{1}{d}\sum_{i=1}^d \sqrt{\cov_k}\wc^{(i)}\wc^{(i) \top} \sqrt{\cov_k},
\end{equation}
where $\wc^{(i)} \in \R^{ \abs{\mathcal{K}_k} }$ are the column vectors of $\W_k$. Now, let \[M = \max_{1\leq i \leq d} \frac{1}{d}\norm{\wc^{(i)}\wc^{(i) \top} -\Id}.\] By \cref{lemma:Max_mmt}, for all $t>0$ and $q \geq 1$, there exists $C_q>0$ such that
\begin{equation}\label{eq:linear_prob_bdd}
    \prob \left(M \geq \frac{2 \abs{\mathcal{K}_k}  }{d}+t \right) \leq \frac{ C_q \abs{\mathcal{K}_k}^{q/2}}{d^{q-1} t^{q}}.
\end{equation}
Moreover,
if $q>p$ there exists $C_{q,p}$ such that the $p$-th moment of $M$ is bounded by
\begin{equation}\label{eq:linear_mmt_bdd}
    \E \left[ M^p \right] \leq \left( \frac{\abs{\mathcal{K}_k}}{d} + \frac{C_{q,p} \sqrt{ \abs{\mathcal{K}_k} }}{d^{1-1/q}} \right)^p.
\end{equation}
Since the columns of $\W$ are isotropic, it follows that
\begin{equation}
    \E \left[ \frac{1}{d}\wc^{(i)} \wc^{(i) \top} \right] = \frac{1}{d}\Id,
\end{equation}
and by \eqref{eq:quad_form_col}, for all unit vectors $\m{u}\in \R^{\abs{\mathcal{K}_k}}$,
\begin{align}
    &\m{u}^\top \E \left[\left(\frac{1}{d}\wc^{(i)} \wc^{(i) \top} \right)^2  \right]\m{u} = \frac{1}{d^2} \E \left[ \wc^{(i) \top} \wc^{(i)} \left( \m{u}^\top \wc^{(i)} \wc^{(i) \top} \m{u}\right) \right] \\ 
    \quad &
    \lesssim \frac{1}{d^2} \left( \norm{\wc^{(i) \top} \wc^{(i)} - \abs{\mathcal{K}_k} }_{L^2}+\abs{\mathcal{K}_k} \right) \left( \norm{\wc^{(i)\top} \m{u}\m{u}^\top \wc^{(i)} - 1 }_{L^2} +1 \right)\lesssim \frac{1}{d^2} \abs{\mathcal{K}_k} . \label{eq:linear_second_mmt_bd}
\end{align}
That is, if we choose $\m{V}^2 = c\Id$ with  $c \asymp  \frac{ \abs{\mathcal{K}_k} }{d}  $, then $\m{V}^2 \succeq \sum_{i=1}^d \E \left[\left(\frac{1}{d}\wc^{(i)} \wc^{(i)\top}\right)^2  \right]$. 

\noindent\textit{Step 1: Two-sided bound on $\frac{1}{d}\W_{0}^\top \cov_{0}\W_{0} $.} 

In the case that $k=0$, we have $\abs{\mathcal{K}_0} = k_*$. So, $\norm{\m{V}^2} = O\left(\frac{k_*}{d} \right)$ and $r(\m{V}^2) = k_*$. It is easy to check that for sufficiently large $d$, $\max \left\{ \norm{\m{V}^2}, 32 \E [M] \right \} \leq \frac{1}{24}$. We now apply  \cref{lemma:minsker_2} with parameters $(t,q)= \left(\frac{1}{24},2 \right)$ as specified in the respective lemma to get
\begin{align}
    & \prob \left( \norm{\frac{1}{d}\W_0\W_0^\top -\Id } \geq \frac{1}{2} \right) \\
    \quad &\lesssim k_* \exp \left( \frac{-1}{O \left(\norm{\V^2}+  \E[M] \right) } \right) +  \prob \left(M \geq \frac{1}{24} \right) + \left( \left(\frac{8}{\log(2e)}  \right)^2  \E[M^2] \right)^2\\
    &\lesssim k_* \exp \left( \frac{-1}{O \left(\norm{\V^2}+  \E[M] \right) } \right) +  \prob \left(M \geq \frac{k_*}{d}+\frac{1}{48} \right) + \E[M^2]^2
    \\
    &\lesssim k_* \exp \left( O\left( \frac{k_*}{d}+ \frac{\sqrt{k_*}}{{d^{1-1/q}} } \right)^{-1} \right) + \frac{k_*^{q/2}}{d^{q-1}} + \left( \frac{k_*}{d}+\frac{\sqrt{k_*}}{d^{1-1/q}} \right)^4.\label{eq:head_bdd}
\end{align}
Choosing $q=4$, there exists $c_1>0$ such that if $k_* = c_1 d\log^{-1}(d) $ then \eqref{eq:head_bdd} is bounded by $O\left( \log^{-4}(d)\right)$. Which is to say that with probability at least $1-O\left( \log^{-4}(d)\right)$,
$\norm{\frac{1}{d}\W_0 \W_0^\top - \Id} < \frac{1}{2}$.
Therefore, 
 $ \frac{1}{2}  \cov_0 \preceq  \frac{1}{d}\sqrt{\cov_0}\W_0 \W_0^\top \sqrt{\cov_0}
    \preceq \frac{3}{2}\cov_0.$
By the \reb{Courant-Fischer min-max theorem}, we see that for all $1\leq j \leq c_1 d \log^{-1}(d)$,
\begin{equation}\label{eq:linear_head_up_low_bdd}
    \lambda_j\left( \frac{1}{d}\W_0^\top \cov_0 \W_0 \right) \asymp \lambda_j(\cov_0) \asymp j^{-\pwrco}.
\end{equation}
\noindent\textit{Step 2: Upper bound on $\frac{1}{d}\W_{k}^\top \cov_{k}\W_{k} $.} 

In the case that $k\geq 1$, we have
$\abs{\mathcal{K}_{k}} = k_*2^{k-1}$. Similarly to the above case, we see that $\norm{\m{V}^2} = O\left(\frac{k_* 2^{k-1}}{d} \right)$ and $r(\m{V}^2) = k_*2^{k-1}$. It is then easy to check for sufficiently large $d$ and any choice of $\beta\in (1,\pwrco)$, $\max \left\{ \norm{\m{V}^2}, 32 \E [M] \right \} \leq \frac{2^{\beta k}}{12}$. By \cref{lemma:minsker_2} with parameters
$(t,q) = \left(\frac{2^{\beta k}}{12},2 \right)$, we obtain
\begin{align}
    &\prob \left( \norm{ \frac{1}{d}\W_k\W_k^\top -\Id} \geq 2^{\beta k} \right)\\
    &\lesssim k_*2^{k} \exp \left( \frac{-2^{2\beta k}}{O\left( \norm{\V^2}+ 2^{\beta k}\E[M] \right)} \right) + \prob \left( M \geq 2^{\beta k}\right) +  \left(\frac{\E\left[M^2 \right]}{2^{2\beta k}} \right)^2\\
    & 
    \lesssim 
    k_*2^{k} \exp \left( \frac{-2^{2\beta k}}{O\left( \norm{\V^2}+ 2^{\beta k}\E[M] \right)} \right) + \prob \left( M \geq \frac{2k_*2^k}{d} + \frac{2^{\beta k}}{2}\right) +  \left(\frac{\E\left[M^2 \right]}{2^{2\beta k}} \right)^2\\
    &
    \lesssim k_*2^k \exp \left( \frac{-2^{\beta k}}{O\left( \frac{k_*2^k}{d}+ \frac{\sqrt{k_* 2^k}}{{d^{1-1/q}}} \right)} \right) + \frac{ \left( k_* 2^k \right)^{q/2} }{d^{q-1}2^{2\beta k} } + \left( \frac{k_*2^{k(1-\beta)} }{d }+ \frac{\sqrt{k_*} 2^{(1/2 -\beta) k}}{d^{1-1/q}} \right)^4. \label{eq:linear_tail_bdd_1}
\end{align}
We may choose $q=4$ so that the tail bound reduces to
\begin{equation}\label{eq:linear_tail_bdd_2}
    \prob \left( \norm{\frac{1}{d}\W_k\W_k^\top -\Id} \geq 2^{\beta k} \right) \lesssim \frac{2^k}{\log(d)}  d^{1-2^{k(\beta-1)}} + \frac{2^{2k(1-\beta)}}{d\log^2(d)}+ \left( \frac{2^{k(1-\beta)}}{\log(d)}\right)^4.
\end{equation}
Since $\beta>1$, \eqref{eq:linear_tail_bdd_2} is summable and moreover,
\begin{equation}
    \prob \left( \bigcup_{k=1}^\infty \norm{\frac{1}{d}\W_k\W_k^\top -\Id} \geq 2^{\beta k} \right) \leq O\left( \log^{-4}(d) \right).
\end{equation}
It follows that with probability at least $1-O\left( \log^{-4}(d) \right)$, for all $k\geq 1$,
\begin{equation}
    \norm{\frac{1}{d}\sqrt{\cov_k}\W_k\W_k^\top \sqrt{\cov_k}}  \leq (1+2^{\beta k}) \norm{\cov_k} \leq \left( \frac{1}{k_* 2^k}\right)^\pwrco + \frac{2^{k(\beta-\pwrco)}}{k_*^\pwrco}.
\end{equation}
Given that $\beta<\pwrco$, we see that the norm of the tail matrices is also summable, and in particular,
\begin{equation}\label{eq:linear_tail_bdd}
    \sum_{k=1}^\infty \norm{\frac{1}{d}\W_k^\top \cov_k \W_k}  \leq O \left( k_*^{-\pwrco} \right).
\end{equation}
Finally, to conclude our proof, we combine \eqref{eq:dyadic_psd_bds}, \eqref{eq:linear_head_up_low_bdd} and \eqref{eq:linear_tail_bdd}. That is, there exist positive constants $c_1,c_2$ and $c_3$ such that for all $1\leq j \leq  c_1 d \log^{-1}(d)$,
\begin{equation}
   c_ 2j^{-\pwrco} \leq  \lambda_j\left( \frac{1}{d} \W^\top \cov \W \right) \leq  \lambda_j \left( \frac{1}{d}\W_0^\top \cov_0 \W_0 \right)+ \sum_{k=1}^\infty \norm{\frac{1}{d}\W_k^\top \cov_k \W_k} \leq c_3 j^{-\pwrco}.
\end{equation}
\qquad
\end{proof}

\section{Higher order case: monomial activations}

For  this section, we will generalize \cref{thm:linear} to higher-order monomials $f(x) = x^{\p}$. We assume that the data $\x \sim N(0,\cov)$, and we 
replace the moment assumption on the sketch matrix $\W \in \R^{v \times d}$ \reb{with the assumption that its entries are independent standard Gaussians}. Since $\m{W}$ is invariant under orthogonal transformations, we will assume without loss of generality that $\cov$ is diagonal. To obtain bounds on the eigenvalues of $\E_{\x}[f(\m{W}^\top \x)^{\otimes 2}]$, where it is understood that $f$ is applied \ep{entrywise} to $\W^\top \x$, it will be convenient to write $\x = \sqrt{\cov} \m{z}$ for $\z\sim N(0,\Id)$ and to denote $\m{y}^{(i)} = \sqrt{\cov} \m{W}_{\col(i)}$. We shall begin by decomposing this kernel into leading and sub-leading terms.  
For $\p$-th order monomials, the $ij$-th entry is given by
\begin{align}
    \E_{\x}\left[ f(\m{W}^\top \x )^{\otimes 2}_{ij} \right] &= \E_{\z}\left[ f(\m{W}^\top \sqrt{\cov}\z)_{ij}^{\otimes 2} \right] = \E_{\z}\left[  \inner{\m{y}^{(i)}}{\z}^\p \inner{\m{y}^{(j)}}{\z} ^\p\right]\\
    &=\sum_{i_1,\cdots ,i_{2\p}} \E\left[ \z_{i_1}\dots \z_{i_{2\p}}\right]\m{y}^{(i)}_{i_1}\dots \m{y}^{(i)}_{i_\p}\m{y}^{(j)}_{i_{\p+1}}\dots\m{i}^{(j)}_{i_{2\p}}. \label{eq:p_outerprod}
\end{align}
Using Isserlis' theorem \cite{janson1997gaussian}, the expectation of a product of Gaussians can be expressed as a sum of 
products of second moments. Visually, 
this product may be interpreted as a graph made up of $2\p$-vertices, where each vertex has exactly one edge. We will refer to an edge joining vertex $r$ and vertex $s$ as a pair $(r,s)$. 
Let $\mathcal{P}_2(2\p)$ be the set of all partitions of $\{1,\dots,2\p\}$ into $\p$ pairs. We may then write the expectation of a product of Gaussians as
\begin{align}
     \E_{\z}\left[ \z_{i_1}\dots \z_{i_{2\p}}\right] =\sum_{\xi \in \mathcal{P}_2(2\p)} \prod_{(r,s)\in \xi}\E \left[\z_{i_r} \z_{i_s} \right] = \sum_{\xi \in \mathcal{P}_2(2\p)} \prod_{(r,s)\in \xi} \delta_{i_r,i_s}.
\end{align}
Thus, 
\begin{equation}
 \E_{\x}\left[ f(\m{W}^\top \x )^{\otimes 2}_{ij} \right] 
 = \sum_{\xi \in \mathcal{P}_2(2\p)} \sum_{i_1,\cdots,i_{2\p}}  \prod_{(r,s)\in \xi} \delta_{i_r,i_s} \m{y}_{i_1}^{(i)}\dots \m{y}^{(i)}_{i_\p}\m{y}^{(j)}_{i_{\p+1}}\dots\m{y}^{(j)}_{i_{2\p}}
\end{equation}
However, given that sub-indices from $1$ to $\p$ and sub-indices from $\p+1$ to $2\p$ correspond to entries of $\m{y}^{(i)}$ and $\m{y}^{(j)}$ respectively, we can interpret each pair $(r,s) \in \xi$ to be the product
$\m{y}^{(i)}_{i_r} \m{y}^{(j)}_{i_s}$. Combined with the fact that the \reb{Kronecker delta} implies $i_r = i_s$, the sum of $\m{y}^{(i)}_{i_r} \m{y}^{(j)}_{i_r}$ across the index $i_r$ corresponds to the inner product $\inner{\m{y}^{(i)}}{\m{y}^{(j)}}$.  Let $C_\xi$ be the number of pairs formed between $\{1,\dots,\p\}$ and $\{\p+1,\dots,2\p\}$, or equivalently the number of inner products $\inner{\m{y}^{(i)}}{\m{y}^{(j)}}$ formed. Then, 
\begin{equation}\label{eq:kernel_decomp}
    \E_{\x}\left[ f(\m{W}^\top \x )^{\otimes 2}_{ij} \right]  =\sum_{\xi \in \mathcal{P}_2(2\p)} \left( \inner{\m{y}^{(i)}}{\m{y}^{(i)}} \inner{\m{y}^{(j)}}{\m{y}^{(j)}}\right)^{\frac{\p-C_\xi}{2}} \inner{\m{y}^{(i)}}{\m{y}^{(j)}}^{C_\xi}.
\end{equation}
We will write $A_\xi = \frac{\p-C_\xi}{2}$ and $q= C_\xi$ for simplicity of notation. \reb{If $C_\xi\neq 0$, then} the cross term inner product can be further expanded into
\begin{equation} \label{eq:kernel_xi_ij}
    \inner{\m{y}^{(i)}}{\m{y}^{(j)}}^{q}
    =
    \sum_{i_1,\dots,i_q} \m{y}_{i_1}^{(i)} \y_{i_1}^{(j)} \dots \y^{(i)}_{i_q}\y^{(j)}_{i_q}.
\end{equation}
Since the summand is symmetric, we can rewrite this sum as an ordered sum. In particular, given an unordered multiset $(j_1,\dots, j_q)$, if $\ell$ is the number of unique indices in $(j_1,\dots, j_q)$ and $(i_1<\cdots<i_\ell)$ is the ordered unique indices, we may define $\pi_r$ to be the multiplicity of $i_r$ in $(j_1,\dots,j_q)$. Then, $\pi = (\pi_1,\dots,\pi_\ell)$ is a composition of $q$, i.e. $\sum_{r=1}^\ell \pi_r = q$.
There are exactly $\reb{\frac{q!}{\pi_1! \dots \pi_l!} \eqqcolon C_{\pi}} $ multisets that map to the pair $((i_1<\cdots<i_{\ell}),\pi)$. Thus, \eqref{eq:kernel_xi_ij} becomes
\begin{align}\label{eq:inner_prod_power_p0}
    \inner{\m{y}^{(i)}}{\m{y}^{(j)}}^{q}
    =
    \sum_{\ell=1}^q 
    \sum_{\pi}
    \sum_{i_1<\dots<i_\ell}
    \reb{C_{\pi}}
    \left( \y^{(i)}_{i_1} \y^{(j)}_{i_1}\right)^{\pi_1} \dots \left( \y^{(i)}_{i_\ell} \y^{(j)}_{i_\ell}\right)^{\pi_\ell},
\end{align}
where the sum across $\pi$ indicates the summation across all compositions of $q$ of size $\ell$. 
Using this new ordered tuple $(i_1<\cdots<i_\ell)$, if  we define $\m{Y}^\pi \in \R^{{v \choose l} \times d}$ and $\D \in \R^{ d \times d}$ by
\begin{equation} \m{Y}_{(i_1<\dots<i_\ell),i}^{\pi} = \m{y}_{i_1}^{(i),\pi_1}\dots \m{y}_{i_\ell}^{(i),\pi_\ell} \quad \text{and}\quad \m{D} = \operatorname{Diag}\left(\inner{\m{y}^{(i)}}{\m{y}^{(i)}} : 1\leq i \leq d\right),
\end{equation}
we see that each summand of \eqref{eq:kernel_decomp} is just the $ij$-th entry of the sum of 
Gram-matrices $\left( \m{Y}^{\pi}  \D^{A_\xi} \right)^\top \m{Y}^{\pi}  \D^{A_\xi} $.  \reb{On the other hand, if $C_\xi =0$,  the summand of \eqref{eq:kernel_decomp} is the rank-$1$ matrix given by
\begin{equation}
    \frac{1}{d} \left(\m{u}^{\p/2}\right)^{\otimes 2}, \quad \text{where } \m{u}^{\p/2}_i = \left(\W_{\col(i)}^\top \cov \W_{\col(i)} \right)^{\p/2}.
\end{equation}
Given that the columns of $\W$ are independent and moreover subgaussian, the operator norm or equivalently its trace
\begin{equation}
    \norm{\frac{1}{d} \left(\m{u}^{\p/2}\right)^{\otimes 2}} = \frac{1}{d}\sum_{i=1}^d \left(\W_{\col(i)}^\top \cov \W_{\col(i)} \right)^{\p} = \frac{1}{d}\sum_{i=1}^d  \norm{ \sqrt{\cov} \W_{\col(i)}} ^{2\p}, 
\end{equation}
concentrates to $\E\left[\norm{ \sqrt{\cov} \W_{\col(1)}} ^{2\p} \right] = O(1)$. \ep{When $\p$ is even there are exactly $((\p-1)!!)^2$ pairings such that $C_\xi = 0$, and when $\p$ is odd there are none. Since $\E_\x\left[\inner{\W_{\col(i)}}{\x}^{\p}\right] = (\p-1)!!\left(\W_{\col(i)}^\top \cov \W_{\col(i)}\right)^{\p/2}$ for even $\p$, this term is exactly the rank-$1$ matrix $\frac{1}{d}\E_\x\left[f(\W^\top\x)\right]^{\otimes 2}$ that separates the second-moment matrix from the covariance (see Appendix~\ref{Appendix:cov_vs_2nd}).}
}
\reb{Putting all this together, we see that the covariance kernel may be decomposed into a finite $\p$-dependent collection of positive semidefinite terms. Namely, }
\begin{equation}\label{eq:Krep}
    \frac{1}{d}\E_\x\left[ f(\W^\top \x)^{\otimes 2} \right] =\reb{  \frac{\ep{((\p-1)!!)^2}}{d} \left(\m{u}^{\p/2}\right)^{\otimes 2}}  +\sum_{ \substack{\xi \in \mathcal{P}_2(2\p) \\C_\xi \neq 0 }} \sum_{\ell=1}^{C_\xi}
    \sum_{\pi}  \reb{C_{\pi}} \left(\frac{1}{d}\m{D}^{A_\xi}\m{Y}^{\pi \top} \m{Y}^{\pi}
    \m{D}^{A_\xi} \right).
\end{equation}
\reb{
The remainder of this paper will be devoted to showing that each Gram matrix, $\frac{1}{d}\m{D}^{A_\xi}\m{Y}^{\pi \top} \m{Y}^{\pi}\m{D}^{A_\xi}$, has a weak power-law spectrum with exponent $\pwrco$. Once this is established, the following lemma implies that the second-moment matrix also has a weak power-law spectrum with the same exponent $\pwrco$.}

\reb{
\begin{lemma}\label{lemma:sum_PSD_pwr_law_spec}
    \ep{Let $J\leq d$, let $f:\R^{+} \to \R^{+}$ be non-decreasing, and let $\{\m{A}_i \}_{i=1}^N$ be $d \times d$ positive semidefinite matrices with $\m{A} = \sum_{i=1}^N \m{A}_i$.
    \begin{enumerate}[label=(\roman*)]
        \item If $\lambda_j(\m{A}_i) \lesssim  \left( \frac{f(j)}{j} \right)^{\pwrco}$ for all $i \in [N]$ and $j\in [J]$, then $\lambda_j(\m{A}) \lesssim \left( \frac{f(j)}{j} \right)^{\pwrco}$ for all $j\in [J]$.
        \item If, in addition, $\lambda_j(\m{A}_{i_0}) \gtrsim \left( \frac{f(j)}{j} \right)^{\pwrco}$ for some $i_0\in[N]$ and all $j\in[J]$, then $\lambda_j(\m{A}) \asymp \left( \frac{f(j)}{j} \right)^{\pwrco}$ for all $j\in [J]$.
        \item Both statements remain true with $\m{A}$ replaced by $\m{A}+\m{B}$, where $\m{B}$ is positive semidefinite of rank $1$ with $\norm{\m{B}}\lesssim 1$.
    \end{enumerate}
    The implicit constants in the conclusions depend only on $N$, $\pwrco$, $f(1)$ and the implicit constants in the hypotheses.}
\end{lemma}
\begin{proof}
    Let $k\geq 1$ and $j = kN$, then repeated applications of  Weyl's inequality yields 
    \begin{equation}
        \lambda_{kN}(\m{A}) \leq  \sum_{i=1}^N \lambda_k(\m{A}_i) \lesssim N^{1+\pwrco} \left( \frac{f(kN)}{kN} \right)^{\pwrco} \label{eq:lambda_kn}.
    \end{equation}
    If $j= kN+q$, for some $k\geq 1$ and $1\leq q<N$, then by \eqref{eq:lambda_kn}
    \begin{equation}
        \lambda_{kN+q}(\m{A}) \leq \lambda_{kN}(\m{A}) \lesssim \left(\frac{f(kN+q)}{kN} \right)^\pwrco \lesssim \left(\frac{ 2 f(kN+q)}{2kN} \right)^\pwrco \lesssim \left(\frac{f(kN+q)}{kN+q} \right)^\pwrco .
    \end{equation}
    Now, if $j<N$ then
    \begin{equation}
        \lambda_j(\m{A}) \leq \sum_{i=1}^N \lambda_1(\m{A}_i) \lesssim N f(1)^{\pwrco}.
    \end{equation}
    \ep{Since $f$ is non-decreasing, $f(1)\leq f(j)$, and therefore $N f(1)^\pwrco \leq N j^{\pwrco}\left( \frac{f(j)}{j} \right)^{\pwrco} \leq N^{1+\pwrco}\left( \frac{f(j)}{j} \right)^{\pwrco}$ for $1\leq j<N$.}
    That is, for all $j<N$,
    \begin{equation}
        \lambda_j(\m{A}) \lesssim \ep{N^{1+\pwrco}} \left( \frac{f(j)}{j} \right)^{\pwrco},
    \end{equation}
    which concludes \ep{part (i); part (ii) then follows from the Loewner ordering $\m{A}\succeq \m{A}_{i_0}$. For part (iii), the lower bound again follows from $\m{A}+\m{B}\succeq\m{A}$. For the upper bound, Weyl's inequality gives $\lambda_1(\m{A}+\m{B})\leq \lambda_1(\m{A})+\norm{\m{B}}\lesssim f(1)^{\pwrco}+1$ and, since $\lambda_2(\m{B})=0$, $\lambda_j(\m{A}+\m{B})\leq \lambda_{j-1}(\m{A})\lesssim \left(\frac{f(j-1)}{j-1}\right)^{\pwrco}\leq 2^{\pwrco}\left( \frac{f(j)}{j} \right)^{\pwrco}$ for $2\leq j\leq J$; see also Appendix \ref{Appendix:cov_vs_2nd}.} Thus, $\m{A}+\m{B}$ also follows the same weak power-law spectrum up to dimension independent constants.
\end{proof}
}

\subsection{Principal term}

Following \eqref{eq:Krep}, we first consider pairings $\xi$ that only matches elements of $\{1,\dots,\p\}$ to $\{\p+1,\cdots,2\p \}$. For such pairings, we have that $C_\xi = \p$, $A_\xi = 0$ and we sum over all compositions $\pi$ of $\p$. The trivial composition is the all $1$-composition which we will denote by $\m{1}^\p=(1,\dots,1)$. We will refer to $\frac{1}{d}\m{Y}^{\m{1}^\p \top}\m{Y}^{\m{1}^\p}$ as the principal term, since we will derive both upper and lower bounds for its spectrum.

Before proceeding with the proof of the bounds for the principal term, it will be useful to untangle $\cov$ from $\m{Y}^{\pi}$ by recalling that  $\y^{(i)}_{i_k} = \sqrt{\cov_{i_k}} \m{W}_{i_k,i}$. So, if we define a new sketch matrix $\W^\pi \in \R^{{v \choose l} \times d}$, along with a new covariance matrix $\cov^\pi \in \R^{{ v\choose l}\times {v \choose l}}$ by
\begin{equation} 
    \begin{gathered}
    \m{W}_{(i_1<\dots<i_\ell),i}^{\pi} = \W_{i_1,i}^{\pi_1}\dots \W_{i_\ell,i}^{\pi_\ell}, \quad \text{and}\quad \cov^\pi  = \text{Diag}\left( \left\{ \cov_{i_1}^{\pi_1}\dots \cov_{i_\ell}^{\pi_\ell}\right\} _{(i_1<\dots<i_\ell)}  \right),
    \end{gathered}
    \label{eq:W^pi_K^pi}
\end{equation}
then we see that
    $\frac{1}{d}\m{Y}^{\pi\top} \m{Y}^{\pi}= \frac{1}{d}\W^{\pi \top} \cov^\pi \W^\pi$,
which is analogous to the linear case as seen in \eqref{eq:linear_cov}. As such, we may apply a similar proof strategy of block-wise decomposition of $\m{Y}^{\pi \top} \m{Y}^{\pi}$ to isolate the leading eigenvalue terms.

\begin{proposition}[Principal term]\label{prop:principal} Assume $\alpha>1$. \ep{The upper bounds below hold without any lower bound on $v$ in terms of $d$. The lower bounds hold provided $v\geq d$ if $\p=1$, and $v\geq d^{\frac{1}{\p}+\epsilon}$ for some fixed $\epsilon>0$ if $\p\geq 2$, in which case the constants also depend on $\epsilon$.}
    There are \reb{dimension independent} constants $c_1,c_2$, $c_3$, and $c_4$ so that with  probability $1-O(\log^{-4}(d))$, for all $1 \leq j \leq c_1d \log^{-(\p+1)}(d)$, 
    \begin{equation} \label{eq:principal_eig_asym}
    c_2 \left(   \frac{\log^{(\p-1)}(j+1)}{ j} \right)^{\pwrco}
    \leq
    \lambda_j\left(\frac{1}{d}\m{Y}^{\m{1}^\p\top} \m{Y}^{\m{1}^\p} \right) 
    \leq c_3 \left(   \frac{\log^{(\p-1)}(j+1)}{ j} \right)^{\pwrco}.
    \end{equation}
\reb{And when $\p\geq 2$, for all $c_1d \log^{-(\p+1)}(d)<j\leq  d$, 
   \begin{equation}
   \label{eq:princpal_upper_lower_lower}
 c_2 \left( j^{-1}\log^{-2}(j+1) \right)^\pwrco \leq   \lambda_j\left(\frac{1}{d}\m{Y}^{\m{1}^\p\top} \m{Y}^{\m{1}^\p} \right) 
    \leq c_3 \left( \frac{\log^{2 \p}(j+1)}{j} \right)^\pwrco.
    \end{equation} } 
\end{proposition}
\begin{proof} Let $k_*>0$ and $\mathcal{K}_0$ be the collection of ordered tuples $(i_1<\cdots<i_\p)$ such that $\cov_{(i_1<\cdots<i_\p)}^{\m{1}^\p}$ are among the first $k_*$ largest eigenvalues of $\cov^{\m{1}^\p}$. Let $\W_0^{\m{1}^\p} \in \R^{k_* \times d}$ and $\cov_0^{\m{1}^\p} \in \R^{k_* \times k_*}$ be the respective row restriction of\ $\W^{\m{1}^\p}$ and $\cov^{\m{1}^\p}$ onto $\mathcal{K}_0$. Similarly, for all $k\geq 1$, let $\mathcal{K}_k$ be the collection of ordered tuples $(i_1<\cdots<i_\p)$ such that $\cov_{(i_1<\cdots<i_\p)}^{\m{1}^\p}$ are among the $(k_*2^{k-1}+1)$ to $(k_*2^k)$ largest eigenvalues of $\cov^{\m{1}^\p}$. Let $\W_k^{\m{1}^\p}$ and $\cov_k^{\m{1}^\p}$ be the row restriction of $\W^{\m{1}^\p}$ and $\cov^{\m{1}^\p}$ onto $\mathcal{K}_k$. We set the rows of $\W^{\m{1}^\p}_k$ to be zero if the corresponding ordered tuple does not exist.  We may then decompose $\frac{1}{d}\m{Y}^{\m{1}^\p\top} \m{Y}^{\m{1}^\p}$ as
    \begin{align}\label{eq:decomposeY}
        \frac{1}{d}\m{Y}^{\m{1}^\p\top} \m{Y}^{\m{1}^\p}= \frac{1}{d}\m{W}^{\m{1}^\p\top} \m{H^{\m{1}^\p}}\m{W}^{\m{1}^\p}=\frac{1}{d}\m{W}_{0}^{\m{1}^\p\top} \m{H}_{0}^{\m{1}^\p}\m{W}_{0}^{\m{1}^\p}+\frac{1}{d} \sum_{k=1}^{\infty}\m{W}_{k}^{\m{1}^\p \top} \m{H}_{k}^{\m{1}^\p}\m{W}_{k}^{\m{1}^\p}.
    \end{align}
From \cref{lemma:principal_eig_bounds}, we know for $j \leq d$,
$ \lambda_j(\m{H}^{\m{1}^\p})\asymp \left( \frac{\log^{(\p-1)}(j+1)}{j} \right)^\pwrco$, 
and for $j>d$, $\lambda_j(\cov^{\m{1}^\p}) \lesssim \left( \frac{\log^{(\p-1)}(j+1)}{j} \right)^\pwrco$. 
As in the linear case, it suffices to show $\norm{\frac{1}{d}{\m{W}_k^{\m{1}_\p}}  \m{W}_k^{\m{1}_\p \top} - \Id}$ is sufficiently bounded. Fix $k \geq 0$, by a rank-$1$ decomposition we have

\begin{align}\label{eq:Hk1p}
    \frac{1}{d}\sqrt{\m{H}_{k}^{\m{1}^\p}}\m{W}_{k}^{\m{1}^\p}{\m{W}_{k}^{\m{1}^\p \top}} \sqrt{\m{H}_{k}^{\m{1}^\p}}-\m{H}_{k}^{\m{1}^\p}=  \sqrt{\m{H}_{k}^{\m{1}^\p}}\left(\frac{1}{d} \sum_{i=1}^d \wc^{(i)} \wc^{(i) \top} -\Id \right)\sqrt{\m{H}_{k}^{\m{1}^\p}},
\end{align}
where $\wc^{(i)}$ are the column vectors of $\W_k^{\m{1}^\p}$. Let $M=\max_{i\in [d]} \frac{1}{d} \left\|  \wc^{(i)} \wc^{(i) \top} -\m{I}\right\|$. We first provide moments and tail estimates for $M$. Note that
\begin{align}
    M &\leq \max_{i} \frac{1}{d} \|\wc^{(i)}\|^2+\frac{1}{d} \\
    &\leq \max_{i} \frac{1}{d} | \|\wc^{(i)}\|^2- \abs{\mathcal{K}_k} |+\frac{ \abs{\mathcal{K}_k} +1}{d}\leq \tilde M+\frac{2 \abs{\mathcal{K}_k}  }{d}, \label{eq:M_tildeM}
\end{align}
where $\tilde M=\max_{i} \frac{1}{d} | \|\wc^{(i)}\|^2- \abs{\mathcal{K}_k}  |$. For simplicity of notation, we drop the superscription from $\wc^{(i)}$ and write just $\wc$. Then each term $\|\wc\|^2=\sum_{(i_1,\dots,i_\p)\in \mathcal K_k}  \wc_{i_1}^2\cdots \wc_{i_\p}^2$ is a degree-$2\p$ polynomial in Gaussian random variables $\wc$. Moreover, \begin{align}
    \Var(\|\wc\|^2)\leq \E \|\wc\|^4 -(\E \|\wc\|^2)^2 \leq  3^\p( \abs{\mathcal{K}_k}  )^2.
\end{align}
From the Gaussian hypercontractivity inequality in \cref{lem:hypercontractivity},
\begin{align}\label{eq:tail_tildeM}
        \Pr \left(| \|\wc^{(i)}\|^2-\abs{\mathcal{K}_k}  |\geq t \right)\leq e^2 \exp \left( -C\left(\frac{t}{ \abs{\mathcal{K}_k} }\right)^{1/\p}\right).
\end{align}
Taking union bound over $i\in [d]$, we obtain 
\begin{align} \label{eq:tail_bdd_principal}
    \mathbb P (M\geq t)\lesssim d \exp  \left( -C\left(\frac{td}{ \abs{\mathcal{K}_k}  }-2\right)^{1/\p}\right).
\end{align}
From \eqref{eq:tail_tildeM}, we also get the moment bound by taking $t_0=\abs{\mathcal{K}_k} \log^\p(d)$ and  
\begin{align}
    \E \left[ \tilde M^q \right]  & =d^{-q} q\int_{0}^{t_0} t^{q-1}\Pr \left( \tilde M \geq t\right) dt+ d^{-q}q\int_{t_0}^{\infty}  t^{q-1}\Pr \left( \tilde M \geq t\right) dt\\
    &\lesssim d^{-q} t_0^q+d^{-q} \int_{t_0}^{\infty} \exp(-C(t/k_*)^{1/\p}) dt \\
    &\lesssim \left( \frac{ \abs{\mathcal{K}_k} }{d} \right )^q \log^{\p q}(d),
\end{align} 
where in the last inequality we use the gamma tail estimate for $x\geq 2(\p-1)$ and $\p \geq 1$,$\int_{x}^{\infty}  u^{\p-1} e^{-u} du\leq 2 x^{\p-1}e^{-x}$. 
Therefore, from \eqref{eq:M_tildeM},
    \begin{equation} \label{eq:mmt_bdd_principal}
    \E \left[ M^q  \right]\lesssim \left(\frac{ \abs{\mathcal{K}_k}  \log^\p (d)}{d } \right)^q. 
    \end{equation}
    Let $\m{A}_i =  \frac{1}{d}\left(\wc^{(i)}\wc^{(i) \top} -\m{I}\right).$ It is easy to check via independence that $\E[\m{A}_i^2]$ is diagonal and the $j$-th diagonal term satisfies $\E 
    \left [  \left( \wc^{(i)} \wc^{(i) \top} \right)^2 \right]_{jj}
    \leq 3^\p\abs{\mathcal{K}_k}.$
    By choosing $\m{V}^2 = \frac{3^\p \abs{\mathcal{K}_k} }{d}\Id$, we see that
\begin{equation}
    \sum_{i=1}^d \E \left[ \m{A}_i^2\right] \succeq \m{V}^2 \quad \text{and} \quad \norm{\m{V}^2} = \frac{3^\p \abs{\mathcal{K}_k}  }{d}.
\end{equation}

\noindent\textit{Step 1: Two-sided bound on $\frac{1}{d}\m{W}_{0}^{\m{1}^\p \top} 
\m{H}_{0}^{\m{1}^\p}\m{W}_{0}^{\m{1}^\p}$.}

If $k=0$, then $\abs{ \mathcal{K}_0} = k_*$. We now take $k_* = c_1 d \log^{-(\p+1)}(d)$ for a sufficiently small constant $c_1>0$. Then, $\max \{  4\sigma, 32\E M\} \leq \frac{1}{24}$ for sufficiently large $d$.
Using \eqref{eq:tail_bdd_principal} and \eqref{eq:mmt_bdd_principal}, we may apply \cref{lemma:minsker_2} with $t=\frac{1}{24}$ and $q=2$ to get 
\begin{align}\label{eq:Head_tail_bdd_principal}
    \Pr \left( \norm{ \frac{1}{d}\W_0^{\m{1}^\p}\W_0^{\m{1}^\p \top} - \Id } \geq \frac{1}{2}\right)  \lesssim \frac{1}{\log^{4}(d)}.
\end{align}
Therefore, with probability at least $1-O(\log^{-4}(d))$,
\begin{align}
   \frac{1}{2}\m{H}_{0}^{\m{1}^\p}\preceq \frac{1}{d}\sqrt{\m{H}_{0}^{\m{1}^\p}}\m{W}_{0}^{\m{1}^\p}{\m{W}_{0}^{\m{1}^\p}}^\top \sqrt{\m{H}_{0}^{\m{1}^\p}}\preceq \frac{3}{2} \m{H}_{0}^{\m{1}^\p}.
\end{align}
Hence, by the Loewner ordering of PSD matrices, for $1\leq j\leq c_1 d \log^{-(\p+1)}(d)$, 
\begin{align}\label{eq:two_sided_lambda}
    \lambda_j\left (\frac{1}{d}\m{W}_{0}^{\m{1}^\p\top} 
\cov_{0}^{\m{1}^\p}\m{W}_{0}^{\m{1}^\p} \right)\asymp  \lambda_j(\m{H}_0^{\m{1}^\p})\asymp \left( \frac{\log^{(\p-1)}(j+1)}{j} \right)^\pwrco.
\end{align}

\noindent \textit{Step 2: Upper bound on $ \sum_{k=1}^{\infty}\frac{1}{d}\m{W}_{k}^{\m{1}^\p\top} \m{H}_{k}^{\m{1}^\p}\m{W}_{k}^{\m{1}^\p}$}.

If $k \geq 1$, then $\abs{\mathcal{K}_k}= k_* 2^{k-1}$. It is easy to check for any $\beta \in (1,\pwrco)$ that $\max\{ 4 k_* 2^{k-1}, 32 \E[M]\} \leq 12 \left(2^{\beta k} \right)$. Thus, using \cref{lemma:minsker_2} with  $q=2$ and $t = 12(2^{\beta k}-1)$, we get
\begin{align}
&\Pr\left( \left\| \frac{1}{d}\m{W}_{k}^{\m{1}^\p \top} \m{W}_{k}^{\m{1}^\p}-\m{I}\right\|\geq 2^{\beta k }-1\right)\\
\quad &\lesssim  2^{k -1} d^{1-\frac{1}{c_3} 2^{(\beta-1)k }}+ d \exp(-C 2^{(\beta-1)k } \log^{(p+1)/\p}(d))+\left(\frac{2^{(1-\beta)k }}{\log d}\right)^{4},
\end{align}
which is summable. Therefore, after taking the union bound, we see that with probability at least $1-O(\log^{-4}(d))$, for all $k\geq 1$,
$  \frac{1}{d}\m{W}_{k }^{\m{1}^\p \top} \m{W}_{k}^{\m{1}^\p} \preceq 2^{\beta k }\m{I}.
$
In particular, this implies
\begin{align}
   & \frac{1}{d}\|\m{W}_{k }^{\m{1}^\p\top} \m{H}_{k }^{\m{1}^\p}\m{W}_{k }^{\m{1}^\p}\| \leq  2^{\beta k } \|\m{H}_{k }^{\m{1}^\p}\| \\&\lesssim 2^{\beta k } (2^{k -1}k_*)^{-\alpha} \log(2^{k-1}k_*)^{\alpha(p-1)}\\
    &\lesssim 2^{(\beta-\alpha)k }\left( \frac{\log(2^{ k-1} k_*)}{\log k_*}\right)^{\alpha(p-1)} \lambda_{k_*}(\m{H}^{\m{1}^\p}).
\end{align}
Now, summing over $k \geq 1$, we obtain 
\begin{align}\label{eq:one_sided_lambda}
    \sum_{k =1}^{\infty}  \frac{1}{d}\|\m{W}_{k}^{\m{1}^\p\top} \m{H}_{k }^{\m{1}^\p}\m{W}_{k }^{\m{1}^\p}\| \lesssim  \lambda_{k_*}(\m{H}^{\m{1}^\p}).
\end{align}
Combining \eqref{eq:decomposeY}, \eqref{eq:two_sided_lambda} and \eqref{eq:one_sided_lambda}, we see that for all $1\leq j\leq k_*$,
\begin{align}
    \lambda_j\left(\frac{1}{d}\m{Y}^{\m{1}^\p\top} \m{Y}^{\m{1}^\p} \right)\asymp \lambda_j(\m{H}^{\m{1}^\p})\asymp j^{-\alpha} (\log(1+j))^{\alpha(p-1)}.
\end{align}
This concludes the proof of \eqref{eq:principal_eig_asym}.

\noindent \textit{Step 3: Upper bound on $\lambda_j\left(\frac{1}{d}\m{Y}^{\m{1}^\p\top} \m{Y}^{\m{1}^\p} \right)$ for $j> c_1 d\log^{-(\p+1)}(d)$.} 

If $k_*<j \leq d$, we have $\log(j+1)\asymp \log d$. Therefore,
\begin{align}
     \lambda_j\left(\frac{1}{d}\m{Y}^{\m{1}^\p\top} \m{Y}^{\m{1}^\p} \right)&\leq \lambda_{k_*} \left(\frac{1}{d}\m{Y}^{\m{1}^\p\top} \m{Y}^{\m{1}^\p} \right)\\
     &\lesssim \left(\frac{d}{\log^{\p+1}(d)} \right)^{-\alpha}\log^{\alpha(\p-1)}(d)\lesssim j^{-\alpha} \log^{2\alpha \p} (j+1). \label{eq:principal_lower_eig_uppr_bdd}
\end{align}

\noindent \textit{Step 4: Lower bound on $\lambda_j \left( \m{Y}^{\m{1}^\p \top} \m{Y}^{\m{1}^\p} \right)$ for $j>c_1 d \log^{-(\p+1)}(d)$.} 

\medskip
\reb{\noindent\textit{(a) The range $c_1 d \log^{-(\p+1)}(d)\leq j\leq cd$.}}

Let $\m{r}_j^\top$ be the $j$-th row of $\m{W}^{\m{1}^\p}$. Choose $n_*=cd$ for some sufficiently small constant $c>0$, then
    \begin{align}
           \frac{1}{d}\m{Y}^{\m{1}^\p\top} \m{Y}^{\m{1}^\p}= \frac{1}{d}\m{W}^{\m{1}^\p\top} \m{H^{\m{1}^\p}}\m{W}^{\m{1}^\p}&=\frac{1}{d}\sum_{j=1}^{\infty} \lambda_j(\m{H^{\m{1}^\p}}) \m{r}_j \m{r}_j^\top\\
           &\succeq  \frac{1}{d}\lambda_{n_*}(\m{H^{\m{1}^\p}}) \sum_{j=1}^{n_*}\m{r}_j \m{r}_j^\top=\lambda_{n_*} (\m{H^{\m{1}^\p}}) \cdot \frac{1}{d} \m{W}^{\reb{\m{1}^\p}\top}_{n_*} \m{W}^{\reb{\m{1}^\p}}_{n_*},\label{eq:M_K}
    \end{align}
    Now $\m{W}_{n_*}^{\reb{\m{1}^\p} \top}\in \R^{d\times n_*}$ is a random matrix with independent, isotropic rows $\wc$ such that for every $\x \in \R^{n_*}$ and $\| \x \|_2=1$, $\langle \wc,\x\rangle$ is an order-$\p$ polynomial of independent Gaussian random variables with  $\Var(\langle \wc,\x\rangle)=\|\x\|_2^2=1$. From Gaussian hypercontractivity in \cref{lem:hypercontractivity},
    \[\mathbb P(|\langle \wc , \x \rangle|\geq u )\leq 2\exp(-Cu^{2/\p})\leq  \frac{L}{u^{2+\eta}}.\]
    for some $\eta>0$ and $L>0$.
    Then, from the least singular value bound of random matrices with independent rows in \cite[Theorem 1.3]{koltchinskii2015bounding}, we have with probability $1-e^{-Cd}$, \[\lambda_{\reb{n_*}}(\frac{1}{d} \m{W}_{n_*}^{ \reb{\m{1}^\p} \top} \m{W}_{n_*}^{\reb{\m{1}^\p}} )\geq \frac{1}{2}.\]  From \eqref{eq:M_K}, this implies 
    \reb{
    \begin{align}
         \lambda_{n_*}\left(\frac{1}{d}\m{Y}^{\m{1}^\p\top} \m{Y}^{\m{1}^\p} \right)\geq \lambda_{n_*} (\m{H^{\m{1}^\p}})   \lambda_{n_*}\left(\frac{1}{d} \m{W}_{n_*}^{\reb{\m{1}^\p} \top} \m{W}_{n_*}^{\reb{\m{1}^\p}}\right)\gtrsim \lambda_{n_*} (\m{H^{\m{1}^\p}}) \geq c_2 \left( \frac{\log^{\p-1}(d)}{d}\right)^\pwrco.
    \end{align}}
    Then for  $c_1d \log^{-(\p+1)}(d)\leq j\leq \reb{n_*}$,
    \begin{align}\lambda_j\left(\frac{1}{d}\m{Y}^{\m{1}^\p\top} \m{Y}^{\m{1}^\p} \right)&\geq \lambda_{n_*}\left(\frac{1}{d}\m{Y}^{\m{1}^\p\top} \m{Y}^{\m{1}^\p} \right)\\
    &\geq c_2 d^{-\alpha} (\log(d))^{\alpha(\p-1)}\geq c_2' j^{-\alpha} (\log (1+j))^{-2\alpha}.\label{eq:final_parta}
    \end{align}

    The matrix on the right-hand side of \eqref{eq:M_K} has rank at most $n_*$, so a different argument is required to lower bound $\lambda_j\left(\frac{1}{d}\m{Y}^{\m{1}^\p\top} \m{Y}^{\m{1}^\p} \right)$ when $j>n_*$, which we will finish in part (b).

\medskip

\noindent\textit{(b) The range $cd\leq j\leq d$.} 
It suffices to bound $\lambda_d(\frac{1}{d}\m{Y}^{\m{1}^\p\top} \m{Y}^{\m{1}^\p})$. To do so, we explore the tensor structure in $\frac{1}{d}\m{Y}^{\m{1}^\p\top} \m{Y}^{\m{1}^\p}$.

Recall $\W\in \R^{v\times d}$. Take $m=\lceil (d\log^2 d)^{1/p} \rceil$. Then \ep{$pm\leq \min\{d,v\}$ for sufficiently large $d$, since $\p\geq 2$, $pm\asymp (d\log^2 d)^{1/\p}$ and $v\geq d^{\frac{1}{\p}+\epsilon}$,} and
\[m^{\p}\asymp  d\log^2 d.\] We define independent Gaussian random vectors $\bg_{a,b}\in \R^{m}$ for $1\leq a\leq \p, 1\leq b\leq d$ from the first $pm$ rows of $\W$ by
\begin{align}
    \bg_{a,b}(i)=\W_{(a-1)m+i,b}, \quad 1\leq i\leq m.
\end{align}
Now we can create independent tensorized vectors $\bt_1,\dots,\bt_d$ such that
\begin{align}
    \bt_b=\bg_{1,b}\otimes \bg_{2,b}\cdots \otimes \bg_{\p,b}\in \R^{m^\p}, \quad 1\leq b\leq d,
\end{align}
which means for $i_1,\dots,i_p\in [m]$,
\[ \bt_b(i_1,\dots,i_p)=\bg_{1,b}(i_1)\cdots \bg_{\p,b}(i_p). \]
Let $\bT=[\bt_1,\dots,\bt_d]\in \R^{(m^\p)\times d}$, where each $\bt_b$ is a column vector. Then  $\bT$ is a submatrix of $\m{W}^{\m{1}^\p}\in \R^{\binom{v}{p}\times d}$ by choosing a subset of rows.  Let $(\m{H^{\m{1}^\p}})_{\bT} \in \R^{m^\p\times m^{\p}}$ be the submatrix of $\m{H^{\m{1}^\p}}\in \R^{\binom{v}{p}\times \binom{v}{p}}$ corresponding to the indices of $\bt_b$ among the $\binom{v}{p}$ indices in the column vectors of $\m{Y}^{\m{1}^\p}$.
Since $\cov_i\asymp i^{-\alpha}$, we have 
\begin{align}
    \lambda_{\min} ((\m{H^{\m{1}^\p}})_{\bT})\gtrsim \prod_{a=1}^{\p} (am)^{-\alpha}\gtrsim m^{-\alpha \p}.
\end{align}
Consequently, 
\begin{align}
    \frac{1}{d}\m{Y}^{\m{1}^\p\top} \m{Y}^{\m{1}^\p}&=\frac{1}{d}\m{W}^{\m{1}^\p\top} \m{H^{\m{1}^\p}}\m{W}^{\m{1}^\p}\succeq  \frac{1}{d}\bT^\top (\m{H^{\m{1}^\p}})_{\bT} \bT\succeq \frac{1}{d}  \lambda_{\min} ((\m{H^{\m{1}^\p}})_{\bT}) \bT^\top \bT
    \end{align}
    Hence
    \begin{align}
    \lambda_{d}\left(\frac{1}{d}\m{Y}^{\m{1}^\p\top} \m{Y}^{\m{1}^\p}\right)&\gtrsim  m^{-\alpha \p} \lambda_d\left(\frac{1}{d} \bT^\top \bT\right). \label{eq:PSD_compare_T}
\end{align}
Next we prove a concentration inequality for $\frac{1}{d} \bT^\top \bT$.  Define  for $1\leq b\leq d$,
\begin{align}
    \bz_b=\bg_{1,b}\otimes \cdots \otimes \bg_{\p-1,b}\in \R^{m^{\p-1}}.
\end{align}
Then $\bt_b=\bz_b\otimes g_{\p,b}$ and
\begin{align}
    (\bT^\top \bT)_{ij}&=\langle \bt_i,\bt_j\rangle=\langle \bz_{i}\otimes \bg_{\p,i},\bz_j \otimes \bg_{\p,j}\rangle=\langle \bz_i,\bz_j\rangle \langle  \bg_{\p,i},\bg_{\p,j}\rangle.
\end{align}
Write $\bD_i=\mathrm{diag}(\bg_{\p,1}(i),\dots, \bg_{\p,d}(i))$ for $1\leq i\leq m$ and $\B_{ij}=\langle \bz_i,\bz_j\rangle$. Then we have the following decomposition
\begin{align}\label{eq:T_conditional_decomp}
    \frac{1}{d} \bT^\top \bT=\frac{1}{d}\sum_{i=1}^m \bD_i \bB\bD_i.
\end{align}
Conditioned on $\bB$, $\frac{1}{d} \bT^\top \bT$ is a sum of $m$ independent random matrices.
    We first record two bounds for $\bB$. Let $\bZ=[\bz_1,\dots,\bz_d]$, so that $\bB=\bZ^\top\bZ$. Since
\[
    \bB_{bb}=\|\bz_b\|^2
    =\prod_{a=1}^{\p-1}\|\bg_{a,b}\|^2,
\]
and each factor has a $\chi_m^2$ distribution, chi-square concentration and a union bound give
\begin{equation}\label{eq:B_diagonal_bound}
    2^{-(\p-1)}m^{\p-1}\leq \bB_{bb}
    \leq 2^{\p-1}m^{\p-1},\qquad b\in[d],
\end{equation}
with probability at least  $1-C d e^{-cm}$.
The vectors $\bz_1,\dots,\bz_d$ are independent and isotropic in
$\mathbb R^{m^{\p-1}}$. To obtain an operator-norm bound, define
\[
    \widetilde{\bz}_b
    =\bz_b\mathbf 1_{\{\|\bz_b\|^2\leq
       2^{\p-1}m^{\p-1}\}}.
\]
The matrices $\widetilde{\bz}_b\widetilde{\bz}_b^\top$ are independent and positive semidefinite. Moreover,
\[
    \widetilde{\bz}_b\widetilde{\bz}_b^\top
    \preceq2^{\p-1}m^{\p-1}\m I,
    \qquad
    \mathbb E[\widetilde{\bz}_b\widetilde{\bz}_b^\top]\preceq\m I.
\]
The  matrix Chernoff bound \cite[Corollary 5.2]{tropp2012user} therefore yields
\[
    \mathbb P\left(
       \left\|\sum_{b=1}^d
       \widetilde{\bz}_b\widetilde{\bz}_b^\top\right\|>2d
       \right)
    \leq m^{\p-1}\exp\left(-c\frac{d}{m^{\p-1}}\right).
\]
On the event in \eqref{eq:B_diagonal_bound}, no column is truncated. Since
$\bZ\bZ^\top$ and $\bZ^\top\bZ=\bB$ have the same nonzero eigenvalues, we conclude that
\begin{equation}\label{eq:B_good_event}
    \|\bB\|\leq2d,\qquad
    2^{-(\p-1)}m^{\p-1}\leq \bB_{bb}
    \leq2^{\p-1}m^{\p-1}\quad(b\in[d])
\end{equation}
with  probability at least
\[
    1-C d e^{-cm}
    -m^{\p-1}\exp\left(-c\frac{d}{m^{\p-1}}\right)
    =1-O(d^{-10}).
\]
We now condition on the event \eqref{eq:B_good_event}. The matrices
$\bD_1,\dots,\bD_m$ are still independent. Set
\[
    \Delta=\operatorname{diag}(\bB_{11},\dots,\bB_{dd}).
\]
The entries of each $\bD_r$ are independent standard Gaussian variables, so
\[
    \mathbb E[\bD_r\bB\bD_r\mid\bB]=\Delta.
\]
Consequently, \eqref{eq:T_conditional_decomp} and \eqref{eq:B_good_event} give
\begin{equation}\label{eq:T_conditional_mean}
    \mathbb E[\bT^\top\bT\mid\bB]
    =m\Delta\succeq2^{-(\p-1)}m^\p\m I_d.
\end{equation}

It remains to show that $\bT^\top\bT$ does not fall far below its conditional mean. Define
\[
    \m X_r=\Delta-\bD_r\bB\bD_r,
    \qquad 1\leq r\leq m.
\]
Conditioned on $\bB$, the matrices $\m X_r$ are independent, self-adjoint, and centered. Moreover, because $\bD_r\bB\bD_r\succeq0$,
\begin{equation}\label{eq:X_upper_bound}
    \lambda_{\max}(\m X_r)
    \leq\|\Delta\|
    \leq2^{\p-1}m^{\p-1}.
\end{equation}
For the conditional variance, the Gaussian identities
$\mathbb E g^2=1$ and $\mathbb E g^4=3$ give
\begin{align*}
    \mathbb E[(\bD_r\bB\bD_r)^2\mid\bB]_{bb}
    &=\sum_{c=1}^d\bB_{bc}^2
      \mathbb E[(\bD_r)_{bb}^2(\bD_r)_{cc}^2]\\
    &=\sum_{c\neq b}\bB_{bc}^2+3\bB_{bb}^2
      =(\bB^2)_{bb}+2\bB_{bb}^2.
\end{align*}
We obtain
\[
    \mathbb E[\m X_r^2\mid\bB]
    =\operatorname{diag}\left((\bB^2)_{bb}+\bB_{bb}^2\right)_{b=1}^d.
\]
Since $\bB^2\preceq\|\bB\|\bB$, \eqref{eq:B_good_event} and
$m^{\p-1}\leq d$ imply
\begin{equation}\label{eq:X_variance_bound}
    \left\|\sum_{r=1}^m
       \mathbb E[\m X_r^2\mid\bB]\right\|
    \leq C d m^\p.
\end{equation}
Apply the one-sided matrix Bernstein inequality
\cite[Theorem 6.1(ii)]{tropp2012user}, using the bounds in
\eqref{eq:X_upper_bound} and \eqref{eq:X_variance_bound}.
Set $t=2^{-\p}m^\p$. Since
\[
    \sum_{r=1}^m\m X_r=m\Delta-\bT^\top\bT,
\]
the conditional mean bound \eqref{eq:T_conditional_mean} gives
\begin{align*}
    &\mathbb P\left(
       \lambda_d(\bT^\top\bT)<2^{-\p}m^\p
       \,\middle|\,\bB\right)\\
    &\qquad\leq
       d\exp\left(-\frac{c m^{2\p}}
       {d m^\p+m^{2\p-1}}\right)
       \leq d\exp\left(-c\frac{m^\p}{d}\right)
       =O(d^{-5}).
\end{align*}
In the second inequality we again used $m^{\p-1}\leq d$, while the final equality follows from $m^\p\asymp d\log^2d$. Combining this estimate with   \eqref{eq:B_good_event}, we have
\begin{equation}\label{eq:T_smallest_eigenvalue}
    \lambda_d\left(\frac1d\bT^\top\bT\right)
    \gtrsim\frac{m^\p}{d}
\end{equation}
with probability $1-O(d^{-5})$.
Finally, \eqref{eq:T_smallest_eigenvalue} and the positive semidefinite comparison \eqref{eq:PSD_compare_T} give
\begin{align*}
\lambda_d\left(\frac1d\m{Y}^{\m{1}^\p\top}\m{Y}^{\m{1}^\p}\right)
    &\gtrsim m^{-\alpha\p}\frac{m^\p}{d}
      =\frac{m^{\p(1-\alpha)}}d\asymp d^{-\alpha}\log^{-2(\alpha-1)}d.
\end{align*}
For $n_*\leq j\leq d$, we have $j\asymp d$ and
$\log(j+1)\asymp\log d$. Therefore,
\[
    \lambda_j\left(\frac1d\m{Y}^{\m{1}^\p\top}\m{Y}^{\m{1}^\p}\right)
    \geq\lambda_d\left(\frac1d\m{Y}^{\m{1}^\p\top}\m{Y}^{\m{1}^\p}\right)
    \gtrsim j^{-\alpha}\log^{-2\alpha}(j+1).
\]
Together with \eqref{eq:final_parta}, this proves the lower bound in
\eqref{eq:princpal_upper_lower_lower}. 
Combining it with the upper
bound from Step~3 and taking a union bound with the events in
Steps~1--2 proves \eqref{eq:princpal_upper_lower_lower}, with
probability $1-O(\log^{-4}(d))$, throughout the range
\[
    c_1d\log^{-(\p+1)}(d)<j\leq d.
\]
This finishes the proof.
\end{proof}

\subsection{Sub-leading terms and Wick chaos}

In this section, we focus on the remaining non-principal matrices of \eqref{eq:Krep}. We will refer to these matrices as lower order terms or sub-leading terms, and they consist of matrices of the form:
\begin{enumerate}
    \item $\frac{1}{d}\m{Y}^{\pi\top}\m{Y}^{\pi}$, if
    $C_\xi = \p$ and $\pi \neq \m{1}^\p$. 
    \item $\frac{1}{d}\D^{A_\xi}\m{Y}^{\pi \top}\m{Y}^{\pi} \D^{A_\xi}$, if $C_\xi \neq \p$ and $A_\xi \geq 1$.
\end{enumerate}
Using \eqref{eq:W^pi_K^pi}, we can write the sample data matrix in both cases as \[\m{Y}^\pi \D^{A_\xi} =  \sqrt{\cov^\pi}\W^\pi \D^{A_\xi}.\] Notice that in contrast to the principal term, when $\m{Y}^{\m{1}^\p} =\sqrt{\cov^{\m{1}^\p}}\W^{\m{1}^\p}$, the entries of $\m{Y}^{\pi} \D^{A_\xi}$ terms are not uncorrelated. Implying that even after applying our block-wise strategy, $\frac{1}{d}\D^{A_\xi}\m{Y}^{\pi \top}\m{Y}^{\pi} \D^{A_\xi}$ would not concentrate to the diagonal covariance $\cov^\pi$. Therefore, we will need to uncorrelate the lower order terms before applying our block-wise strategy. To do so, we will utilize the Wick product decomposition.

\begin{definition}\label{def:wick_prod}
    Let $\m{P}_n$ be the set of all Gaussian polynomials of degree less than or equal to $n$ and $\m{G}^{:n:}=  \overline{\m{P}}_n \cap \overline{\m{P}}_{n-1}^{\perp}$, where the closure and orthogonal complement are taken with respect to the $L^2$ inner-product. If $g_1,\dots,g_n$ are centered joint Gaussians, then the Wick product denoted by $\wick{g_1\cdots g_n}$ is defined as the orthogonal projection of $\prod_{i=1}^n g_i$ onto $ \m{G}^{:n:}$, i.e.,
    \begin{equation}
        \wick{g_1\cdots g_n} \equiv \pi_n(g_1 \cdots g_n).
    \end{equation}
\end{definition}
For an overview on Wick products, see \cite{janson1997gaussian} or the appendix of \cite{paquette2026power}. There are two main advantages of working with Wick products. 
First, one can perform a change of basis to move between the original Gaussian product and the orthogonal projection. See \cref{lemma:wick_decomp} and \cref{lemma:prod_Wick_decomp}. This transformation can be systematically tracked by graphical representations known as Feynman diagrams:

\begin{definition}\label{def:feynman}
    A rank-$r$ Feynman diagram of a collection of random variables $g_1,\dots,g_n$, is a graph $\gamma$ consisting of $n$ vertices labeled $g_1,\dots, g_n$ with $r$ edges that are vertex disjoint. Let $\mathcal{P}(\gamma)$ be the set of all paired vertices and $\mathcal{U}(\gamma)$ the set of all unpaired vertices, we define the value of $\gamma$ to be
    \begin{equation}
        v(\gamma) = \prod_{(r,s) \in \mathcal{P}(\gamma)}\E \left[ g_r g_s\right] \prod_{i\in \mathcal{U}(\gamma)} g_i.
    \end{equation}
    We say that a Feynman diagram is complete if the number of edges is $n/2$, i.e., $\mathcal{U}(\gamma) = \emptyset$. 
\end{definition}
 This implies that the machinery we used in the proof of the principal case should carry over nicely after projecting $\m{Y}^\pi \D^{A_\xi} $. The second advantage is that the Wick matrices formed from the orthogonal projection of $\m{Y}^{\pi} \D^{A_\xi} $ 
inherently possess diagonal covariance, thus solving the correlation problem. See \cref{lemma:wick_cor}. We illustrate the Wick decomposition method in the following lemma:
\begin{lemma}\label{lemma:lower_order_wick_decomp}
\reb{There exists a rank-$1$ matrix $\m{U}_{A_\xi}\in \R^{d \times d}$} with $O(1)$ trace, along with a finite collection of index sets  $\mathcal{I}_\theta$ and Wick matrices $\wick{\m{Y}^{\theta}} \in \R^{ | \mathcal{I}_\theta| \times d}$ such that
\begin{equation}
    \reb{\frac{1}{d}}\D^{A_\xi} \m{Y}^{\pi\top}\m{Y}^{\pi}\D^{A_\xi} \preceq  \reb{\frac{1}{d}} \sum_{\theta} \wick{\m{Y}^{\theta}}^{\top} \wick{\m{Y}^{\theta}} +\reb{\m{U}_{A_\xi}}.
\end{equation}
\end{lemma}

\begin{proof}
    Using \eqref{eq:inner_prod_power_p0}, we may expand the entry of $\D^{A_\xi}$ to be
\begin{align}
    \D^{A_\xi}_{a,a} &= \inner{\m{y}^{(a)}}{\m{y}^{(a)}}^{A_\xi} = \sum_{r=1}^{A_\xi} \sum_{\psi} \sum_{s_1<\cdots<s_r} \reb{C_{\psi}} \, \m{y}^{(a),2\psi_1}_{s_1} \dots \m{y}^{(a),2\psi_r}_{s_r}\\
    &= \sum_{r=1}^{A_\xi} \sum_{\psi}\sum_{\m{s}} \reb{C_{\psi}}\,\prod_{t=1}^r \m{y}^{(a),2\psi_t}_{s_t},
\end{align}
where $\psi$ is a composition of $A_\xi$ of size $r$, \reb{$C_{\psi} = \frac{ \left(\sum_{i=1}^r \psi_i\right)!}{\psi_1!\dots \psi_r!}$} and $\m{s}=(s_1<\cdots<s_r)$ an ordered tuple. Combining this with $\m{Y}^{\pi}$ gives us two sets of indices, $\m{i} = (i_1<\cdots<i_\ell)$ and $\m{s} =(s_1<\cdots<s_r)$, to track. In particular, 
 \begin{align}\label{eq:general_YD_ia}
     \left( \m{Y}^\pi \D^{A_\xi} \right)_{\m{i},a} &= \sum_{r=1}^{A_\xi} \sum_{\psi} \sum_{\m{s}} \reb{C_{\psi}} \, \prod_{t=1}^r \m{y}^{(a),2\psi_t}_{s_t} \prod_{j=1}^\ell \m{y}^{(a), \pi_j}_{i_j}.
 \end{align}
Notice if $\m{i}$ and $\m{s}$ overlap, then for some $j\in [\ell]$, $\m{y}^{(a)}_{i_j}$ will be raised to a higher power than $\pi_j$. Therefore, before proceeding we must consider all potential overlap scenarios.

The following construction will be used to derive a new composition that accounts for overlaps between $\m{i}$ and $\m{s}$. Define the overlap function $\Theta:[\ell] \times [r]  \to \{0,1 \}$
such that if $\Theta(j,t) = 1$ then $\Theta(j',t) = 0$ for all $j'\neq j$. 
That is, for any ordered tuple pair $(\m{i},\m{s})$, there will exist an overlap function $\Theta$ that encodes the overlaps in $(\m{i},\m{s})$ by the rule $\Theta(j,t) = 1$ if and only if $i_j = s_t$. We will write $s_t \not\in \m{i}$ if $s_t$ does not overlap into $\m{i}$. Let $n = \abs{\Theta^{-1}(\{1\})}$
be the number of overlaps. Then, there are $(r-n)$ many $s_t$ that do not overlap into $\m{i}$. Let $\Phi:\{1,\cdots,r-n \} \to \{ t\,; \,\Theta(j,t) = 0 \, \forall j \in [\ell] \}$ be a bijective and monotonic increasing map. We may interpret the image of $\Phi$ as the remaining indices of $\m{s}$ that do not overlap into $\m{i}$.
Now, given an overlap function $\Theta$, there are $\ell+(r-n)$ unique indices. To assign the appropriate power to each index, we define a new composition $\rho = (\rho_1,\cdots, \rho_\ell,\rho_{\ell+1},\dots,\rho_{\ell+(r-n)})$ by $\rho_j = \pi_j+ \sum_{t=1}^r 2\psi_t \cdot \Theta(t,j)$ for $j\in [\ell]$ and $\rho_{\ell+t} = 2\psi_{\Phi(t)}$ for $t\in [r-n]$. Fixing an $\m{i}$, there is a bijection from $(\m{i},\m{s})$ to the set of all
pairs $\left(\rho, (s_1<\cdots<s_{r-n}) \right)$, where $\rho$ is a composition as described above and $s_t \not \in \m{i}$ for all $t \in [r-n]$.

\begin{exmp}{(Composition $\rho$)}
    If $r = 1$, then there can either be a single overlap between $\m{i}$ and $\m{s}$ or no overlaps. As such,
    there are only $\ell+1$ possible overlap functions; the overlap function that is identically $0$ and the overlap function that maps a single pair $(j,t)$ to $1$ and everything else to $0$.
    In the case of the former, $\Phi$ is a singleton map. 
    Therefore, the set of all possible new compositions $\rho$ are the single overlap cases $(\pi_1+2\psi_1, \cdots, \pi_\ell),\, \cdots,\,(\pi_1,\cdots,\pi_\ell+2\psi_1)$ and the no overlap case $(\pi_1,\cdots,\pi_\ell,2\psi_1)$. 
\end{exmp}
\noindent Therefore, we may re-index the sum that appears in \eqref{eq:general_YD_ia} to
\begin{align}
    \sum_{\m{s}} \prod_{j=1}^\ell \m{y}^{(a), \pi_j}_{i_j}  \prod_{t=1}^r  \m{y}^{(a),2\psi_t}_{s_t}  = \sum_{\rho} \sum_{\substack{s_1<\cdots<s_{r-n}\\s_t \not \in \m{i}}}  \prod_{j=1}^\ell \m{y}^{(a),\rho_j}_{i_j} \prod_{t=1}^{r-n} \m{y}^{(a),\rho_{\ell+t}}_{s_t}. \label{eq:distinct_s}
\end{align}
Going forward, we work only with the expression on the right hand side of \eqref{eq:distinct_s}, hence when we write $(\m{i},\m{s})$ it will be understood that $\m{s} = (s_1<\cdots<s_{r-n})$ and that $\m{s}$ and $\m{i}$ do not overlap. \reb{Note that $\m{i}$ and $\m{s}$ are internally ordered within their respective tuples, but no ordering is imposed between the entries of $\m{i}$  and $\m{s}$.} In the case that we fix an $\m{i}$, we will write $\m{s} \cap \m{i} = \emptyset$ to emphasize that $\m{s}$ and $\m{i}$ do not overlap. By \cref{lemma:prod_Wick_decomp}, the  product of Gaussians may be decomposed into a sum of Wick products, i.e.,
\begin{equation}\label{eq:subleading_s_t_product}
     \prod_{j=1}^\ell \m{y}^{(a),\rho_j}_{i_j} \prod_{t=1}^{r-n} \m{y}^{(a),\rho_{\ell+t}}_{s_t} = \sum_\gamma \prod_{(v_1,v_2)\in \mathcal{P}(\gamma)}\E \left[ \m{y}^{(a)}_{v_1} \m{y}^{(a)}_{v_2} \right] \wick{ \prod_{u\in \mathcal{U)(\gamma)}} \m{y}^{(a)}_u },
\end{equation}
with sum taken across all Feynman diagrams with vertices labeled by

\noindent $\{i_1,\dots,i_\ell,s_{\ell+1},\dots,s_{\ell+ r-n}\}$.
Since $\m{y}^{(a)}_{i_j}$ and $\m{y}^{(a)}_{s_t}$ are pairwise independent, the only non-zero Feynman diagrams are those that connect vertices with other vertices that share the same label. That is, $\m{y}^{(a)}_{i_j}$ and $\m{y}^{(a)}_{s_t}$ can only be paired with another copy of itself. As such, we may map each Feynman diagram to a tuple \[\eta = (\eta_1,\dots, \eta_{\ell},\eta_{\ell+1},\dots \eta_{\ell+(r-n)}),\] where $\eta_j$ for $j\in [\ell]$ and $\eta_{\ell+t}$ for $t \in [r-n]$ corresponds respectively to the number of $\m{y}^{(a)}_{i_j}$ and $\m{y}^{(a)}_{s_t}$ pairs formed. 
Note, this mapping is not injective, so we will denote $N_\eta$ to be the number of Feynman diagrams mapped to $\eta$ (see \cref{fig:wick1}). Therefore,  \eqref{eq:subleading_s_t_product} becomes
\begin{align}
    & \prod_{j=1}^\ell \m{y}^{(a),\rho_j}_{i_j} \prod_{t=1}^{r-n} \m{y}^{(a),\rho_{\ell+t}}_{s_t}  \\
    =& \sum_{\eta}N_\eta \prod_{j=1}^\ell \cov_{i_j}^{\eta_j}   \prod_{t=1}^{r-n} \cov_{s_t}^{\eta_{\ell+t}} \wick{ \prod_{j=1}^\ell \m{y}^{(a),\rho_j-2\eta_j}_{i_j} \prod_{t=1}^{r-n} \m{y}^{(a),\rho_{\ell+t} - 2\eta_{\ell+t}}_{s_t} }.
\end{align}

\begin{figure}[h] 
\centering

\begin{tikzpicture}[point/.style={circle,inner sep=1.2pt,fill=black}]


\node[point,blue, thick, draw=blue, fill=none]  (v11) at (0,2) {$i_1$};
\node[point,  blue, thick, draw=blue, fill=none] (v12) at (1,2) {$i_1$};

\node[point, blue, draw=blue, thick, fill=none] (v21) at (0,1) {$i_1$};
\node[point,blue, draw = blue,thick, fill = none] (v22) at (1,1) {$i_1$};

\node[point, red, draw = red, thick, fill = none] (v31) at (0,0) {$i_2$};
\node[point,red,draw = red, thick, fill = none] (v32) at (1,0) {$i_2$};

\node[point, green, draw = green, thick, fill = none] (v41) at (0,-1) {$t_1$};
\node[point,green,draw = green, thick, fill = none] (v42) at (1,-1) {$t_1$};

\draw [very thick](v11) -- (v12);
\draw [very thick](v31) -- (v32);

\node at (0.5,-1.7) {$\gamma_1$};

\end{tikzpicture}
\hspace{0.5cm}
\begin{tikzpicture}[point/.style={circle,inner sep=1.2pt,fill=black}]

\node[point,blue, thick, draw=blue, fill=none]  (v11) at (0,2) {$i_1$};
\node[point,  blue, thick, draw=blue, fill=none] (v12) at (1,2) {$i_1$};

\node[point, blue, draw=blue, thick, fill=none] (v21) at (0,1) {$i_1$};
\node[point,blue, draw = blue,thick, fill = none] (v22) at (1,1) {$i_1$};

\node[point, red, draw = red, thick, fill = none] (v31) at (0,0) {$i_2$};
\node[point,red,draw = red, thick, fill = none] (v32) at (1,0) {$i_2$};

\draw[very thick] (v21) -- (v22);
\draw[very thick] (v31) -- (v32);

\node[point, green, draw = green, thick, fill = none] (v41) at (0,-1) {$t_1$};
\node[point,green,draw = green, thick, fill = none] (v42) at (1,-1) {$t_1$};

\node at (0.5,-1.7) {$\gamma_2$};
\end{tikzpicture}
\hspace{0.5cm}
\begin{tikzpicture}[point/.style={circle,inner sep=1.2pt,fill=black}]

\node[point,blue, thick, draw=blue, fill=none]  (v11) at (0,2) {$i_1$};
\node[point,  blue, thick, draw=blue, fill=none] (v12) at (1,2) {$i_1$};

\node[point, blue, draw=blue, thick, fill=none] (v21) at (0,1) {$i_1$};
\node[point,blue, draw = blue,thick, fill = none] (v22) at (1,1) {$i_1$};

\node[point, red, draw = red, thick, fill = none] (v31) at (0,0) {$i_2$};
\node[point,red,draw = red, thick, fill = none] (v32) at (1,0) {$i_2$};

\node[point, green, draw = green, thick, fill = none] (v41) at (0,-1) {$t_1$};
\node[point,green,draw = green, thick, fill = none] (v42) at (1,-1) {$t_1$};

\draw[very thick] (v11) -- (v22);
\draw[very thick] (v31) -- (v32);

\node at (0.5,-1.7) {$\gamma_3$};
\end{tikzpicture}
\hspace{0.5cm}
\begin{tikzpicture}[point/.style={circle,inner sep=1.2pt,fill=black}]

\node[point,blue, thick, draw=blue, fill=none]  (v11) at (0,2) {$i_1$};
\node[point,  blue, thick, draw=blue, fill=none] (v12) at (1,2) {$i_1$};

\node[point, blue, draw=blue, thick, fill=none] (v21) at (0,1) {$i_1$};
\node[point,blue, draw = blue,thick, fill = none] (v22) at (1,1) {$i_1$};

\node[point, red, draw = red, thick, fill = none] (v31) at (0,0) {$i_2$};
\node[point,red,draw = red, thick, fill = none] (v32) at (1,0) {$i_2$};

\node[point, green, draw = green, thick, fill = none] (v41) at (0,-1) {$t_1$};
\node[point,green,draw = green, thick, fill = none] (v42) at (1,-1) {$t_1$};

\draw[very thick] (v21) -- (v12);
\draw[very thick] (v31) -- (v32);

\node at (0.5,-1.7) {$\gamma_4$};
\end{tikzpicture}
\hspace{0.5cm}
\begin{tikzpicture}[point/.style={circle,inner sep=1.2pt,fill=black}]

\node[point,blue, thick, draw=blue, fill=none]  (v11) at (0,2) {$i_1$};
\node[point,  blue, thick, draw=blue, fill=none] (v12) at (1,2) {$i_1$};

\node[point, blue, draw=blue, thick, fill=none] (v21) at (0,1) {$i_1$};
\node[point,blue, draw = blue,thick, fill = none] (v22) at (1,1) {$i_1$};

\node[point, red, draw = red, thick, fill = none] (v31) at (0,0) {$i_2$};
\node[point,red,draw = red, thick, fill = none] (v32) at (1,0) {$i_2$};

\node[point, green, draw = green, thick, fill = none] (v41) at (0,-1) {$t_1$};
\node[point,green,draw = green, thick, fill = none] (v42) at (1,-1) {$t_1$};

\draw[very thick] (v11) -- (v21);
\draw[very thick] (v31) -- (v32);

\node at (0.5,-1.7) {$\gamma_5$};
\end{tikzpicture}
\hspace{0.5cm}
\begin{tikzpicture}[point/.style={circle,inner sep=1.2pt,fill=black}]

\node[point,blue, thick, draw=blue, fill=none]  (v11) at (0,2) {$i_1$};
\node[point,  blue, thick, draw=blue, fill=none] (v12) at (1,2) {$i_1$};

\node[point, blue, draw=blue, thick, fill=none] (v21) at (0,1) {$i_1$};
\node[point,blue, draw = blue,thick, fill = none] (v22) at (1,1) {$i_1$};

\node[point, red, draw = red, thick, fill = none] (v31) at (0,0) {$i_2$};
\node[point,red,draw = red, thick, fill = none] (v32) at (1,0) {$i_2$};

\node[point, green, draw = green, thick, fill = none] (v41) at (0,-1) {$t_1$};
\node[point,green,draw = green, thick, fill = none] (v42) at (1,-1) {$t_1$};
\draw[very thick] (v12) -- (v22);
\draw[very thick] (v31) -- (v32);

\node at (0.5,-1.7) {$\gamma_6$};
\end{tikzpicture}

    \caption{If $\ell = 2$, $r-n=1$ and
    $\rho = (4,2,2)$, then there are $4$-copies of $i_1$, $2$-copies of $i_2$ and $2$-copies of $t_1$. The Feynman diagrams $\gamma_1, \cdots, \gamma_6$ are all mapped to $\eta = (1,1,0)$. Which is to say that they all 
    yield the same Wick product $\cov_{i_1}\cov_{i_2}\wick{\m{y}^{(a),2}_{i_1} \m{y}^{(a),2}_{t_1}}$. Thus, $N_{(1,1,0)}=6$.}
    \label{fig:wick1}
\end{figure}
Putting all this together,
\begin{align}
    & \left( \m{Y}^\pi \D^{A_\xi} \right)_{\m{i},a}  \\ \quad &= \sum_{r=1}^{A_\xi} \reb{C_{\psi}} \sum_{\psi,\rho} \sum_{ \substack{ \m{s} \cap \m{i}=\emptyset} }  \sum_{\eta}N_\eta \prod_{j=1}^\ell \cov_{i_j}^{\eta_j} \prod_{t=1}^{r-n} \cov_{s_t}^{\eta_{\ell+t}}  \wick{\prod_{j=1}^\ell \m{y}^{(a),\rho_j-2\eta_j}_{i_j}  \prod_{t=1}^{r-n} \m{y}^{(a),\rho_{\ell+t} - 2\eta_{l+t}}_{s_t} }. \label{eq:q>0_Y_ia_0}
\end{align} 
Notice that the contributing terms to the Wick product are determined only by the sub-indices $j$ and $t$ such that $\rho_j -2\eta_j \neq 0$ and $\rho_{\ell+t}-2\eta_{\ell+t}\neq 0$. 
Let $\mathcal{W}_{\rho,\eta}^1 = \{j \in [\ell]\,;\, \rho_j -2\eta_j \neq 0\}$ and $\mathcal{W}_{\rho,\eta}^2 = \{t \in [r-n] \,; \rho_{\ell+t}-2\eta_{\ell+t}\neq 0\}$ be the contributing Wick sub-indices. Define the projection of $(\m{i},\m{s})$ onto $\mathcal{W}_{\rho,\eta}^1$ and $\mathcal{W}_{\rho,\eta}^2$ respectively by $P_{\mathcal{W}_{\rho,\eta}}(\m{i},\m{s}) = \left( (i_j)_{j\in \mathcal{W}_{\rho,\eta}^1}, (s_t)_{t \in \mathcal{W}_{\rho,\eta}^2} \right)$. Let $\mathcal{I}_{\rho,\eta} = \text{Image}\left( P_{\mathcal{W}_{\rho,\eta}}\right)$. Which is to say that $(\tilde{\m{i}},\tilde{\m{s}}) \in \mathcal{I}_{\rho,\eta}$ if and only if there exists an ordered tuple pair $(\m{i},\m{s})$ such that $\tilde{\m{i}}= (i_j)_{j \in \mathcal{W}^1_{\rho,\eta}}$ and $\tilde{\m{s}}= (s_t)_{t\in \mathcal{W}_{\rho,\eta}^2}$. On the other hand, given $(\tilde{\m{i}},\tilde{\m{s}}) \in \mathcal{I}_{\rho,\eta}$, let $\mathcal{U}_{\rho,\eta}(\tilde{\m{i}},\tilde{\m{s}})$ be the set of all ordered tuple pairs  $\left((i_j)_{j \not\in \mathcal{W}^1_{\rho,\eta}}, (s_t)_{t \not \in \mathcal{W}_{\rho,\eta}^2} \right)$ such that the concatenation of $(\tilde{\m{i}},\tilde{\m{s}})$ and $\left((i_j)_{j \not\in \mathcal{W}^1_{\rho,\eta}}, (s_t)_{t \not \in \mathcal{W}_{\rho,\eta}^2} \right)$ form an ordered tuple pair $(\m{i},\m{s})$.   

\begin{exmp}{(General restricted ordered tuple $\mathcal{I}_{\rho,\eta}$)}
If $r = 2$, $\mathcal{W}_{\rho,\eta}^1 = \{2,4\}$ and $\mathcal{W}_{\rho,\eta}^2 = \{2\}$. Then, $((i_2,i_4),(s_2))$ is in $\mathcal{I}_{\rho,\eta}$ provided that $i_2, s_2 \neq 1$ and $i_4> i_2+1$. 
\end{exmp}
\noindent By separating sub-indices into contributing and non-contributing terms within the Wick product, we may write the summand of \eqref{eq:q>0_Y_ia_0} as
\begin{align}
     &\reb{\prod_{j=1}^\ell \cov_{i_j}^{\eta_j}}  \prod_{t=1}^{r-n} \cov_{s_t}^{\eta_{\ell+t}}  \wick{ \prod_{j=1}^\ell \m{y}^{(a),\rho_j-2\eta_j}_{i_j}  \prod_{t=1}^{r-n} \m{y}^{(a),\rho_{\ell+t} - 2\eta_{\ell+t}}_{s_t} } \\
     &\quad = \prod_{j=1}^\ell \cov_{i_j}^{\eta_j} \prod_{t=1}^{r-n} \cov_{s_t}^{\eta_{\ell+t}}   \wick{ \prod_{j\in \mathcal{W}_{\rho,\eta}^1} \m{y}^{(a),\rho_j-2\eta_j}_{i_j} \prod_{t\in \mathcal{W}_{\rho,\eta}^2} \m{y}^{(a),\rho_{\ell+t} - 2\eta_{\ell+t}}_{s_t} }. \label{eq:wick_summand}
\end{align}
\reb{
We distinguish three possible classes of Wick products;
\begin{enumerate}
    \item The pure Wick product, corresponding to when $\eta = 0.$
    \item The mean term, corresponding to when $2 \eta = \rho$.
    \item The mixed Wick product, corresponding to when $\eta \neq 0$ and $2 \eta \neq \rho$.
\end{enumerate}
Given that Wick products have mean $0$, we see that the case $2\eta = \rho$ preserves the mean of $\left( \m{Y}^\pi \m{D}^{A_\xi} \right)_{\m{i},a}$, hence the name. Naturally, if $\p$ is odd then the mean term is $0$.
}
\reb{Recall that $\m{y}^{(a)} = \sqrt{\cov}\W_{\col(a)}$ and $\rho_{\ell+t}=2\psi_{\Phi(t)}$, thus for all $t \in \mathcal{W}^2_{\rho,\eta}$, $\rho_{\ell+t}- 2\eta_{\ell+t} \geq 2$. It will be helpful, for reasons to be made clear later, to pull out $\sqrt{\cov}$ from a copy of $\y_{s_t}$ in the case of pure and mixed Wick products. That is, for all $2\eta \neq \rho$,
\begin{align} \label{eq:cov_wick}
    &\wick{ \prod_{j\in \mathcal{W}_{\rho,\eta}^1} \m{y}^{(a),\rho_j-2\eta_j}_{i_j} \prod_{t\in \mathcal{W}_{\rho,\eta}^2} \m{y}^{(a),\rho_{\ell+t} - 2\eta_{\ell+t}}_{s_t} } \\
    &\quad = \prod_{t \in \mathcal{W}^2_{\rho,\eta}} \sqrt{\cov_{s_t}}  \wick{ \prod_{j\in \mathcal{W}_{\rho,\eta}^1} \m{y}^{(a),\rho_j-2\eta_j}_{i_j} \prod_{t\in \mathcal{W}_{\rho,\eta}^2} \m{y}^{(a),\rho_{\ell+t} - 2\eta_{\ell+t} -1 }_{s_t} \W_{s_t,a}.
    }
\end{align}
}
For notational compactness, we will \reb{denote $\Gamma$ and $\Gamma_M$ to be the admissible tuples $(r,\psi,\rho,\eta)$ where $2\eta \neq \rho$ and $2\eta = \rho$ respectively. For $\Gamma$ tuples, we define the Wick coefficients to be}
\begin{equation} \label{eq:C_Gamma}
    C_{\Gamma}(\m{i},\m{s}) = C_{r,\psi,\rho,\eta}(\m{i},\m{s}) 
    = C_\psi N_\eta \prod_{j=1}^{\ell} \cov_{i_j}^{\eta_j}\prod_{t=1}^{r-n}\cov_{s_t}^{\eta_{\ell+t}} \reb{\prod_{t\in \mathcal{W}_{\rho,\eta}^2} \sqrt{\cov_{s_t}}},
\end{equation}
\reb{and  for $\Gamma_M$ tuples, we define the mean coefficients to be 
\begin{equation}
     C_{\Gamma_M}(\m{i},\m{s})=C_{r,\psi,\rho,\eta}(\m{i},\m{s})= C_{\psi} N_\eta  \prod_{j=1}^\ell \cov_{i_j}^{\eta_j} \prod_{t=1}^{r-n}\cov_{s_t}^{\eta_{\ell+t}} .
\end{equation}}
Next, \reb{if $2\eta \neq \rho$,} define $\wick{ \m{Y}^{\Gamma} }\in \R^{\abs{\mathcal{I}_{\rho,\eta}}\times d}$ by
\begin{equation}\label{eq:wick_Y_Gamma}
\wick{\m{Y}^{\Gamma}}_{(\tilde{\m{i}},\tilde{\m{s}}),a}=\wick{ \prod_{j\in \mathcal{W}_{\rho,\eta}^1} \m{y}^{(a),\rho_j-2\eta_j}_{i_j} \prod_{t\in \mathcal{W}_{\rho,\eta}^2} \m{y}^{(a),\rho_{\ell+t} - 2\eta_{\ell+t} \reb{-1 } }_{s_t} \reb{ \W_{s_t,a}} }.
\end{equation}
We may now write \eqref{eq:q>0_Y_ia_0} as
\begin{align}\label{eq:q>0_Y_ia_0_2}
    \left( \m{Y}^\pi \D^{\reb{A_\xi}} \right)_{\m{i},a} &= \reb{ \sum_{\Gamma} \sum_{\substack{  \m{s} \cap \m{i}=\emptyset } }  C_{\Gamma}(\m{i},\m{s})  \wick{\m{Y}^{\Gamma}}_{(\tilde{\m{i}},\tilde{\m{s}}),a}+\sum_{\Gamma_M} \sum_{\m{s}\cap\m{i}=\emptyset} C_{\Gamma_M}(\m{i},\m{s}).}
\end{align}
where it is understood that $(\ti,\ts) = P_{\mathcal{W}_{\rho,\eta}}(\m{i},\m{s}) \in \mathcal{I}_{\rho,\eta}$.

It will be convenient to define the quadratic $Q:\R^d \to \R$ by $\x \mapsto \left( \sum_{a}\x_a \right)^2$. Then, taking the quadratic form of $\frac{1}{d}\D^{A_\xi}\m{Y}^{\pi \top}\m{Y}^{\pi}\D^{A_\xi}$, we see that
\begin{align}
    \reb{\frac{1}{d}}\x^\top \left( \D^{A_\xi} \m{Y}^{\pi\top} \m{Y}^{\pi} \D^{A_\xi} \right)\x &= \reb{\frac{1}{d}} \sum_{a,b}\x_a \x_b \sum_{\m{i}}  \left( \m{Y}^\pi \D^{A_\xi} \right)_{\m{i},a} \left( \m{Y}^\pi \D^{A_\xi}\right)_{\m{i},b}\\
    &=\reb{\frac{1}{d}} \sum_{\m{i}} Q\left( \x \odot \left( \m{Y}^\pi \D^{A_\xi} \right)_{\m{i}} \right),
\end{align}
where $ \left( \m{Y}^\pi \D^{A_\xi} \right)_{\m{i}}$ is the $\m{i}$-th \reb{row} of  $\m{Y}^\pi \D^{A_\xi}$. 
By \eqref{eq:q>0_Y_ia_0_2}, we may expand $\x \odot \left( \m{Y}^\pi \D^{A_\xi} \right)_{\m{i}}$ into a linear combination of \reb{Hadamard}  products,
\begin{align}
    \x \odot \left( \m{Y}^\pi \D^{A_\xi} \right)_{\m{i}} &= \sum_{\Gamma}\sum_{\substack{  \m{s} \cap \m{i}=\emptyset }} C_{\Gamma}(\m{i},\m{s}) \left( \x \odot \left( \wick{\m{Y}^{\Gamma}} \right)_{(\ti,\ts)} \right)\\
    &\quad \reb{+ \x \sum_{\Gamma_M} \sum_{\substack{  \m{s} \cap \m{i}=\emptyset }} C_{\Gamma_M} (\m{i},\m{s}). } \label{eq:quad_row}
\end{align}
\reb{
Plugging this into $Q$, we see that
\begin{align}
    & Q\left(\x \odot \left( \m{Y}^\pi \D^{A_\xi} \right)_{\m{i}} \right)\\ &\quad
    = \left( \sum_{\Gamma} \sum_{\m{s}\cap \m{i}=\emptyset} C_{\Gamma} (\m{i},\m{s}) \left(\sum_a \x_a \wick{\m{Y}^{\Gamma}}_{(\ti,\ts),a} \right)  + \sum_{ \Gamma_M} \sum_{ \m{s}\cap \m{i}= \emptyset} C_{\Gamma_M}(\m{i},\m{s})  \left(   \sum_a \x_a \right)\right)^2.
\end{align}
Before proceeding, let us look at the summability of the coefficients in the three possible cases: $C_\Gamma$ with $\eta \neq 0$ (mixed Wick product), $C_\Gamma$ with $\eta = 0$ (pure Wick product), and $\Gamma_M$ (mean). In the case that $\eta \neq 0$, we may decompose the Wick coefficients into
\begin{equation}
    C_\Gamma(\m{i},\m{s}) =\underbrace{C_\psi N_\eta \prod_{j=1}^\ell \cov_{i_j}^{\eta_j} \prod_{t \not\in \mathcal{W}_{\rho,\eta}^2} \cov_{s_t}^{\eta_{\ell+t}} }_{\lambda_\Gamma(\m{i},\m{s})} \underbrace{ \prod_{t \in \mathcal{W}_{\rho,\eta}^2} \cov_{s_t}^{\eta_{\ell+t}+1/2}}_{\mu_\Gamma(\ts)},
\end{equation}
where $\ts = (s_t)_{t\in \mathcal{W}_{\rho,\eta}^2}$. Recall that $t \in \mathcal{W}^2_{\rho,\eta}$ if and only if $\rho_{\ell+t}\neq 2\eta_{\ell+t} $. 
Thus, we see that $\eta_{\ell+t}\neq 0$ for all $t \not\in \mathcal{W}_{\rho,\eta}^2$. Since $\cov_j \asymp j^{-\pwrco}$, it follows that $\sqrt{\lambda_\Gamma(\m{i},\m{s})} \mu_\Gamma(\ti) $ is square-summable on $\{\m{s};\m{s}\cap \m{i} = \emptyset\}$. If $\eta = 0$, then $\mathcal{W}^{2}_{\rho,\eta} = [r-n]$ and $C_\Gamma = C_\psi N_\eta \prod_{t=1}^{r-n}\sqrt{\cov_{s_t}}$. So, the coefficients of the pure Wick product are also square-summable on $\{\m{s};\m{s}\cap \m{i} = \emptyset\}$. This is precisely the reason why we factored out $\sqrt{\cov}$ in \eqref{eq:cov_wick}. Lastly, the mean coefficients correspond to $2\eta = \rho$, thus $\eta_k \neq 0$ for all $k\in [\ell+r-n]$ and we have that $C_{\Gamma_M}$ is summable on $ (\m{i},\m{s})$. Putting all this together, we see that
\begin{equation}\label{eq:CS_constant}
    \sum_{\Gamma,\eta \neq 0} \sum_{\m{s}\cap \m{i} = \emptyset} \lambda_\Gamma(\m{i},\m{s})\mu_\Gamma(\ts)^2 + \sum_{\Gamma,\eta = 0} \sum_{\m{s}\cap \m{i} = \emptyset} C_\Gamma(\m{i},\m{s})^2 + \sum_{\Gamma_M}\sum_{\m{s}\cap \m{i} =\emptyset}C_{\Gamma_M}(\m{i},\m{s}),
\end{equation}
may be bounded independently of $\m{i}$. Combining this with the Cauchy-Schwarz inequality,
\begin{align}
    & Q\left(\x \odot \left( \m{Y}^\pi \D^{A_\xi} \right)_{\m{i}} \right)\\ 
    &\quad \lesssim  \sum_{\Gamma, \eta \neq 0 } \sum_{\m{s}\cap \m{i}=\emptyset} \lambda_\Gamma(\m{i},\m{s}) Q\left(  \x \odot \wick{\m{Y}^{\Gamma}}_{(\ti,\ts)} \right) + \sum_{\Gamma, \eta =0} \sum_{\m{s}\cap \m{i} = \emptyset}   Q\left(  \x \odot \wick{\m{Y}^{\Gamma}}_{(\ti,\ts)} \right) \\
    &\quad \quad + \sum_{\Gamma_M} \sum_{\m{s}\cap \m{i}=\emptyset} C_{\Gamma_M}(\m{i},\m{s}) Q(\x). 
\end{align}
}
\reb{
Plugging this into \eqref{eq:quad_row}, we obtain
\begin{align}
    \frac{1}{d} \x^\top \left( \D^{A_\xi} \m{Y}^{\pi\top} \m{Y}^{\pi} \D^{A_\xi} \right)\x &\lesssim  \frac{1}{d} \sum_{\Gamma,\eta \neq 0} \sum_{(\m{i},\m{s}) } \lambda_{\Gamma}(\m{i},\m{s})Q \left( \x \odot \wick{\m{Y}^\Gamma}_{(\ti,\ts)} \right) \\
    &\quad + \frac{1}{d}\sum_{\Gamma, \eta=0} \sum_{(\m{i},\m{s}) } Q \left( \x \odot \wick{\m{Y}^\Gamma}_{(\ti,\ts)} \right)  + \frac{1}{d} \sum_{\Gamma_M}\sum_{(\m{i},\m{s})} C_{\Gamma_M}(\m{i},\m{s}) Q(\x).
\end{align}
As $C_{\Gamma_M}(\m{i},\m{s})$ is summable, the latter sum is equivalent to the quadratic of a rank-$1$ matrix with $O(1)$ trace. We will denote this term by $\m{U}_{A_\xi}$. Moreover,
we observe that the Wick quadratics now depend only on contributing Wick sub-indices given by $\mathcal{W}_{\rho,\eta}^1$ and $\mathcal{W}_{\rho,\eta}^2$.
If $\eta = 0$, then $(\ti,\ts) = (\m{i},\m{s})$. If $\eta \neq 0$, 
We may re-index the sum across $(\m{i},\m{s})$ to the sum across $(\ti,\ts)\in \mathcal{I}_{\rho,\eta}$ and $\mathcal{U}_{\rho,\eta}(\ti,\ts)$. That is,
\begin{align}
    \sum_{(\m{i},\m{s})} \lambda_{\Gamma}(\m{i},\m{s})Q \left( \x \odot \wick{\m{Y}^\Gamma}_{(\ti,\ts)} \right) &\lesssim \sum_{(\ti,\ts)} Q \left( \x \odot \wick{\m{Y}^\Gamma}_{(\ti,\ts)} \right) \sum_{\mathcal{U}_{\rho,\eta}(\ti,\ts)}  \prod_{j=1}^{\ell} \cov_{i_j}^{\eta_j} \prod_{t \not \in \mathcal{W}_{\rho,\eta}^2} \cov_{s_t}^{\eta_{\ell+t}}  \\
    &\lesssim  \sum_{(\ti,\ts)} Q \left( \x \odot \wick{\m{Y}^\Gamma}_{(\ti,\ts)} \right) \sum_{\mathcal{U}_{\rho,\eta}(\ti,\ts)} \prod_{j \not\in \mathcal{W}_{\rho,\eta}^1} \cov^{\eta_j}_{i_j} \prod_{t \not \in \mathcal{W}_{\rho,\eta}^2} \cov_{s_t}^{\eta_{\ell+t}}\\
    &\lesssim  \sum_{(\ti,\ts)} Q \left( \x \odot \wick{\m{Y}^\Gamma}_{(\ti,\ts)} \right), 
\end{align}
where the upper bound follows as a consequence of $\eta_j \neq 0$ for all $j\not\in \mathcal{W}_{\rho,\eta}^1$ and $\eta_{\ell+t} \neq 0$ for all $t\not\in \mathcal{W}_{\rho,\eta}^2$. The sum across $(\ti,\ts)$ of $Q\left(\x \odot \wick{\m{Y}^\Gamma}_{(\ti,\ts)}\right)$ is simply the quadratic 
\begin{equation}
    \x^\top \wick{\m{Y}^\Gamma}^\top \wick{\m{Y}^\Gamma}\x.
\end{equation}
Putting all this together, we see that
\begin{align}\label{eq:PSD_wick_decomp}
    \frac{1}{d}\x^\top \left( \D^{A_\xi} \m{Y}^{\pi \top} \m{Y}^{\pi} \D^{A_\xi} \right)\x   \lesssim   \frac{1}{d}\sum_{\Gamma}  \x^\top \wick{\m{Y}^{\Gamma}}^\top \wick{\m{Y}^{\Gamma}} \x + \x^\top \m{U}_{A_\xi} \x. 
\end{align}
Finally, to conclude our proof we recall that the rows of $\wick{\m{Y}^{\Gamma}}$ are indexed by $\mathcal{I}_{\rho,\eta}$, and since there are only finitely many $\Gamma$, there are only finitely many index sets $\mathcal{I}_{\rho,\eta}$.
}
\end{proof}

\begin{lemma}\label{lemma:wick_eig_uppr_bd}
    Fix $\Gamma = (r,\psi,\rho,\eta)$, let $\wick{\m{Y}} = \wick{ \m{Y}^{\Gamma}} \in \R^{\abs{\mathcal{I}_{\rho,\eta}} \times d} $ be defined as in \eqref{eq:wick_Y_Gamma} and 
\noindent $\p' = \sum_{j\in \mathcal{W}_{\rho,\eta}^1} \rho_j -2\eta_j + \sum_{t\in \mathcal{W}_{\rho,\eta}^2} \rho_{\ell+t}-2\eta_{\ell+t} -\reb{1}$. If $\pwrco>1$, there exist constants $\theta_\star \geq 1$ and $\mu \leq \p-1$ that depend only on $\Gamma$ and $\p$, along with dimension independent constants $c_1,c_2>0$ such that with probability $1-O\left(\log^{-4}(d) \right)$, for all $1\leq j \leq  c_1 d \log^{-(\p'+1)}(d)$,
\begin{equation}\label{eq:subleading_Head_boud}
    \lambda_j\left(\frac{1}{d}\wick{\m{Y}}^\top \wick{\m{Y}} \right) \leq c_2 \left( \frac{\log(j+1)^{\mu-1}}{j}\right)^{\pwrco \theta_\star}.
\end{equation}
Moreover, for all
$j>c_1 d \log^{-(\p'+1)}(d)$,
\begin{equation}\label{eq:subleading_tail_bound}
    \lambda_j\left( \frac{1}{d}\wick{\m{Y}}^\top \wick{\m{Y}} \right) \leq c_2 \left( \frac{ \log^{\mu+\p'}(j+1)}{j} \right)^{\pwrco \theta_\star}.
\end{equation}
\end{lemma}
\begin{remark}
    Given that the rank of $\wick{\m{Y}}^\top \wick{\m{Y}}$ is $\min\{d, \abs{\mathcal{I}_{\rho,\eta}} \}$, if $\wick{\m{Y}}^\top \wick{\m{Y}}$ is rank deficient and $\abs{\mathcal{I}_{\rho,\eta}} \leq c_1 d \log^{-(\p'+1)}(d)$, then \eqref{eq:subleading_Head_boud} holds for all non-zero eigenvalues. 
\end{remark}

\begin{proof}[Proof of \cref{lemma:wick_eig_uppr_bd}]
The proof for a sub-leading Wick matrix is similar to the proof of \cref{prop:principal}, with the main difference being that $\wick{\m{Y}}$ has column vectors of Wick products as opposed to $\m{Y}^{\m{1}^\p}$ with column vectors of products of Gaussians. However, by \cref{lemma:wick_decomp}, the columns of $\wick{\m{Y}}$ can be expressed as a Gaussian polynomial. Moreover, by \cref{lemma:wick_cor}, they also have diagonal covariance. As such, the Wick matrix shares concentration properties similar to those of $\m{Y}^{\m{1}^\p}$. Given the similarities of the proof with the principal case, we will only provide details of the steps leading up to the application of \cref{lemma:minsker_2}.

\reb{Without loss of generality, we assume that $\mathcal{W}_{\rho,\eta}^1 = \left[ \hat{\ell} \right]$ and $\mathcal{W}_{\rho,\eta}^2 = \left[ \hat{r} \right]$ for some $\hat{\ell}\leq \ell$ and $\hat{r}\leq r-n$. It will be convenient to concatenate the exponents into a single vector $\theta = (\theta_1,\cdots,\theta_{\hat{\ell}},\cdots, \theta_{\hat{\ell}+\hat{r} })$   where it is understood that the first $ \hat{\ell}$ entries of $\theta$ corresponds to $\{ \rho_{j}-2\eta_j\}_{j \in \mathcal{W}_{\rho,\eta}^1}$ and the last $\hat{r}$ entries corresponds to $\{ \rho_{\ell+t}-2\eta_{\ell+t}\}_{t\in \mathcal{W}_{\rho,\eta}^2}$.
Define
\[
\theta_\star = \min_{1\leq j \leq \hat{\ell}+\hat{r}} \theta_k \quad \text{and}\quad \mu = \# \{ k:\, \theta_k = \theta_\star \}.
\]
Clearly, $\theta_\star \geq 1$ and $\mu \leq \hat{\ell}+\hat{r} \leq \p -1$. Next, let
}
 $\wick{\m{W}^{\rho,\eta}} \in \R^{\abs{\mathcal{I}_{\rho,\eta}} \times d}$ and $\cov^{\rho,\eta}\in \R^{\abs{\mathcal{I}_{\rho,\eta}} \times \abs{\mathcal{I}_{\rho,\eta}} }$ be defined respectively by
\begin{align}
\wick{\W^{\rho,\eta}}_{(\ti,\ts),a}&= \wick{ \prod_{j=1}^{\reb{ \hat{\ell} }} \W^{\rho_j-2\eta_j}_{i_j,a} \prod_{t=1}^{ \reb{ \hat{r} }} \W^{\rho_{\ell+t} - 2\eta_{\ell+t}}_{s_t,a}} ,\\ \cov^{\rho,\eta} &= \text{Diag}\left( \left\{ \prod_{j=1}^{\reb{ \hat{\ell}} }  \cov_{i_j}^{\rho_j-2\eta_j} \prod_{t=1}^{ \reb{ \hat{r} }} \cov_{s_t}^{\rho_{\ell+t}-2\eta_{\ell+t} \reb{-1} } \right\}_{(\ti,\ts)} \right).
\end{align}
Then, recalling that $\m{y}^{(a)} = \sqrt{\cov}\W_{\col(a)}$, we may write \[\wick{\m{Y}} = \sqrt{\cov^{\rho,\eta}}\wick{\W^{\rho,\eta}}, \quad  \text{and} \quad  \wick{\m{Y}}^\top \wick{\m{Y}} = \wick{\m{W}^{\rho,\eta}}^\top \cov^{\rho,\eta}\wick{\m{W}^{\rho,\eta}}.\]
Let 
\begin{equation}
k_* =   c_1 d \log^{-(\p'+1)}(d),   
\end{equation}
where $c_1>0$ is sufficiently small 
and $\mathcal K_0$ be the collection of all ordered tuple pairs $(\ti,\ts)$ such that $\cov^{\rho,\eta}_{(\ti,\ts)}$ are among the first $k_*$ largest eigenvalues.
Similarly, for $k\geq 1$, let $\mathcal{K}_{k}$ be ordered tuple pairs $(\ti,\ts)$ such that $\cov^{\rho,\eta}_{(\ti,\ts)}$ are among the $(k_*2^{k-1}+1)$ to the $(k_* 2^{k})$ largest eigenvalues. Let  $\wick{\W^{\rho,\eta}_{k}} \in \R^{\abs{\mathcal{K}_k} \times d} $ and $\cov^{\rho,\eta}_{k}  \in \R^{\abs{\mathcal{K}_k} \times \abs{\mathcal{K}_k} }$ be the respective row restriction of $\wick{\W^{\rho,\eta}}$ and $\cov^{\rho,\eta}$ onto $\mathcal{K}_{k}$.

We may decompose $\frac{1}{d}\wick{\m{W}^{\rho,\eta}}^\top \cov^{\rho,\eta}\wick{\m{W}^{\rho,\eta}}$ into 
\begin{equation}\label{eq:W_rho_eta_decomp1}
    \frac{1}{d}\wick{\m{W}^{\rho,\eta}}^\top \cov^{\rho,\eta}\wick{\m{W}^{\rho,\eta}} = \frac{1}{d} \wick{\m{W}^{\rho,\eta}_{0}}^\top \cov^{\rho,\eta}_{0}\wick{\m{W}^{\rho,\eta}_{0}}+ \sum_{k=1}^\infty \frac{1}{d} \wick{\m{W}^{\rho,\eta}_{k}}^\top \cov^{\rho,\eta}_{k}\wick{\m{W}^{\rho,\eta}_{k}}.
\end{equation}
Unlike the proof of \cref{prop:principal}, for each $k\geq 0$, we only need to upper bound the spectrum of
$\frac{1}{d}\sqrt{\cov^{\rho,\eta}_{k}} \wick{\W^{\rho,\eta}_{k}} \wick{\W^{\rho,\eta}_{k}}^\top \sqrt{\cov_{k}^{\rho,\eta}}$. By \cref{lemma:tail_eig_bound},
we know that for all $j\geq 1$,
\begin{equation}\label{eq:rho_eta_eig_uppr_bdd}
    \lambda_j(\cov^{\rho,\eta}) \lesssim \left( \frac{\log^{\mu-1}(j+1)}{j} \right)^{\pwrco \theta_\star}.
\end{equation}
Thus, it suffices to bound the operator norm of $\frac{1}{d}\wick{\W_{k}^{\rho,\eta}} \wick{\W_{k}^{\rho,\eta} }^\top$ using \cref{lemma:minsker_2}.

For a fixed $k \geq 0$, we apply the following rank-$1$ decomposition,
\begin{equation}
    \frac{1}{d}\sqrt{\cov^{\rho,\eta}_{k}} \wick{\W^{\rho,\eta}_{k}} \wick{\W^{\rho,\eta}_{k}}^\top \sqrt{\cov_{k}^{\rho,\eta}} = \sum_{a=1}^d \frac{1}{d} \sqrt{\cov_{k}^{\rho,\eta}} \wc^{(a)}\wc^{(a) \top} \sqrt{\cov_{k}^{\rho,\eta}},
\end{equation}
where $\wc^{(a)}$ are the column vectors of $\wick{\W_{k}^{\rho,\eta}}$. 
 Let $(\rho-2\eta)! = \prod_{j\in \mathcal{W}_{\rho,\eta}^1} (\rho_j-2\eta_j)! \prod_{t\in \mathcal{W}_{\rho,\eta}^2}(\rho_{\ell+t}-2\eta_{\ell+t})!$. By \cref{lemma:wick_cor}, we easily compute the first moment to be
\begin{equation} \label{eq:wc_mean}
    \frac{1}{d}\E \left [ \wc^{(a)} \wc^{(a)\top} \right] = \frac{(\rho-2\eta)!}{d} \Id. 
\end{equation}
For the second moment, we consider its quadratic form. Let $\m{u}\in \R^{\abs{\mathcal{K}_k}}$, then
\begin{align}
   \m{u}^\top \E \left[ \left( \wc^{(a)} \wc^{(a)\top} \right)^2\right] \m{u} &= \E \left[ \wc^{(a)\top}\wc^{(a)} \m{u}^{\top} \wc^{(a)}\wc^{(a)\top} \m{u} \right]\\
   &
   \leq \norm{\wc^{(a)\top}\wc^{(a)}}_{L^2} \norm{ \m{u}^\top \wc^{(a)}\wc^{(a)\top}\m{u}}_{L^2}. \label{eq:second_moment_col}
\end{align}
It is easy to check via the triangle inequality that there exists $C_{\rho,\eta}>0$ such that
\begin{equation}
    \norm{\wc^{(a)\top}\wc^{(a)}}_{L^2} \leq C_{\rho,\eta} \abs{\mathcal{K}_k}.
\end{equation}
For the latter term of \eqref{eq:second_moment_col}, we observe from \cref{lemma:wick_decomp} that \[(\wc^{(a)})_{(\ti,\ts)} = \wick{ \prod_{j\in \mathcal{W}_{\rho,\eta}^1} \W^{\rho_j-2\eta_j}_{i_j,a} \prod_{t\in \mathcal{W}_{\rho,\eta}^2} \W^{\rho_{\ell+t} - 2\eta_{\ell+t}}_{s_t,a}}\] is a degree-$(\p'+1)$ Gaussian polynomial. That is, $\m{u}^\top \wc^{(a)}\wc^{(a)\top}\m{u}$ is a degree-$2(\p'+1)$ Gaussian polynomial. Computing its variance we get
\begin{align}
    \Var\left[  \m{u}^\top \wc^{(a)}\wc^{(a)\top} \m{u} \right] &= \E \left[ \left( \m{u}^\top \wc^{(a)}\wc^{(a)\top} \m{u}\right)^2 \right]-\E \left[ \m{u}^\top \wc^{(a)}\wc^{(a)\top} \m{u} \right]\\
    &\leq \max_{ \mathcal{K}_k^4 } \E \left[ \wc^{(a)}_{(\ti_1,\ts_1) } \wc^{(a)}_{(\ti_2,\ts_2)} \wc^{(a)}_{(\ti_3,\ts_3))} \wc^{(a)}_{(\ti_4,\ts_4))}  \right] \norm{\m{u}^{\otimes 2}}_F^4 - (\rho-2\eta)! 
    \\ &\leq V_{\rho,\eta}.
\end{align}
So by \cref{lem:hypercontractivity}, for all $t>0$,
\begin{equation}
    \prob \left( \abs{ \m{u}^\top \wc^{(a)}\wc^{(a)\top} \m{u}  - (\rho-2\eta)! } >t \right) \lesssim \exp\left( - \left( \frac{t^2}{V_{\rho,\eta}} \right)^{1/2(\p'+1)} \right).
\end{equation}
A straightforward computation then shows
\begin{align}
    \E \left[ \left( \m{u}^\top \wc^{(a)}\wc^{(a)\top} - (\rho-2\eta)! \right)^2 \right] \leq 4 V_{\rho,\eta} (2(\p'))^{\p'}.
\end{align}
Plugging this back into \eqref{eq:second_moment_col}, we see that 
    $\m{u}^\top \E \left[ \left( \wc^{(a)} \wc^{(a)\top} \right)^2\right] \m{u} \leq C_{\rho,\eta} \abs{\mathcal{K}_k}$.
Therefore, if we choose $\m{V}^2 = \frac{C_{\rho,\eta} \abs{\mathcal{K}_k }}{d} \Id$ then $\V^2 \succeq \sum_{a=1}^d \E \left[ \left(\frac{\wc^{(a)}\wc^{(a) \top}}{d} \right)^2 \right] $
and $r(\V^2) = \abs{\mathcal{K}_{k}}$. The rest of the proof then follows by an application of \cref{lemma:minsker_2} as in the proof of \cref{prop:principal}.
\qquad
\end{proof}

\begin{proposition}\label{prop:tail}
    For all $A_\xi \geq 0$ and $\pwrco>1$, there exist \reb{dimension independent} constants $C_1,C_2>0$ such that with probability at least $1-O(\log^{-4}(d))$, for all $1\leq j \leq C_1 d \log^{-(\p+1)}(d)$,
    \begin{equation}\label{eq:wick_prop2.1}
        \lambda_j \left(\frac{1}{d}\D^{A_\xi} \m{Y}^{\pi \top} \m{Y}^{\pi} \D^{A_\xi} \right) \leq C_2 \left( \frac{\log^{\p-1}(j+1)}{j} \right)^\pwrco.
    \end{equation}
    Moreover, if $C_1 d \log^{-(\p+1)}(d)<j \leq d$ then
    \begin{equation}\label{eq:wick_prop2.2}
         \lambda_j\left(\frac{1}{d}\D^{A_\xi} \m{Y}^{\pi \top} \m{Y}^{\pi} \D^{A_\xi} \right) \leq C_2 \left( \frac{\log^{2\p}(j+1)}{j} \right)^\pwrco.
    \end{equation}
\end{proposition}
\begin{proof}
    By \cref{lemma:lower_order_wick_decomp} or more precisely \eqref{eq:PSD_wick_decomp}, there exists a finite collection $\Gamma$ along with a rank-$1$ matrix $\m{U}_{A_\xi}$ such that 
    \begin{equation}\label{eq:prop2_PSD_order}
        \frac{1}{d}\D^{A_\xi} \m{Y}^{\pi \top}\m{Y}^{\pi}\D^{A_\xi} \preceq \sum_{\Gamma} \frac{1}{d}\wick{\m{Y}^{\Gamma}}^\top \wick{\m{Y}^{\Gamma}} +  \m{U}_{A_\xi}.
    \end{equation}
     Using the Courant–Fischer min-max theorem, we only need to bound the eigenvalues of the sum on the right-hand side. By \cref{lemma:wick_eig_uppr_bd}, for each $\Gamma$, there exist constants $c_1, c_2>0$ such that the event
\begin{equation}
    \lambda_j\left( \frac{1}{d}\wick{\m{Y}^{\Gamma}}^\top \wick{\m{Y}^{\Gamma}} \right) \leq  c_2 \left( \frac{\log^{\mu(\Gamma)-1}(j+1)}{j}\right)^{\pwrco \theta_\star(\Gamma)} \quad \forall {1\leq j \leq c_1 d\log^{-(\p'+1)}(d)},
\end{equation}
holds with probability $1-O(\log^{-4}(d))$. Without loss of generality, we may assume that $c_1$ and $c_2$ are identical across all $\Gamma$ by simply taking the minimum and maximum respectively. 

Since $\theta_\star(\Gamma) \geq 1$ and $\mu(\Gamma)\leq \p-1$, it is easy to check that 
\begin{equation}\label{eq:j_limit}
    \left( \frac{ \log^{\mu-1}(j+1) }{j} \right)^{\pwrco \theta_\star} \left( \frac{j}{\log^{\p-1}(j+1)}  \right)^{\pwrco} \xrightarrow{j \to \infty} 0.
\end{equation}
Which is to say that there exists a positive constant $C_2$ such that for all $1\leq j \leq c_1 d\log^{-(\p'+1)}(d)$,
\begin{equation}
     \lambda_j\left( \frac{1}{d}\wick{\m{Y}^{\Gamma}}^\top \wick{\m{Y}^{\Gamma}} \right) \leq C_2 \left( \frac{\log^{\p-1}(j+1)}{j}\right)^{\pwrco}. \label{eq:prop2_pwr_theta_pwr}
\end{equation}
Taking a union bound, we see that \eqref{eq:prop2_pwr_theta_pwr} holds simultaneously for all $\Gamma$ with probability $1-O(\log^{-4}(d))$. To conclude \eqref{eq:wick_prop2.1}, we observe that $\p'<\p$ for all $\Gamma$. Thus, by \cref{lemma:sum_PSD_pwr_law_spec},
\begin{equation}
    \lambda_j\left( \frac{1}{d} \sum_\Gamma  \wick{\m{Y}^{\Gamma}}^\top \wick{\m{Y}^{\Gamma}} +  \m{U}_{A_\xi}  \right) \lesssim  C_2 \left( \frac{\log^{\p-1}(j+1)}{j}\right)^{\pwrco},
\end{equation}
for all $1\leq j \leq c_1 d\log^{-(\p+1)}(d)$.
If $c_1 d \log^{-(\p+1)}(d) <j \leq d$, then we also have for all $\Gamma$,
\begin{equation}
     \lambda_j\left( \frac{1}{d}\wick{\m{Y}^{\Gamma}}^\top \wick{\m{Y}^{\Gamma}} \right)  \leq c_2 \left( \frac{\log^{\mu+\p'}(j+1)}{j} \right)^{\pwrco \theta_\star} \leq C_2 \left( \frac{\log^{2\p}(j+1)}{j} \right)^{\pwrco}.
\end{equation}
Once again by \cref{lemma:sum_PSD_pwr_law_spec}, \eqref{eq:wick_prop2.2} follows.
\qquad
\end{proof}
\ep{
\begin{remark}[Upper bounds require no assumption on $v$]\label{rem:upper_any_v}
    The upper bounds in \cref{prop:principal} and \cref{prop:tail} use no lower bound on $v$ in terms of $d$. Their proofs rely only on upper bounds for the eigenvalue counts of $\cov^\pi$, which follow from the unrestricted lattice counts because $N^v_\downarrow(X)\leq N_\downarrow(X)$ for all $X$; dyadic blocks that contain no index tuples are simply zero. The hypothesis on $v$ in \cref{theorem:MAIN} enters only through the lower bounds, namely through the two-sided count of \cref{lemma:principal_eig_bounds} and through part (b) of Step 4 in the proof of \cref{prop:principal}.

    The same holds for the upper bound in the linear case. If $v\geq d$, this is \cref{thm:linear}. If $v<d$, the nonzero eigenvalues of $\frac{1}{d}\W^\top\cov\W$ coincide with those of $\sqrt{\cov}\,\m{S}\,\sqrt{\cov}$, where $\m{S}=\frac{1}{d}\W\W^\top\in\R^{v\times v}$. By \cite[Corollary 5.35]{vershynin2010introduction}, with probability at least $1-2e^{-d/2}$ we have $\norm{\m{S}}\leq\frac{1}{d}\left(\sqrt{d}+\sqrt{v}+\sqrt{d}\right)^2\leq 9$. Since $\sqrt{\cov}\,\m{S}\,\sqrt{\cov}\preceq\norm{\m{S}}\,\cov$, it follows that $\lambda_j\left(\frac{1}{d}\W^\top\cov\W\right)\leq 9\,\lambda_j(\cov)\lesssim j^{-\pwrco}$ for all $j$.
\end{remark}
}
\begin{remark}
    Recall by \cref{prop:principal}, that the first $O(d\log^{\ep{-(\p+1)}}(d))$ eigenvalues of $\cov^{\m{1}^\p}$ are asymptotically  $\left( j^{-1}\log^{\p-1}(j+1) \right)^{\pwrco}$. Thus, \eqref{eq:j_limit} tells us that as $v$ and $d$ increase, the eigenvalues of the principal term overwhelmingly dominate the eigenvalues of the sub-leading terms. As such, we can say that the spectrum of $\E_\x \left[ \frac{1}{d} f(\W^\top \x)^{\otimes 2} \right]$ is primarily dictated by the principal term $\cov^{\m{1}^\p}$.
\end{remark}
\section{Proof of \cref{theorem:MAIN} }

We first prove Theorem~\ref{theorem:MAIN} for the monomial activation $f(t)=t^{\p}$ by combining the bounds for the principal term and sub-leading terms.

\begin{proof}[Proof of \cref{theorem:MAIN} in the monomial case]
We note that constants that are independent to $v$ and $d$ may change from line to line. Recall, $\mathcal{P}_2(2\p)$ is the set of all pairings in $\{1,\cdots,2\p\}$, $C_\xi$ is the number of pairs $\xi$ between $\{1,\cdots,\p\}$ and $\{\p+1,\cdots,2\p\}$, $A_\xi = \frac{\p-C_\xi}{2}$ and $\pi$ a composition of $C_\xi$. By \eqref{eq:Krep}, we may decompose the covariance kernel into a finite sum of positive semidefinite terms,
\begin{equation}
    \K_{\reb{\W}} \coloneqq \E_{\x} \left[ \frac{1}{d}f(\W^\top \x)^{\otimes 2}\right] = \sum_{\substack{ \xi\in \mathcal{P}_2(2\p)\\ C_\xi \neq 0 }} \sum_{\pi} \reb{C_\pi} \frac{1}{d}\D^{A_\xi} \m{Y}^{\pi \top}\m{Y}^{\pi} \D^{A_\xi} + \reb{\frac{\ep{((\p-1)!!)^2}}{d}\left( \m{u}^{\p/2} \right)^{\otimes 2} ,}
\end{equation}
where $\m{u}^{\p/2}\in \R^d$ defined by $\m{u}_i^{\p/2} = \left( \W^\top_{\col(i)}\cov \W_{\col(i)} \right)^{\p/2}$. 
\reb{By Loewner ordering of positive semidefinite matrices, it is easy to see for all $1\leq j \leq d$,
and any choice of $\xi$ and $\pi$,
}
\begin{equation}\label{eq:K_bd}
    \lambda_j \left(\frac{1}{d}\D^{A_\xi} \m{Y}^{\pi\top}\m{Y}^{\pi} \D^{A_\xi}  \right)\leq \lambda_j(\K_{\reb{\W}}).
\end{equation}
In particular, if $\pi = \m{1}^\p$ then by \cref{prop:principal}, there exist $c_1, c_2>0$ such that with probability $1-O\left(\log^{-4}(d)\right)$, for all $1\leq j \leq c_1 d \log^{-(\p+1)}(d)$,
\begin{equation} \label{eq:main_lower_bd}
    \lambda_j\left(\K_{\reb{\W}} \right) \geq \lambda_j \left(\frac{1}{d}  \m{Y}^{\m{1}^\p \top}\m{Y}^{\m{1}^\p}   \right) \geq c_2 \left( \frac{\log^{\p-1}(j+1)}{j} \right)^\pwrco.
\end{equation}
\reb{Moreover, if $c_1 d \log^{\ep{-(\p+1)}}(d) \leq j \leq d $ and $\p \geq 2$, 
\begin{equation}
    \lambda_j(\K_\W) \geq \lambda_j \left(\frac{1}{d}  \m{Y}^{\m{1}^\p \top}\m{Y}^{\m{1}^\p}   \right) \geq c_2 \left( j^{-1}\log^{-2}(j+1) \right)^\pwrco.
\end{equation}}
As for the upper bound, by \cref{prop:principal} and \cref{prop:tail}, there exists $c_3>0$ such that with probability $1-O\left( \log^{-4}(d)\right)$, for any choice of pairing $\xi$ satisfying $C_\xi\neq 0$, composition $\pi$ of $C_\xi$ and $1\leq j \leq c_1 d\log^{-(\p+1)}(d)$,
\begin{equation} \label{eq:main_proof_uppr_bd}
    \lambda_j \left( \frac{1}{d}\D^{A_\xi} \m{Y}^{\pi \top}\m{Y}^{\pi} \D^{A_\xi} \right) \leq  c_3 \left( \frac{\log^{\p-1}(j+1)}{j} \right)^\pwrco.
\end{equation}
\reb{The only term that now remains is the rank-$1$ perturbation $\frac{1}{d}\left( \m{u}^{\p/2} \right)^{\otimes 2}$. It is easy to check by Chebyshev's inequality that  $\lambda_j\left(\frac{1}{d}\left( \m{u}^{\p/2} \right)^{\otimes 2}\right) = \delta_{1j}O(1)$ with probability $1-O(d^{-1})$.} \reb{Thus, by Lemma \ref{lemma:sum_PSD_pwr_law_spec},}
\begin{equation}
    \lambda_j(\K_{\reb{\W}}) \leq  c_3 \left( \frac{\log^{\p-1}(j+1)}{j} \right)^\pwrco. 
\end{equation}
Finally, to conclude the proof, if $j> c_1 d\log^{-(\p+1)}(d)$ then combining \eqref{eq:K_bd} with \cref{prop:principal} and \cref{prop:tail}, we see that on the event of \eqref{eq:main_proof_uppr_bd},
\begin{equation} 
     \lambda_j\left(\K_\W  \right) \leq c_3 \left(  \frac{ \log^{2\p}(j+1)}{j}\right)^{\pwrco},
\end{equation}
as desired. 
\end{proof}

\reb{
\section{Higher order case: polynomial activations}
We now extend the monomial case to polynomial activations.}

\begin{lemma}[Extension to polynomial activations]\label{lemma:polynomial_reduction}
Let $f(t)=\sum_{r=0}^{\p}a_rt^r$, where $\p\geq1$ and $a_{\p}\neq0$. \ep{Assume that $v$ satisfies the hypothesis of \cref{theorem:MAIN} for this degree $\p$. If the upper bounds of \cref{theorem:MAIN} hold for each monomial $f_r(t)=t^r$, $1\leq r\leq \p$, and the lower bounds of \cref{prop:principal} hold for the degree-$\p$ principal term, then \cref{theorem:MAIN} holds for $f$.}
\end{lemma}
\begin{proof}
 For the lower bound, fix distinct coordinates $i_1<\cdots<i_{\p}$. Repeated Gaussian integration by parts gives
\begin{align*}
    \E_{\z}\left[f(\inner{\y^{(i)}}{\z})\prod_{t=1}^{\p}z_{i_t}\right]
    &\quad=\E_{\z}\left[
\partial_{z_{i_1}}\cdots\partial_{z_{i_{\p}}}
    f(\inner{\y^{(i)}}{\z})\right]\\
&\quad=\left(\prod_{t=1}^{\p}\y^{(i)}_{i_t}\right)
    \E_{\z}\left[f^{(\p)}(\inner{\y^{(i)}}{\z})\right]
    =a_{\p}\p!\prod_{t=1}^{\p}\y^{(i)}_{i_t}
\end{align*}
The last equality uses $f^{(\p)}\equiv a_{\p}\p!$.

The products $\prod_{t=1}^{\p}z_{i_t}$, indexed by increasing coordinate tuples, form an orthonormal family in Gaussian $L^2$. Indeed, each has squared norm $\prod_{t=1}^{\p}\E[z_{i_t}^2]=1$. For two different tuples, some coordinate occurs exactly once in their product, so independence and centering make its expectation zero. 

Thus, for every $\m{u}\in\R^d$, Bessel's inequality gives
\begin{align*}
    d\,\m{u}^{\top}\K_{\W}\m{u}
    &=\E_{\z}\left[\left(\sum_{i=1}^d
    u_i f(\inner{\y^{(i)}}{\z})\right)^2\right]\\
    &\geq\sum_{i_1<\cdots<i_{\p}}
    \left|\E_{\z}\left[
    \left(\sum_{i=1}^d u_i f(\inner{\y^{(i)}}{\z})\right)
    \prod_{t=1}^{\p}z_{i_t}\right]\right|^2\\
    &=a_{\p}^2(\p!)^2\sum_{i_1<\cdots<i_{\p}}
    \left(\sum_{i=1}^d u_i\prod_{t=1}^{\p}\y^{(i)}_{i_t}\right)^2=a_{\p}^2(\p!)^2\m{u}^{\top}
    \bigl(\m{Y}^{\m{1}^{\p}}\bigr)^{\top}
    \m{Y}^{\m{1}^{\p}}\m{u}.
\end{align*}

Applying the min--max principle yields, for every $1\leq j\leq d$,
\[
    \K_{\W}\succeq\frac{a_{\p}^2(\p!)^2}{d}
    \bigl(\m{Y}^{\m{1}^{\p}}\bigr)^{\top}\m{Y}^{\m{1}^{\p}},
    \qquad
    \lambda_j(\K_{\W})\geq a_{\p}^2(\p!)^2
    \lambda_j\left(\frac{1}{d}
    \bigl(\m{Y}^{\m{1}^{\p}}\bigr)^{\top}\m{Y}^{\m{1}^{\p}}\right).
\]
Consequently, the principal-term estimates in Proposition~\ref{prop:principal} give the lower bound in \eqref{eq:main_1}, and, when $\p\geq2$, the lower bound in \eqref{eq:main_min_tail}.

For the upper bound, apply Cauchy--Schwarz to the $\p+1$ terms of the polynomial. For every $\m{u}\in\R^d$,
\begin{align*}
    d\,\m{u}^{\top}\K_{\W}\m{u}
    &=\E_{\x}\left[\left(\sum_{r=0}^{\p}a_r
    \sum_{i=1}^d u_i\inner{\W_{\col(i)}}{\x}^{r}\right)^2\right]\\
    &\leq(\p+1)\sum_{r=0}^{\p}a_r^2
    \E_{\x}\left[\left(\sum_{i=1}^d
    u_i\inner{\W_{\col(i)}}{\x}^{r}\right)^2\right].
\end{align*}
This  gives the matrix comparison
\[
    \K_{\W}\preceq(\p+1)\left[
    \frac{a_0^2}{d}\m{1}\m{1}^{\top}
    +\sum_{r=1}^{\p}a_r^2\E_{\x}\left[
    \frac{1}{d}f_r(\W^{\top}\x)^{\otimes2}\right]\right].
\]
The min--max principle  implies
\begin{align*}
    \lambda_j(\K_{\W})
    &\leq(\p+1)\lambda_{\lceil j/(\p+1)\rceil}
    \left(\frac{a_0^2}{d}\m{1}\m{1}^{\top}\right)\\
    &\quad+(\p+1)\sum_{r=1}^{\p}a_r^2
    \lambda_{\lceil j/(\p+1)\rceil}\left(
    \E_{\x}\left[\frac{1}{d}f_r(\W^{\top}\x)^{\otimes2}\right]\right).
\end{align*}

Combining the preceding estimates for the monomial case proves
\[
    \lambda_j(\K_{\W})
    \lesssim\left(\sum_{r=0}^{\p}a_r^2\right)
    \left(\frac{\log^{\p-1}(j+1)}{j}\right)^{\pwrco},
    \qquad 1\leq j\leq c_1d\log^{-(\p+1)}d.
\]

For the tail upper bound, take $j>c_1d\log^{-(\p+1)}d$. For sufficiently large $d$,
\[
    \left\lfloor c_1d\log^{-(\p+1)}d\right\rfloor
    \geq\frac{c_1}{2}d\log^{-(\p+1)}d\geq1.
\]
Monotonicity and the sharp upper estimate at this integer yield
\begin{align*}
    \lambda_j(\K_{\W})
    &\leq\lambda_{\lfloor c_1d\log^{-(\p+1)}d\rfloor}(\K_{\W})\\
    &\lesssim\left(\sum_{r=0}^{\p}a_r^2\right)
    \left(\frac{\log^{\p-1}(\lfloor c_1d\log^{-(\p+1)}d\rfloor+1)}
    {\lfloor c_1d\log^{-(\p+1)}d\rfloor}\right)^{\pwrco}\\
    &\lesssim\left(\sum_{r=0}^{\p}a_r^2\right)
    d^{-\pwrco}\log^{2\p\pwrco}d\lesssim  
    \left(\frac{\log^{2\p}(j+1)}{j}\right)^{\pwrco},
\end{align*}
which is the required tail upper bound.
\end{proof}

\reb{
\begin{proof}[Completion of the proof of \cref{theorem:MAIN}]
Let $f(t)=\sum_{r=0}^{\p}a_rt^r$ with $a_{\p}\neq0$. Apply Lemma~\ref{lemma:polynomial_reduction} to the degree-$r$ monomial upper estimates for $1\leq r\leq\p$, and the degree-$\p$ principal-term lower estimates of Proposition~\ref{prop:principal}. \ep{By Remark~\ref{rem:upper_any_v}, the monomial upper estimates hold for every degree $r\leq\p$ without any further assumption on $v$, in particular when $v<d^{\frac{1}{r}+\epsilon}$; the principal-term lower estimates hold under the hypothesis of \cref{theorem:MAIN}.} This completes the proof of Theorem~\ref{theorem:MAIN}.
\end{proof}}

\section*{Acknowledgments}
E. Paquette was supported
by an NSERC Discovery Grant RGPIN-2025-04643, an FRQNT-NSERC NOVA Grant, a CIFAR
Catalyst Grant, and a gift from Google Canada. Y. Zhu was partially supported by the Simons Grant MPS-TSM-00013944 and NSF DMS-2606337. This material is based upon work supported by the Swedish
Research Council under grant no. 202106594 while E. Paquette and Y. Zhu were in residence at Institut Mittag-Leffler in Djursholm, Sweden during the Fall of 2024. Part of the work was done while Y.Zhu was visiting the Simons Institute for the Theory of Computing during the Spectral Theory Beyond Graphs program in Fall 2026.

\bibliographystyle{siamplain}
\bibliography{references}

@article{delange1954ikehara,
  title = {Généralisations du théorème de Ikehara},
  author = {Delange, Jean-Pierre},
  journal = {Annales Scientifiques de l'École Normale Supérieure},
  volume = {71},
  year = {1954},
  pages = {213--242}
}

@book{tenenbaum2015introduction,
  title = {Introduction to Analytic and Probabilistic Number Theory},
  author = {Tenenbaum, Gérald},
  series = {Graduate Studies in Mathematics},
  volume = {163},
  publisher = {American Mathematical Society},
  edition = {3rd},
  year = {2015}
}

@book{montgomery2006multiplicative,
  title = {Multiplicative Number Theory I: Classical Theory},
  author = {Montgomery, Hugh L. and Vaughan, Robert C.},
  series = {Cambridge Studies in Advanced Mathematics},
  volume = {97},
  publisher = {Cambridge University Press},
  year = {2006}
}

@book{apostol1976analytic,
  title = {Introduction to Analytic Number Theory},
  author = {Apostol, Tom M.},
  series = {Graduate Texts in Mathematics},
  publisher = {Springer},
  year = {1976}
}

@book{narkiewicz1984number,
  title = {Number Theory},
  author = {Narkiewicz, W{\l}adys{\l}aw},
  publisher = {World Scientific Publishing Company},
  year = {1984}
}

@article{lin2024scaling,
  title={Scaling laws in linear regression: Compute, parameters, and data},
  author={Lin, Licong and Wu, Jingfeng and Kakade, Sham M and Bartlett, Peter L and Lee, Jason D},
  journal={Advances in Neural Information Processing Systems},
  volume={37},
  pages={60556--60606},
  year={2024}
}

@article{vershynin2010introduction,
  title={Introduction to the non-asymptotic analysis of random matrices},
  author={Vershynin, Roman},
  journal={arXiv preprint arXiv:1011.3027},
  year={2010}
}

@article{jirak2025concentration,
  title={Concentration and moment inequalities for sums of independent heavy-tailed random matrices},
  author={Jirak, Moritz and Minsker, Stanislav and Shen, Yiqiu and Wahl, Martin},
  journal={Probability Theory and Related Fields},
  pages={1--28},
  year={2025},
  publisher={Springer}
}

@article{wortsman2025kernel,
  title={Kernel ridge regression under power-law data: spectrum and generalization},
  author={Wortsman, Arie and Loureiro, Bruno},
  journal={arXiv preprint arXiv:2510.04780},
  year={2025}
}

@book{janson1997gaussian,
  title={Gaussian hilbert spaces},
  author={Janson, Svante},
  number={129},
  year={1997},
  publisher={Cambridge university press}
}

@article{koltchinskii2015bounding,
  title={Bounding the smallest singular value of a random matrix without concentration},
  author={Koltchinskii, Vladimir and Mendelson, Shahar},
  journal={International Mathematics Research Notices},
  volume={2015},
  number={23},
  pages={12991--13008},
  year={2015},
  publisher={Oxford University Press}
}

@article{hinrichs2021random,
  title={Random sections of ellipsoids and the power of random information},
  author={Hinrichs, Aicke and Krieg, David and Novak, Erich and Prochno, Joscha and Ullrich, Mario},
  journal={Transactions of the American Mathematical Society},
  volume={374},
  number={12},
  pages={8691--8713},
  year={2021}
}

@article{pandit2025universality,
  title={Universality of kernel random matrices and kernel regression in the quadratic regime},
  author={Pandit, Parthe and Wang, Zhichao and Zhu, Yizhe},
  journal={Journal of Machine Learning Research},
  volume={26},
  number={224},
  pages={1--73},
  year={2025}
}

@article{kaushik2025general,
  title={A general technique for approximating high-dimensional empirical kernel matrices},
  author={Kaushik, Chiraag and Romberg, Justin and Muthukumar, Vidya},
  journal={arXiv preprint arXiv:2511.03892},
  year={2025}
}

@article{bartlett2021deep,
  title={Deep learning: a statistical viewpoint},
  author={Bartlett, Peter L and Montanari, Andrea and Rakhlin, Alexander},
  journal={Acta numerica},
  volume={30},
  pages={87--201},
  year={2021},
  publisher={Cambridge University Press}
}

@article{el2010spectrum,
  title={The spectrum of kernel random matrices},
  author={El Karoui, Noureddine},
  journal={Annals of statistics},
  volume={38},
  number={1},
  pages={1--50},
  year={2010},
  publisher={Institute of Mathematical Statistics}
}

@inproceedings{wang2023overparameterized,
  title={Overparameterized random feature regression with nearly orthogonal data},
  author={Wang, Zhichao and Zhu, Yizhe},
  booktitle={International Conference on Artificial Intelligence and Statistics},
  pages={8463--8493},
  year={2023},
  organization={PMLR}
}

@article{tikhomirov2018sample,
  title={Sample covariance matrices of heavy-tailed distributions},
  author={Tikhomirov, Konstantin},
  journal={International Mathematics Research Notices},
  volume={2018},
  number={20},
  pages={6254--6289},
  year={2018},
  publisher={Oxford University Press}
}

@inproceedings{liang2020multiple,
  title={On the multiple descent of minimum-norm interpolants and restricted lower isometry of kernels},
  author={Liang, Tengyuan and Rakhlin, Alexander and Zhai, Xiyu},
  booktitle={Conference on Learning Theory},
  pages={2683--2711},
  year={2020},
  organization={PMLR}
}

@article{mei2021generalization,
  title={Generalization error of random feature and kernel methods: Hypercontractivity and kernel matrix concentration},
  author={Mei, Song and Misiakiewicz, Theodor and Montanari, Andrea},
  journal={Applied and Computational Harmonic Analysis},
  volume={59},
  pages={3--84},
  year={2022},
  publisher={Elsevier}
}

@article{mei2019generalization,
  title={The Generalization Error of Random Features Regression: Precise Asymptotics and the Double Descent Curve},
  author={Mei, Song and Montanari, Andrea},
  journal={Communications on Pure and Applied Mathematics},
  year={2019},
  publisher={Wiley Online Library}
}

@article{tropp2012user,
  title={User-friendly tail bounds for sums of random matrices},
  author={Tropp, Joel A},
  journal={Foundations of computational mathematics},
  volume={12},
  number={4},
  pages={389--434},
  year={2012},
  publisher={Springer}
}

@article{defilippis2024dimension,
  title={Dimension-free deterministic equivalents and scaling laws for random feature regression},
  author={Defilippis, Leonardo and Loureiro, Bruno and Misiakiewicz, Theodor},
  journal={Advances in Neural Information Processing Systems},
  volume={37},
  pages={104630--104693},
  year={2024}
}

@article{wang2021deformed,
  title={Deformed semicircle law and concentration of nonlinear random matrices for ultra-wide neural networks},
  author={Wang, Zhichao and Zhu, Yizhe},
  journal={The Annals of Applied Probability},
  volume={34},
  number={2},
  pages={1896--1947},
  year={2024},
  publisher={Institute of Mathematical Statistics}
}

@inproceedings{donhauser2021rotational,
  title={How rotational invariance of common kernels prevents generalization in high dimensions},
  author={Donhauser, Konstantin and Wu, Mingqi and Yang, Fanny},
  booktitle={International Conference on Machine Learning},
  pages={2804--2814},
  year={2021},
  organization={PMLR}
}

@article{cheng2024dimension,
  title={Dimension free ridge regression},
  author={Cheng, Chen and Montanari, Andrea},
  journal={The Annals of Statistics},
  volume={52},
  number={6},
  pages={2879--2912},
  year={2024},
  publisher={Institute of Mathematical Statistics}
}

@article{mendelson2012generic,
  title={On generic chaining and the smallest singular value of random matrices with heavy tails},
  author={Mendelson, Shahar and Paouris, Grigoris},
  journal={Journal of Functional Analysis},
  volume={262},
  number={9},
  pages={3775--3811},
  year={2012},
  publisher={Elsevier}
}

@article{abdalla2024covariance,
  title={Covariance estimation: Optimal dimension-free guarantees for adversarial corruption and heavy tails},
  author={Abdalla, Pedro and Zhivotovskiy, Nikita},
  journal={Journal of the European Mathematical Society},
  volume={2},
  number={5},
  year={2024}
}

@article{hu2024asymptotics,
  title={Asymptotics of Random Feature Regression Beyond the Linear Scaling Regime},
  author={Hu, Hong and Lu, Yue M and Misiakiewicz, Theodor},
  journal={arXiv preprint arXiv:2403.08160},
  year={2024}
}

@article{bartlett2020benign,
  title={Benign overfitting in linear regression},
  author={Bartlett, Peter L and Long, Philip M and Lugosi, G{\'a}bor and Tsigler, Alexander},
  journal={Proceedings of the National Academy of Sciences},
  volume={117},
  number={48},
  pages={30063--30070},
  year={2020},
  publisher={National Academy of Sciences}
}

@article{tsigler2023benign,
  title={Benign overfitting in ridge regression},
  author={Tsigler, Alexander and Bartlett, Peter L},
  journal={Journal of Machine Learning Research},
  volume={24},
  number={123},
  pages={1--76},
  year={2023}
}

@inproceedings{schroder2024asymptotics,
  title={Asymptotics of learning with deep structured (random) features},
  author={Schr{\"o}der, Dominik and Dmitriev, Daniil and Cui, Hugo and Loureiro, Bruno},
  booktitle={Proceedings of the 41st International Conference on Machine Learning},
  pages={43862--43894},
  year={2024}
}

@article{schroder2024deterministic,
  title={Deterministic equivalent and error universality of deep random features learning},
  author={Schr{\"o}der, Dominik and Cui, Hugo and Dmitriev, Daniil and Loureiro, Bruno},
  journal={Journal of Statistical Mechanics: Theory and Experiment},
  volume={2024},
  number={10},
  pages={104017},
  year={2024},
  publisher={IOP Publishing}
}

@article{kaplan2020scaling,
  title={Scaling laws for neural language models},
  author={Kaplan, Jared and McCandlish, Sam and Henighan, Tom and Brown, Tom B. and Chess, Benjamin and Child, Rewon and Gray, Scott and Radford, Alex and Wu, Jeffrey and Amodei, Dario},
  journal={arXiv preprint arXiv:2001.08361},
  year={2020}
}

@inproceedings{hoffmann2022chinchilla,
 author = {Hoffmann, Jordan and Borgeaud, Sebastian and Mensch, Arthur and Buchatskaya, Elena and Cai, Trevor and Rutherford, Eliza and de Las Casas, Diego and Hendricks, Lisa Anne and Welbl, Johannes and Clark, Aidan and Hennigan, Tom and Noland, Eric and Millican, Katie and van den Driessche, George and Damoc, Bogdan and Guy, Aurelia and Osindero, Simon and Simonyan, Karen and Elsen, Erich and Vinyals, Oriol and Rae, Jack and Sifre, Laurent},
 booktitle = {Advances in Neural Information Processing Systems (NeurIPS)},
 title = {An empirical analysis of compute-optimal large language model training},
 volume = {35},
 year = {2022}
}

@article{maloney2022solvable,
  title={A Solvable Model of Neural Scaling Laws},
  author={Maloney, Alexander and Roberts, Daniel A. and Sully, James},
  journal={arXiv preprint arXiv:2210.16859},
  year={2024}
}

@article{bahri2021explaining,
    author = {Bahri, Yasaman and Dyer, Ethan and Kaplan, Jared and Lee, Jaehoon and Sharma, Utkarsh},
     title = {Explaining neural scaling laws},
   journal = {Proc. Natl. Acad. Sci. USA},
    volume = {121},
      year = {2024},
    number = {27},
     pages = {Paper No. e2311878121, 8}
}

@inproceedings{paquette20244+,
 author = {Paquette, Elliot and Paquette, Courtney and Xiao, Lechao and Pennington, Jeffrey},
 booktitle = {Advances in Neural Information Processing Systems (NeurIPS)},
 title = {4+3 Phases of Compute-Optimal Neural Scaling Laws},
 volume = {37},
 year = {2024}
}

@inproceedings{rahimi2008random,
  title = {Random features for large-scale kernel machines},
  author = {Rahimi, A. and Recht, B.},
  booktitle = {Advances in Neural Information Processing Systems (NeurIPS)},
  pages = {1177--1184},
  year = {2008}
}

@article{ferbach2025dana,
  title={Dimension-adapted Momentum Outscales {SGD}},
  author={Ferbach, Damien and Everett, Katie and Gidel, Gauthier and Paquette, Elliot and Paquette, Courtney},
  journal={arXiv preprint},
  year={2025}
}

@article{reventos2025understanding,
  title={Understanding the origin of neural scaling laws from first principles},
  author={Revent{\'o}s, Kenyon and Ganguli, Surya and Wyart, Matthieu},
  journal={arXiv preprint arXiv:2602.07488},
  year={2025}
}

@article{ren2025emergence,
  title={Emergence and scaling laws in {SGD} learning of shallow neural networks},
  author={Ren, Yunwei and Nichani, Eshaan and Wu, Denny and Lee, Jason D.},
  journal={arXiv preprint arXiv:2504.19983},
  year={2025}
}

@article{benarous2025learning,
  title={Learning quadratic neural networks in high dimensions: {SGD} dynamics and scaling laws},
  author={Ben Arous, G{\'e}rard and Erdogdu, Murat A. and Vural, Nuri Mert and Wu, Denny},
  journal={arXiv preprint arXiv:2508.03688},
  year={2025}
}

@article{liu2026universal,
  title={Universal one-third time scaling in learning peaked distributions},
  author={Liu, Yizhou and Liu, Ziming and Pehlevan, Cengiz and Gore, Jeff},
  journal={arXiv preprint arXiv:2602.03685},
  year={2026}
}

@article{ferbach2026logarithmic,
  title={Logarithmic-time schedules for scaling language models with momentum},
  author={Ferbach, Damien and Paquette, Courtney and Gidel, Gauthier and Everett, Katie and Paquette, Elliot},
  journal={arXiv preprint arXiv:2602.05298},
  year={2026}
}

@article{bach2025ztransform,
  title={On the effectiveness of the z-transform method in quadratic optimization},
  author={Bach, Francis},
  journal={arXiv preprint arXiv:2507.03404},
  year={2025}
}

@article{bordelon2026optimal,
  title={Theory of optimal learning rate schedules and scaling laws for a random feature model},
  author={Bordelon, Blake and Mori, Francesco},
  journal={arXiv preprint arXiv:2602.04774},
  year={2026}
}

@InProceedings{varre2022accelerated,
  title={Accelerated {SGD} for non-strongly-convex least squares},
  author={Varre, Aditya and Flammarion, Nicolas},
  booktitle={Proceedings of the 35th Conference on Learning Theory (COLT)},
  series={Proceedings of Machine Learning Research},
  volume={178},
  year={2022},
  publisher={PMLR}
}

@article{liu2018accelerating,
  title={Accelerating {SGD} with momentum for over-parameterized learning},
  author={Liu, Chaoyue and Belkin, Mikhail},
  journal={arXiv preprint arXiv:1810.13395},
  year={2018}
}

@article{mendelson2018robust,
  title={Robust covariance estimation under {$L_4$--$L_2$} norm equivalence},
  author={Mendelson, Shahar and Zhivotovskiy, Nikita},
  journal={The Annals of Statistics},
  volume={48},
  number={3},
  pages={1648--1664},
  year={2020}
}

@techreport{krizhevsky2009learning,
  title={Learning multiple layers of features from tiny images},
  author={Krizhevsky, Alex},
  institution={University of Toronto},
  year={2009}
}

@article{simonyan2014very,
  title={Very deep convolutional networks for large-scale image recognition},
  author={Simonyan, Karen and Zisserman, Andrew},
  journal={arXiv preprint arXiv:1409.1556},
  year={2014}
}

@inproceedings{he2016deep,
  title={Deep residual learning for image recognition},
  author={He, Kaiming and Zhang, Xiangyu and Ren, Shaoqing and Sun, Jian},
  booktitle={Proceedings of the IEEE Conference on Computer Vision and Pattern Recognition (CVPR)},
  pages={770--778},
  year={2016}
}

@article{paquette2026power,
  title={Power-Law Spectrum of the Random Feature Model},
  author={Paquette, Elliot and Xiao, Ke Liang and Zhu, Yizhe},
  journal={arXiv preprint arXiv:2603.14578},
  year={2026}
}

@inproceedings{braun2026fast,
  title={Fast Escape, Slow Convergence: Learning Dynamics of Phase Retrieval under Power-Law Data},
  author={Braun, G. and Loureiro, B. and Quang, M. H. and Imaizumi, M.},
  booktitle={International Conference on Learning Representations},
  year={2026}
}

@inproceedings{defilippis2026scaling,
  title={Scaling Laws and Spectra of Shallow Neural Networks in the Feature Learning Regime},
  author={Defilippis, L. and Xu, Y. and Girardin, J. and Erba, V. and Troiani, E. and Zdeborov{\'a}, L. and Loureiro, B. and Krzakala, F.},
  booktitle={International Conference on Learning Representations},
  year={2026}
}

@book{titchmarsh1986theory,
  title = {The Theory of the Riemann Zeta-Function},
  author = {Titchmarsh, E. C.},
  edition = {Second},
  note = {Revised by D. R. Heath-Brown},
  publisher = {Oxford University Press},
  year = {1986}
}

@misc{paquette2026code,
  title = {Power-law spectra of random features: code and documentation for the numerical experiments},
  author = {Paquette, Elliot and Xiao, Ke Liang and Zhu, Yizhe},
  year = {2026},
  howpublished = {\url{https://github.com/mlexpos/power-law-activations-are-preserved}},
  note = {Documentation: \url{https://mlexpos.github.io/power-law-activations-are-preserved/}}
}

\appendix

\section{Concentration inequalities}
\begin{lemma}\label{lemma:Max_mmt}
Let $\{\wc^{(i)}\}_{i=1}^d$ be independent random vectors in $\R^{s}$ such that for every $i\in [d]$ and every $p\geq 1$, there exist a $C_p>0$ with the property that for all $\m{H}\in \R^{s\times s}$,
\begin{equation}\label{eq:moment_bdd}
    \norm{\inner{\wc_i^{\otimes 2}}{\m{H}} - \tr(\cov)}_{L^p} \leq C_p \norm{\m{H}}_F.
\end{equation}
If $\cov \in \R^{s \times s}$ is PSD and $M = \max_{1\leq i\leq d} \norm{ \frac{\sqrt{\cov}\left(\wc^{(i)} \right)^{\otimes 2}\sqrt{\cov} - \cov}{d} }$ then for all $q \geq 1$, there exists $C_q>0$ such that for all $t>0$,
\begin{equation} \label{eq:max_tail_bd}
    \prob\left( M \geq \frac{2 \tr(\cov)}{d} + t \right) \leq \frac{C_q \norm{\cov}_F^q}{d^{q-1} t^q}.
\end{equation}
Moreover, for all $q>p$,
\begin{equation}\label{eq:E_max_p}
    \E[M^p] \leq \left( \frac{2\tr{\cov}}{d}+ \frac{2pC_q \norm{\cov}_F}{d^{1-1/q}}\right)^p.
\end{equation}
\end{lemma}

\begin{proof}
\eqref{eq:max_tail_bd} follows from the observation that $\sqrt{\m{H}}\wc_i^{\otimes 2}\sqrt{\m{H}}$ is rank-$1$ so its operator norm is the quadratic $\inner{\wc_i^{\otimes 2}}{\m{H}}$. It follows that
\begin{equation}
    \norm{\sqrt{\m{H}} \wc_i^{\otimes 2} \sqrt{\m{H}} -\m{H}} \leq \abs{ \inner{\wc^{\otimes 2}}{\m{H}} }+ \tr(\m{H}). 
\end{equation}
Plugging this into the tail bound gives
\begin{align}
    \prob \left( M \geq \frac{2 \tr{(\cov)}}{d} + t \right) &\leq d \, \prob \left( \norm{\frac{\sqrt{\cov}\wc_1^{\otimes 2}\sqrt{\cov} - \cov}{d}}  \geq \frac{2 \tr{(\cov)}}{d} + t \right)\\
    &\leq d \, \prob \left( \left|\inner{\wc_1^{\otimes 2}}{\cov} - \tr{(\cov)}  \right| >t d \right).
\end{align}
Applying \eqref{eq:moment_bdd} with Markov's inequality then gives the desired tail bound. For \eqref{eq:E_max_p}, using \eqref{eq:max_tail_bd} we see that if $q>p$ and $\epsilon>0$, then
\begin{align}
    \E [M^p] &= \int_0^\infty \prob \left( M \geq t^{1/p} \right)\,\dif t\\
    &= \left( \frac{2\tr(\cov)}{d}+ \epsilon\right)^p + \int_{\left( \frac{2\tr(\cov)}{d}+ \epsilon\right)^p}^\infty \prob \left( M \geq t^{1/p} \right)\,\dif t \\
    &\leq  \left( \frac{2\tr(\cov)}{d}+ \epsilon\right)^p + \frac{p C_q \norm{\cov}_F^q}{d^{q-1}} \epsilon^{1-q} \left( \frac{2\tr(\cov)}{d}+\epsilon \right)^{p-1}.
\end{align}
Choosing $\epsilon = \frac{pC_q \norm{\cov}_F}{d^{1-1/q}}$ yields the result. 
\end{proof}

\begin{lemma}[Theorem 2 in \cite{jirak2025concentration}]\label{lemma:minsker_2}
Let $\W_1,\cdots, \W_n \in \R^{d \times d}$ be centered, independent, self-adjoint random matrices such that $\E[M^q]<\infty$ for some $q\geq 2$. If there exists $\V$ that satisfies $\V^2 \succeq \sum_{k} \E[ \W_k^2]$ and $\sigma^2 = \norm{\V^2}$, then for all $t\geq \max(4\sigma, 32 \E[M]) $
\begin{multline}
    \prob\left( \norm{\sum_{k}\W_k} \geq 12t \right) \leq 128 r(\V^2) \exp \left( \frac{-t^2}{8(\sigma^2+4t \E[M])}\right)\\
    + 4\prob(M>t) + \left( \left( \frac{4q}{\log(eq)} \right)^q \frac{\E[M^q]}{t^q}\right)^2,
\end{multline}
where $K$ is an absolute constant.
\end{lemma}

\begin{lemma}[Gaussian hypercontractivity, Theorem 6.7. in \cite{janson1997gaussian}]\label{lem:hypercontractivity}
    Consider a degree-$q$ polynomial $
f(Y)=f(Y_1,\ldots,Y_n)$
of independent centered Gaussian  random variables 
$Y_1,\ldots,Y_n$. Then for every $\lambda>0$,
\[
\Pr\!\left( |f(Y)-\mathbb{E}f(Y)| \ge \lambda \right)
\;\le\;
2
\exp\!\left(
 - R\left( 
 \frac{\lambda^{2}}{\Var[f(Y)]}
 \right)^{1/q}
\right),
\]
where  $R>0$ is an absolute constant. 
\end{lemma}
\section{Auxiliary lemmas}

\subsection{Wick Products}\label{app:wick}
We present some of the fundamental results of the

\noindent Wiener chaos decomposition. For a complete treatment on Gaussian Hilbert space, we refer the reader to  \cite{janson1997gaussian}.
\begin{definition}
    Let $\m{G}$ be a Gaussian Hilbert space. Define for all $n \geq 0$, 
    \begin{equation}
        \m{P}_n = \left\{ p(g_1,\cdots,g_m) \,; p \text{ is a polynomial of degree}\leq n, \text{ and } \, g_1,\dots, g_m \in \m{G} \right\}.
    \end{equation}
    We will denote $\overline{\m{P}}_n$ to be the $L^2$-closure of $\m{P}_n$.
\end{definition}
\begin{definition}
    A rank-$r$ Feynman diagram of a collection of random variables $g_1,\dots,g_n$, is a graph $\gamma$ consisting of $n$ vertices labeled by $g_1,\dots, g_n$ with $r$ edges that are vertex disjoint. Let $\mathcal{P}(\gamma)$ be the set of all paired vertices and $\mathcal{U}(\gamma)$ the set of all unpaired vertices, we define the value of $\gamma$ to be
    \begin{equation}
        v(\gamma) = \prod_{(r,s) \in \mathcal{P}(\gamma)}\E \left[ g_r g_s\right] \prod_{i\in \mathcal{U}(\gamma)} g_i.
    \end{equation}
    We say that a Feynman diagram is complete if the number of edges is $n/2$, i.e., $\mathcal{U}(\gamma) = \emptyset$. 
\end{definition}
\begin{figure}[h]
\centering
\begin{tikzpicture}[point/.style={circle,inner sep=1.2pt,fill=black}]
    \node[point,blue, thick, draw=blue, fill=none]  (a1) at (0,2) {$a_1$};
    \node[point,  red, thick, draw=red, fill=none] (a2) at (1.2,2) {$a_2$};
    
    \node[point,green, thick, draw=green, fill=none](a3) at (0,1) {$a_3$};
    \node[point,orange, thick, draw=orange, fill=none] (a4) at (1.2,1) {$a_4$};
    
    \node at (0.6,0.4) {$\gamma_1$};

\end{tikzpicture}
\hspace{0.7cm}
\begin{tikzpicture}[point/.style={circle,inner sep=1.2pt,fill=black}]
    \node[point,blue, thick, draw=blue, fill=none]  (a1) at (0,2) {$a_1$};
    \node[point,  red, thick, draw=red, fill=none] (a2) at (1.2,2) {$a_2$};
    
    \node[point,green, thick, draw=green, fill=none](a3) at (0,1) {$a_3$};
    \node[point,orange, thick, draw=orange, fill=none] (a4) at (1.2,1) {$a_4$};
    
    \node at (0.6,0.4) {$\gamma_2$};

    \draw[very thick]  (a1) -- (a2);

\end{tikzpicture}
\hspace{0.7cm}
\begin{tikzpicture}[point/.style={circle,inner sep=1.2pt,fill=black}]
    \node[point,blue, thick, draw=blue, fill=none]  (a1) at (0,2) {$a_1$};
    \node[point,  red, thick, draw=red, fill=none] (a2) at (1.2,2) {$a_2$};
    
    \node[point,green, thick, draw=green, fill=none](a3) at (0,1) {$a_3$};
    \node[point,orange, thick, draw=orange, fill=none] (a4) at (1.2,1) {$a_4$};
    
    \node at (0.6,0.4) {$\gamma_3$};

    \draw[very thick]  (a1) -- (a3);

\end{tikzpicture}
\hspace{0.7cm}
\begin{tikzpicture}[point/.style={circle,inner sep=1.2pt,fill=black}]
    \node[point,blue, thick, draw=blue, fill=none]  (a1) at (0,2) {$a_1$};
    \node[point,  red, thick, draw=red, fill=none] (a2) at (1.2,2) {$a_2$};
    
    \node[point,green, thick, draw=green, fill=none](a3) at (0,1) {$a_3$};
    \node[point,orange, thick, draw=orange, fill=none] (a4) at (1.2,1) {$a_4$};
    
    \node at (0.6,0.4) {$\gamma_4$};
    \draw[very thick]  (a1) -- (a4);

\end{tikzpicture}
\hspace{0.7cm}
\begin{tikzpicture}[point/.style={circle,inner sep=1.2pt,fill=black}]
    \node[point,blue, thick, draw=blue, fill=none]  (a1) at (0,2) {$a_1$};
    \node[point,  red, thick, draw=red, fill=none] (a2) at (1.2,2) {$a_2$};
    
    \node[point,green, thick, draw=green, fill=none](a3) at (0,1) {$a_3$};
    \node[point,orange, thick, draw=orange, fill=none] (a4) at (1.2,1) {$a_4$};
    
    \node at (0.6,0.4) {$\gamma_5$};

    \draw[very thick]  (a2) -- (a4);

\end{tikzpicture}

\begin{tikzpicture}[point/.style={circle,inner sep=1.2pt,fill=black}]
    \node[point,blue, thick, draw=blue, fill=none]  (a1) at (0,2) {$a_1$};
    \node[point,  red, thick, draw=red, fill=none] (a2) at (1.2,2) {$a_2$};
    
    \node[point,green, thick, draw=green, fill=none](a3) at (0,1) {$a_3$};
    \node[point,orange, thick, draw=orange, fill=none] (a4) at (1.2,1) {$a_4$};
    
    \node at (0.6,0.4) {$\gamma_6$};

    \draw[very thick]  (a2) -- (a3);

\end{tikzpicture}
\hspace{0.7cm}
\begin{tikzpicture}[point/.style={circle,inner sep=1.2pt,fill=black}]
    \node[point,blue, thick, draw=blue, fill=none]  (a1) at (0,2) {$a_1$};
    \node[point,  red, thick, draw=red, fill=none] (a2) at (1.2,2) {$a_2$};
    
    \node[point,green, thick, draw=green, fill=none](a3) at (0,1) {$a_3$};
    \node[point,orange, thick, draw=orange, fill=none] (a4) at (1.2,1) {$a_4$};
    
    \node at (0.6,0.4) {$\gamma_7$};

    \draw[very thick]  (a3) -- (a4);

\end{tikzpicture}
\hspace{0.7cm}
\begin{tikzpicture}[point/.style={circle,inner sep=1.2pt,fill=black}]
    \node[point,blue, thick, draw=blue, fill=none]  (a1) at (0,2) {$a_1$};
    \node[point,  red, thick, draw=red, fill=none] (a2) at (1.2,2) {$a_2$};
    
    \node[point,green, thick, draw=green, fill=none](a3) at (0,1) {$a_3$};
    \node[point,orange, thick, draw=orange, fill=none] (a4) at (1.2,1) {$a_4$};
    
    \node at (0.6,0.4) {$\gamma_8$};

    \draw[very thick]  (a1) -- (a2); 
    \draw[very thick]  (a3) -- (a4);

\end{tikzpicture}
\hspace{0.7cm}
\begin{tikzpicture}[point/.style={circle,inner sep=1.2pt,fill=black}]
    \node[point,blue, thick, draw=blue, fill=none]  (a1) at (0,2) {$a_1$};
    \node[point,  red, thick, draw=red, fill=none] (a2) at (1.2,2) {$a_2$};
    
    \node[point,green, thick, draw=green, fill=none](a3) at (0,1) {$a_3$};
    \node[point,orange, thick, draw=orange, fill=none] (a4) at (1.2,1) {$a_4$};
    
    \node at (0.6,0.4) {$\gamma_9$};

    \draw[very thick]  (a1) -- (a3); 
    \draw[very thick]  (a2) -- (a4);

\end{tikzpicture}
\hspace{0.7cm}
\begin{tikzpicture}[point/.style={circle,inner sep=1.2pt,fill=black}]
    \node[point,blue, thick, draw=blue, fill=none]  (a1) at (0,2) {$a_1$};
    \node[point,  red, thick, draw=red, fill=none] (a2) at (1.2,2) {$a_2$};
    
    \node[point,green, thick, draw=green, fill=none](a3) at (0,1) {$a_3$};
    \node[point,orange, thick, draw=orange, fill=none] (a4) at (1.2,1) {$a_4$};
    
    \node at (0.6,0.4) {$\gamma_{10}$};

    \draw[very thick]  (a1) -- (a4); 
    \draw[very thick]  (a2) -- (a3);

\end{tikzpicture}
\caption{For the set $g_{a_1},g_{a_2},g_{a_3},g_{a_4}$, there exists $10$ distinct Feynman diagrams.}
\end{figure}

\begin{definition}
    We define the homogeneous chaos of order-$n$ as
    \begin{align}
        \m{G}^{:n:} = \overline{\m{P}}_n \cap \overline{\m{P}}_{n-1}^{\perp}.
    \end{align}
    When $n=0$, $\m{G}^{:n:} = \m{P}_0$ which is the space of constants.
\end{definition}
Notice that by definition, $\m{G}^{:n:}\perp \m{G}^{:m:}$ for all $n\neq m$. In particular, this implies that if $X \in \m{G}^{:n:}$ for all $n \geq 1$, then $\E[X] = 0$.
\begin{definition}
    If $g_1,\dots,g_n$ are centered Gaussians in $\m{G}$, then the Wick product denoted by $\wick{g_1\cdots g_n}$ is defined as the orthogonal projection of $\prod_{i=1}^n g_i$ onto $ \m{G}^{:n:}$, i.e.,
    \begin{equation}
        \wick{g_1\cdots g_n} \equiv \pi_n(g_1 \cdots g_n).
    \end{equation}
\end{definition}
The following theorem provides a convenient way of computing the product of Wick products using only second moments of Gaussians.
\begin{lemma}\label{thm:Feynman}
    Given $\{ Y_i\}_{i=1}^k$ a set of Wick products such that for each $1\leq i\leq k$,
    $Y_i = \wick{g^{(i)}_{{j_1}} \cdots g^{(i)}_{{j_{N_i}}}}$, where $\{g^{(i)}_{j_t} \}_{t=1}^{N_i}$ are centered jointly Gaussians. The expectation of the product of Wick products can then be computed to be
    \begin{equation}
        \E [ Y_1\cdots Y_k] = \sum_{\gamma}v(\gamma),
    \end{equation}
    where the sum is taken across the set of all complete Feynman diagram $\gamma$ such that the only possible pairings are those from Gaussians from different groups, i.e., $(g^{(i)}_{j_t}, g^{(s)}_{j_\ell})$ is a valid pair if and only if $i\neq s$.
\end{lemma}
\noindent This generalizes the matching formula (Isserlis' theorem) for the mixed moments of a product of Gaussians, which corresponds to the case that all $N_i=1$.

\begin{figure}[h]
\centering
\begin{tikzpicture}[scale=0.9,point/.style={circle,inner sep=1pt,fill=black}]

    \node[point,blue, thick, draw = blue, fill = none] (v1) at (-0.5,0.866) {$a_1$};
    
    \node[point, red, thick, draw = red, fill = none] (v2) at (0.5,0.866) {$b_1$};
    
    \node[point, thick,red, draw = red, fill = none] (v3) at (1,0) {$b_2$};
    
    \node[point, thick, green, draw = green, fill = none] (v4) at (0.5,-0.866) {$c_2$};
    
    \node[point, thick,green, draw = green, fill = none] (v5) at (-0.5,-0.866) {$c_1$};
    
    \node[point, blue, thick, draw = blue, fill = none](v6) at (-1,0) {$a_2$};

    \path (v1) -- (v6) coordinate[midway] (mid);

    \draw[blue, thick, rotate around={-30:(mid)}] 
        (mid) ellipse [x radius=0.35, y radius=1];
        
     \path (v2) -- (v3) coordinate[midway] (mid23);
    \draw[red, thick, rotate around={30:(mid23)}] 
        (mid23) ellipse [x radius=0.35, y radius=1];

    \path (v4) -- (v5) coordinate[midway] (mid23);
    \draw[green, thick, rotate around={90:(mid23)}] 
        (mid23) ellipse [x radius=0.35, y radius=1];

   \draw[very thick]  (v1) -- (v2);   
   \draw[very thick]  (v6) -- (v5);  
   \draw[very thick]  (v3) -- (v4);  
\end{tikzpicture}
\hspace{0.5cm}
\begin{tikzpicture}[scale=0.9,point/.style={circle,inner sep=1pt,fill=black}]

    \node[point,blue, thick, draw = blue, fill = none] (v1) at (-0.5,0.866) {$a_1$};
    
    \node[point, red, thick, draw = red, fill = none] (v2) at (0.5,0.866) {$b_1$};
    
    \node[point, thick,red, draw = red, fill = none] (v3) at (1,0) {$b_2$};
    
    \node[point, thick, green, draw = green, fill = none] (v4) at (0.5,-0.866) {$c_2$};
    
    \node[point, thick,green, draw = green, fill = none] (v5) at (-0.5,-0.866) {$c_1$};
    
    \node[point, blue, thick, draw = blue, fill = none](v6) at (-1,0) {$a_2$};

    \path (v1) -- (v6) coordinate[midway] (mid);

    \draw[blue, thick, rotate around={-30:(mid)}] 
        (mid) ellipse [x radius=0.35, y radius=1];
        
     \path (v2) -- (v3) coordinate[midway] (mid23);
    \draw[red, thick, rotate around={30:(mid23)}] 
        (mid23) ellipse [x radius=0.35, y radius=1];

    \path (v4) -- (v5) coordinate[midway] (mid23);
    \draw[green, thick, rotate around={90:(mid23)}] 
        (mid23) ellipse [x radius=0.35, y radius=1];
   \draw[very thick]  (v1) -- (v2);   
   \draw[very thick]  (v3) -- (v5);  
   \draw[very thick]  (v4) -- (v6);  
\end{tikzpicture}
\hspace{0.5cm}
\begin{tikzpicture}[scale=0.9,point/.style={circle,inner sep=1pt,fill=black}]

    \node[point,blue, thick, draw = blue, fill = none] (v1) at (-0.5,0.866) {$a_1$};
    
    \node[point, red, thick, draw = red, fill = none] (v2) at (0.5,0.866) {$b_1$};
    
    \node[point, thick,red, draw = red, fill = none] (v3) at (1,0) {$b_2$};
    
    \node[point, thick, green, draw = green, fill = none] (v4) at (0.5,-0.866) {$c_2$};
    
    \node[point, thick,green, draw = green, fill = none] (v5) at (-0.5,-0.866) {$c_1$};
    
    \node[point, blue, thick, draw = blue, fill = none](v6) at (-1,0) {$a_2$};

    \path (v1) -- (v6) coordinate[midway] (mid);

    \draw[blue, thick, rotate around={-30:(mid)}] 
        (mid) ellipse [x radius=0.35, y radius=1];
        
     \path (v2) -- (v3) coordinate[midway] (mid23);
    \draw[red, thick, rotate around={30:(mid23)}] 
        (mid23) ellipse [x radius=0.35, y radius=1];

    \path (v4) -- (v5) coordinate[midway] (mid23);
    \draw[green, thick, rotate around={90:(mid23)}] 
        (mid23) ellipse [x radius=0.35, y radius=1];
   \draw[very thick]  (v1) -- (v3);   
   \draw[very thick]  (v6) -- (v4);  
   \draw[very thick]  (v2) -- (v5);  
\end{tikzpicture}
\hspace{0.5cm}
\begin{tikzpicture}[scale=0.9,point/.style={circle,inner sep=1pt,fill=black}]

    \node[point,blue, thick, draw = blue, fill = none] (v1) at (-0.5,0.866) {$a_1$};
    
    \node[point, red, thick, draw = red, fill = none] (v2) at (0.5,0.866) {$b_1$};
    
    \node[point, thick,red, draw = red, fill = none] (v3) at (1,0) {$b_2$};
    
    \node[point, thick, green, draw = green, fill = none] (v4) at (0.5,-0.866) {$c_2$};
    
    \node[point, thick,green, draw = green, fill = none] (v5) at (-0.5,-0.866) {$c_1$};
    
    \node[point, blue, thick, draw = blue, fill = none](v6) at (-1,0) {$a_2$};

    \path (v1) -- (v6) coordinate[midway] (mid);

    \draw[blue, thick, rotate around={-30:(mid)}] 
        (mid) ellipse [x radius=0.35, y radius=1];
        
     \path (v2) -- (v3) coordinate[midway] (mid23);
    \draw[red, thick, rotate around={30:(mid23)}] 
        (mid23) ellipse [x radius=0.35, y radius=1];

    \path (v4) -- (v5) coordinate[midway] (mid23);
    \draw[green, thick, rotate around={90:(mid23)}] 
        (mid23) ellipse [x radius=0.35, y radius=1];

   \draw[very thick]  (v1) -- (v3);   
   \draw[very thick]  (v6) -- (v5);  
   \draw[very thick]  (v2) -- (v4);  
\end{tikzpicture}

\vspace{0.5cm}
\begin{tikzpicture}[scale=0.9,point/.style={circle,inner sep=1pt,fill=black}]

    \node[point,blue, thick, draw = blue, fill = none] (v1) at (-0.5,0.866) {$a_1$};
    
    \node[point, red, thick, draw = red, fill = none] (v2) at (0.5,0.866) {$b_1$};
    
    \node[point, thick,red, draw = red, fill = none] (v3) at (1,0) {$b_2$};
    
    \node[point, thick, green, draw = green, fill = none] (v4) at (0.5,-0.866) {$c_2$};
    
    \node[point, thick,green, draw = green, fill = none] (v5) at (-0.5,-0.866) {$c_1$};
    
    \node[point, blue, thick, draw = blue, fill = none](v6) at (-1,0) {$a_2$};

    \path (v1) -- (v6) coordinate[midway] (mid);

    \draw[blue, thick, rotate around={-30:(mid)}] 
        (mid) ellipse [x radius=0.35, y radius=1];
        
     \path (v2) -- (v3) coordinate[midway] (mid23);
    \draw[red, thick, rotate around={30:(mid23)}] 
        (mid23) ellipse [x radius=0.35, y radius=1];

    \path (v4) -- (v5) coordinate[midway] (mid23);
    \draw[green, thick, rotate around={90:(mid23)}] 
        (mid23) ellipse [x radius=0.35, y radius=1];
   \draw[very thick]  (v1) -- (v4);   
   \draw[very thick]  (v2) -- (v5);  
   \draw[very thick]  (v6) -- (v3);  
\end{tikzpicture}
\hspace{0.5cm}
\begin{tikzpicture}[scale=0.9,point/.style={circle,inner sep=1pt,fill=black}]

    \node[point,blue, thick, draw = blue, fill = none] (v1) at (-0.5,0.866) {$a_1$};
    
    \node[point, red, thick, draw = red, fill = none] (v2) at (0.5,0.866) {$b_1$};
    
    \node[point, thick,red, draw = red, fill = none] (v3) at (1,0) {$b_2$};
    
    \node[point, thick, green, draw = green, fill = none] (v4) at (0.5,-0.866) {$c_2$};
    
    \node[point, thick,green, draw = green, fill = none] (v5) at (-0.5,-0.866) {$c_1$};
    
    \node[point, blue, thick, draw = blue, fill = none](v6) at (-1,0) {$a_2$};

    \path (v1) -- (v6) coordinate[midway] (mid);

    \draw[blue, thick, rotate around={-30:(mid)}] 
        (mid) ellipse [x radius=0.35, y radius=1];
        
     \path (v2) -- (v3) coordinate[midway] (mid23);
    \draw[red, thick, rotate around={30:(mid23)}] 
        (mid23) ellipse [x radius=0.35, y radius=1];

    \path (v4) -- (v5) coordinate[midway] (mid23);
    \draw[green, thick, rotate around={90:(mid23)}] 
        (mid23) ellipse [x radius=0.35, y radius=1];

   \draw[very thick]  (v1) -- (v4);   
   \draw[very thick]  (v6) -- (v2);  
   \draw[very thick]  (v3) -- (v5);  
\end{tikzpicture}
\hspace{0.5cm}
\begin{tikzpicture}[scale=0.9,point/.style={circle,inner sep=1pt,fill=black}]

    \node[point,blue, thick, draw = blue, fill = none] (v1) at (-0.5,0.866) {$a_1$};
    
    \node[point, red, thick, draw = red, fill = none] (v2) at (0.5,0.866) {$b_1$};
    
    \node[point, thick,red, draw = red, fill = none] (v3) at (1,0) {$b_2$};
    
    \node[point, thick, green, draw = green, fill = none] (v4) at (0.5,-0.866) {$c_2$};
    
    \node[point, thick,green, draw = green, fill = none] (v5) at (-0.5,-0.866) {$c_1$};
    
    \node[point, blue, thick, draw = blue, fill = none](v6) at (-1,0) {$a_2$};

    \path (v1) -- (v6) coordinate[midway] (mid);

    \draw[blue, thick, rotate around={-30:(mid)}] 
        (mid) ellipse [x radius=0.35, y radius=1];
        
     \path (v2) -- (v3) coordinate[midway] (mid23);
    \draw[red, thick, rotate around={30:(mid23)}] 
        (mid23) ellipse [x radius=0.35, y radius=1];

    \path (v4) -- (v5) coordinate[midway] (mid23);
    \draw[green, thick, rotate around={90:(mid23)}] 
        (mid23) ellipse [x radius=0.35, y radius=1];

   \draw[very thick]  (v1) -- (v5);   
   \draw[very thick]  (v6) -- (v2);  
   \draw[very thick]  (v3) -- (v4);  
\end{tikzpicture}
\hspace{0.5cm}
\begin{tikzpicture}[scale=0.9,point/.style={circle,inner sep=1pt,fill=black}]

    \node[point,blue, thick, draw = blue, fill = none] (v1) at (-0.5,0.866) {$a_1$};
    
    \node[point, red, thick, draw = red, fill = none] (v2) at (0.5,0.866) {$b_1$};
    
    \node[point, thick,red, draw = red, fill = none] (v3) at (1,0) {$b_2$};
    
    \node[point, thick, green, draw = green, fill = none] (v4) at (0.5,-0.866) {$c_2$};
    
    \node[point, thick,green, draw = green, fill = none] (v5) at (-0.5,-0.866) {$c_1$};
    
    \node[point, blue, thick, draw = blue, fill = none](v6) at (-1,0) {$a_2$};

    \path (v1) -- (v6) coordinate[midway] (mid);

    \draw[blue, thick, rotate around={-30:(mid)}] 
        (mid) ellipse [x radius=0.35, y radius=1];
        
     \path (v2) -- (v3) coordinate[midway] (mid23);
    \draw[red, thick, rotate around={30:(mid23)}] 
        (mid23) ellipse [x radius=0.35, y radius=1];

    \path (v4) -- (v5) coordinate[midway] (mid23);
    \draw[green, thick, rotate around={90:(mid23)}] 
        (mid23) ellipse [x radius=0.35, y radius=1];
   \draw[very thick]  (v1) -- (v5);   
   \draw[very thick]  (v6) -- (v3);  
   \draw[very thick]  (v2) -- (v4);  
\end{tikzpicture}

\caption{
The complete Feynman diagrams for $\E \left[ \wick{g_{a_1}g_{a_2}} \wick{g_{b_1}g_{b_2}} \wick{g_{c_1}g_{c_2}} \right]$.
}
\label{fig:wick_prod_compt}
\end{figure}

Since Wick products are projections of Gaussians onto lower-order polynomials, we may alternatively express them as linear combinations of Gaussian products.
\begin{lemma} \label{lemma:wick_decomp}
    The Wick product is given by
    \begin{equation}\label{eq:wick_decomp0}
        \wick{\prod_{i=1}^n g_i} = \sum_{\gamma}(-1)^{ \abs{\mathcal{P}(\gamma)} }v(\gamma),
    \end{equation}
    where the sum is taken over all Feynman diagrams.
\end{lemma}

\begin{exmp}\label{exmp:wick} Using Lemma \ref{lemma:wick_decomp}, we may compute the following Wick products:
    \begin{enumerate}
        \item $\wick{g_1} = g_1$,
        \item $\wick{g_1 g_2} = g_1g_2 -\E[g_1 g_2]$,
        \item $\wick{g_1 g_2 g_3} = g_1 g_2 g_3 -\E[g_2 g_3]g_1 - \E[g_1 g_3]g_2  - E[g_1 g_2]g_3$.
    \end{enumerate}
\end{exmp}

\noindent Using Lemma \ref{lemma:wick_decomp}, we can also write the product of Gaussians as a linear combination of Wick products.

\begin{lemma}\label{lemma:prod_Wick_decomp}
The product of $n$-Gaussian can be expressed as a linear combination of Wick products,
\begin{align}
    \prod_{i=1}^n g_i =  \sum_{\gamma} \prod_{(r,s)\in \mathcal{P}(\gamma)} \E [ g_r g_s] \wick{\prod_{i \in \mathcal{U}(\gamma)} g_i},
\end{align}
with sum taken over all Feynman diagrams $\gamma$.
\end{lemma}
\noindent 
See Figure \ref{fig:feyn_12} for an example of a Feynman diagram interpreted as a Wick product.
\begin{figure}[h] 
\centering

\begin{tikzpicture}[scale=0.9,point/.style={circle,inner sep=1.2pt,fill=black}]
\node[point,c1, thick, draw=c1, fill=none]  (a1) at (1,2) {$a_1$};

\node[point,c2, thick, draw=c2, fill=none]  (a2) at (2,2) {$a_2$};

\node[point,c3, thick, draw=c3, fill=none]  (a3) at (4,2) {$a_3$};

\node[point,c4, thick, draw=c4, fill=none]  (a4) at (5,2) {$a_4$};

\node[point,c5, thick, draw=c5, fill=none]  (a5) at (0.5,1) {$a_5$};

\node[point,c6, thick, draw=c6, fill=none]  (a6) at (1.5,1) {$a_6$};

\node[point,c7, thick, draw=c7, fill=none]  (a7) at (4.5,1) {$a_7$};
\node[point,c8, thick, draw=c8, fill=none]  (a8) at (5.5,1) {$a_8$};
\node[point,c9, thick, draw=c9, fill=none]  (a9) at (1,0) {$a_9$};

\node[point,c10, thick, draw=c10, fill=none]  (a10) at (2,0) {$a_{10}$};

\node[point,c11, thick, draw=c11, fill=none]  (a11) at (4,0) {$a_{11}$};

\node[point,c12, thick, draw=c12, fill=none]  (a12) at (5,0) {$a_{12}$};

\draw[very thick]  (a2) -- (a11);
\draw[very thick]  (a3) -- (a10);

\draw[very thick] (a1) .. controls (-0.3,1.8) and (-0.3,0.2) .. (a9);

\draw[very thick] (a4) .. controls (6.6,2) and (6.6,0) .. (a12);

\draw[very thick]  (a7) -- (a8);

\end{tikzpicture}
\hspace{0.5cm}
\begin{tikzpicture}[scale=0.9,point/.style={circle,inner sep=1.2pt,fill=black}]
\node[point,c1, thick, draw=c1, fill=none]  (a1) at (1,2) {$a_1$};

\node[point,c2, thick, draw=c2, fill=none]  (a2) at (2,2) {$a_2$};

\node[point,c3, thick, draw=c3, fill=none]  (a3) at (4,2) {$a_3$};

\node[point,c4, thick, draw=c4, fill=none]  (a4) at (5,2) {$a_4$};

\node[point,c5, thick, draw=c5, fill=none]  (a5) at (0.5,1) {$a_5$};

\node[point,c6, thick, draw=c6, fill=none]  (a6) at (1.5,1) {$a_6$};

\node[point,c7, thick, draw=c7, fill=none]  (a7) at (4.5,1) {$a_7$};
\node[point,c8, thick, draw=c8, fill=none]  (a8) at (5.5,1) {$a_8$};
\node[point,c9, thick, draw=c9, fill=none]  (a9) at (1,0) {$a_9$};

\node[point,c10, thick, draw=c10, fill=none]  (a10) at (2,0) {$a_{10}$};

\node[point,c11, thick, draw=c11, fill=none]  (a11) at (4,0) {$a_{11}$};

\node[point,c12, thick, draw=c12, fill=none]  (a12) at (5,0) {$a_{12}$};

\draw[very thick]  (a2) -- (a3);
\draw[very thick]  (a6) -- (a11);

\draw[very thick]  (a5) -- (a9);
\end{tikzpicture}

\caption{ The left and right Feynman diagrams correspond respectively to the Wick product given by  $\E[g_{a_1}g_{a_9}] \E[g_{a_2}g_{a_{11}}] \E[g_{a_3}g_{a_{10}}] \E[g_{a_7}g_{a_8}]\E[g_{a_4}g_{a_12}] \wick{g_{a_5}g_{a_6}}$
and  $\E[g_{a_5}g_{a_9}] \E[g_{{a_6}} g_{a_{11}}] \E[ g_{a_2} g_{a_3}] \wick{g_{a_1} g_{a_{4}} g_{a_7} g_{a_8} g_{a_{10}} g_{a_{12}} }.$
}\label{fig:feyn_12}
\end{figure}

\subsection{$L^p$ bounds of Wick Products}
\begin{lemma} \label{lemma:wick_cor}
    If $\{\m{W}_{i_j}\}_{j=1}^\ell$ are independent standard Gaussians, then 
    \[\E\left[\wick{\prod_{j=1}^\ell \W_{i_j}^{\pi_j}}\right] = 0 \quad  \text{and} \quad  
        \E\left[ \wick{\prod_{j=1}^\ell \W_{i_j}^{\pi_j}}^2\right] = \pi_1!\dots\pi_\ell!.\] Moreover, if $(i_1,\dots,i_\ell) \neq (s_1,\cdots,s_\ell)$ then $\wick{\prod_{j=1}^\ell \W_{i_j}^{\pi_j}}$ and $\wick{\prod_{j=1}^\ell \W_{s_j }^{\pi_j}}$ are uncorrelated.
\end{lemma}
\begin{proof}
Let $p = \sum_{j=1}^\ell \pi_j$, then the mean follows immediately from the fact that $\wick{\prod_{j=1}^\ell \W_{i_j}^{\pi_j}} \in \m{G}^{:p:}$. To compute the variance, it will be convenient to refer to the Gaussians that make up one copy of $\wick{\prod_{j=1}^\ell \W_{i_j}^{\pi_j}}$ as group $A$ and the Gaussians making up the second copy as group $B$. Then by Theorem \eqref{thm:Feynman},
\begin{align}
    \E \left[ \wick{\prod_{j=1}^\ell \W_{i_j}^{\pi_j}} \wick{\prod_{j=1}^\ell \W_{i_j}^{\pi_j}}\right] = \sum_{\gamma}v(\gamma),
\end{align}
such that the sum is taken across complete Feynman diagrams $\gamma$ such that Gaussians from group $A$ are only paired with Gaussians from group $B$. Note, each group consists of $p$-many Gaussians including repetition. Now, 
within each group, we may further group multiple copies of the same Gaussian into subgroups of $\pi_j$ which we will refer to as subgroup $A_j/B_j$, i.e., subgroup $A_j$ consists of $\pi_j$-copies of $\W_{i_j}$. 
Since Gaussians from group $A$ can only be paired with Gaussians from group $B$, if $\gamma$ is a Feynman diagram consisting of an edge that pairs Gaussians from $A_j$ to $B_t$, then by independence  $\E[\W_{i_j}\W_{i_t}] = \delta_{j,t}$. Which is to say that Feynman diagrams that mix different subgroups have $v(\gamma) = 0$. On the other hand, if the Feynman diagram only connects Gaussians from subgroups $A_j$ to $B_j$ for all $j\in [\ell]$, then $v(\gamma) = 1$.
It follows that there are $\pi_1! \dots \pi_\ell!$ many non-zero valued Feynman diagrams. This concludes the computation of the second moment.

If $(i_1<\dots<i_\ell) \neq (s_1<\dots<s_\ell)$ then there must exist a subgroup within each group that doesn't exist as an identical copy in the other. Which means that any Feynman diagram $\gamma$ must form at least one edge between non-identical subgroups. Thus, $v(\gamma) = 0$ for all $\gamma$ and 
\begin{equation}
    \E \left[ \wick{\prod_{j=1}^\ell \W_{s_j}^{\pi_j}} \wick{\prod_{j=1}^\ell \W_{t_j}^{\pi_j}} \right]=0.
\end{equation}
\end{proof}

\begin{lemma}\label{lemma:wick_Max_bd}
Let $\mathcal{I} = \{ \m{i}= (i_1,\cdots, i_\ell)\,;\, i_1<\cdots<i_\ell\, \text{ and } i_j \in [v] \}$ and $\wc \in \R^{|\mathcal{I}|}$ be a Wick vector defined by $\wc_{\m{i}} = \wick{\prod_{j=1}^\ell \W_{i_j}^{\pi_j}}$. If $p = \sum_{j=1}^\ell \pi_j$ and $\pi! = \pi_1! \cdots \pi_\ell!$ then there exists $C_\pi>0$ such that
\begin{equation}
    \prob \left( \left| \wc^\top  \wc - \pi! |\mathcal{I}| \right| >t \right) \leq e^2 \exp \left( - \left( \frac{t}{ C_\pi|\mathcal{I}|}\right)^{1/p} \right).
\end{equation}
Moreover, if $q>1$ and $\{\wc^{(a)}\}_{a \in [d]}$ are i.i.d. random Wick vectors in $\R^{|\mathcal{I}|}$, then for $d$ sufficiently large
\begin{equation}
    \E \left[\max_{a\in [d]} \left| \wc^\top \wc - \pi! |\mathcal{I}| \right|^q \right] \leq C_{\pi,q}  \abs{\mathcal{I}}^q \log^{pq}(d).
\end{equation}
\end{lemma}
\begin{proof}
    By Lemma \ref{eq:wick_decomp0}, $\wc^\top \wc$ is a Gaussian polynomial of degree-$2p$ with mean
    \begin{equation}
        \E \left [ \wc^\top \wc \right] = \sum_{\m{i}} \E \left[ \wc_{\m{i}}^2 \right] = \pi! \abs{\mathcal{I}}, 
    \end{equation}
and variance bound,
\begin{align}
    \Var\left [ \wc^\top \wc \right] &= \E \left [(\wc^\top \wc)^2 \right] - \E \left[\wc^\top \wc \right]^2 \\
    &= \sum_{\m{i},\m{s}} \E \left[  \wc_{\m{i}}^2 \wc_{\m{s}}^2\right] - \left( \pi! \abs{\mathcal{I}}\right)^2\\
    &\leq \max_{\m{i},\m{s}} \E \left[ \wc_{\m{i}}^2 \wc_{\m{s}}^2\right] \abs{\mathcal{I}}^2 -  \left( \pi! \abs{\mathcal{I}}\right)^2\\
    &\leq C_\pi \abs{\mathcal{I}}^2.
\end{align}
By Lemma \ref{lem:hypercontractivity},
\begin{equation}
     \prob \left( \left| \wc^\top  \wc - \pi! |\mathcal{I}| \right| >t \right) \leq e^2 \exp \left( - \left( \frac{t}{ C_\pi|\mathcal{I}|}\right)^{1/p} \right).
\end{equation}
For the $L^q$ bound, let $t_0= \abs{\mathcal{I}} \log(d)^p $ then
\begin{align}
    \E\left [ \max_{1\leq a \leq d} \abs{ \wc^{(a)\top} \wc^{(a)} - \pi! \abs{\mathcal{I}}}^q \right]    &= \int_0^\infty \prob \left( \max_{1\leq a \leq d}\left| \wc^{(a),\top}  \wc^{(a)} - \pi! |\mathcal{I}| \right|^q >t \right)\,\dif t \\
    &\leq  q t_0^{q}+qde^2 \int_{ \left( \frac{t_0}{C_\pi \abs{\mathcal{I}}} \right)^{1/p}}^\infty \left(C_\pi \abs{\mathcal{I}} t^{p} \right)^{q-1}\exp\left(-t \right)\,\dif t\\
    &\lesssim C_{\pi,q} \abs{\mathcal{I}}^q \log^{pq}(d)
    \end{align}
where we use the gamma tail estimate for $x\geq 2(p-1)$ and $p\geq 1$,\[\int_{x}^{\infty}  u^{p-1} e^{-u} du\leq 2 x^{p-1}e^{-x}.\] 
\end{proof}

\section{Lattice point counting}

In this section, we prove for completeness a lattice point counting problem, which gives the asymptotic behavior of the population covariance spectra, as well as an upper bound on Weyl Chamber version of the same counting problem (Theorem \ref{thm:ordered}).  Problems like these have been studied extensively in number theory, see for example \cite{apostol1976analytic,montgomery2006multiplicative,tenenbaum2015introduction}.

Fix $k\ge1$ and positive real numbers $\pi_1,\dots,\pi_k$.
Define
\[
\pi_*\coloneqq\min_{1\le i\le k}\pi_i,\qquad
m\coloneqq\#\{i:\ \pi_i=\pi_*\},\qquad
\sigma_0:=\frac{1}{\pi_*}.
\]
Define the unordered and ordered lattice point counts as
\[
A_k(X)\coloneqq \#\Bigl\{(s_1,\dots,s_k)\in\mathbb{N}^k:\ s_1^{\pi_1}\cdots s_k^{\pi_k}\le X\Bigr\},
\]
and respectively
\[
N_\downarrow(X)
\coloneqq\#\Bigl\{(s_1,\dots,s_k)\in\mathbb{N}^k:\ 1\leq  s_1<\cdots\leq s_k,\ s_1^{\pi_1}\cdots s_k^{\pi_k}\le X\Bigr\}.
\]
Let $\zeta(s)$ be the Riemann $\zeta$-function, which for $s > 1$ is given by the Dirichlet series
\[
\zeta(s) \coloneqq \sum_{n=1}^\infty \frac{1}{n^s}.
\]
We will use Tauberian theory to give the asymptotic behavior of the lattice point counting problems.

\begin{theorem}[Unordered tuples]\label{thm:unordered}
As $X\to\infty$,
\begin{equation}\label{eq:unordered}
A_k(X)\sim
\frac{\pi_*^{\,1-m}}{\Gamma(m)}
\Biggl(\prod_{\substack{1\le i\le k\\ \pi_i>\pi_*}}\zeta\!\Bigl(\frac{\pi_i}{\pi_*}\Bigr)\Biggr)
\,X^{1/\pi_*}\,(\log X)^{m-1}.
\end{equation}
In particular, if all $\pi_i$ are equal to a common value $\pi$ (so $m=k$), then
\[
A_k(X)\sim \frac{\pi^{\,1-k}}{\Gamma(k)}\,X^{1/\pi}(\log X)^{k-1}.
\]
\end{theorem}
For the ordered case, the count differs from the unordered case, save for when all $\pi_i$ are equal.
\begin{theorem}[Strictly decreasing tuples]\label{thm:ordered}
  Let $k\ge1$ and positive integers $\pi_1,\dots,\pi_k$. Define
  $\Pi_r \coloneqq \pi_k +\cdots+\pi_{k-r+1}$ and $\theta_r \coloneqq r/\Pi_r$, and set
  \[
  \theta^\star \coloneqq \max_{1\le r\le k}\theta_r.
  \]
  For the ordered count
  \[
  N_\downarrow(X) \coloneqq \#\{(1\leq s_1<\cdots<s_k):\ s_1^{\pi_1}\cdots s_k^{\pi_k}\le X\},
  \]
  let
  \[
  \mu \coloneqq \#\{\,r\in\{1,\dots,k\}:\ \theta_r=\theta^\star\,\}.
  \]
  Then, as $X\to\infty$,
  \[
  N_\downarrow(X)\ \lesssim  X^{\theta^\star}(\log X)^{\mu-1}.
  \]
  In the case that all $\pi_i$ are equal to a common value $\pi$,
  \[
    N_\downarrow(X)\sim \frac{\pi^{\,1-k}}{\Gamma(k)\Gamma(k+1)}\,X^{1/\pi}(\log X)^{k-1}.
  \]
\end{theorem}

\begin{remark}[\ep{The full main term of the eigenvalue count}]\label{rem:theory_curve}
\ep{
Theorems~\ref{thm:unordered} and~\ref{thm:ordered} give the leading order of the lattice counts. For eigenvalues of moderate index, which are the ones seen in numerical experiments, the lower-order terms matter. They can be computed exactly from the same Dirichlet series, by Perron's formula.

\emph{A model count.} Let $\cov_s=c\,s^{-\pwrco}$, let $f$ be a polynomial of degree $\p\le 3$, and let $b_k$ be its Wick coefficients at the mean variance $\tr\cov$ (that is, with $s_{\m{w}}=\m{w}^\top\cov\m{w}$ replaced by $\tr\cov$). The kernel $\kappa(\m{w},\m{w}')=\E_\x[f(\m{w}^\top\x)f(\m{w}'^\top\x)]$ is then $\sum_k k!\,b_k^2(\m{w}^\top\cov\m{w}')^k$. Expanding it in the monomials $\m{w}^S$, where $S$ ranges over the multisets of $[v]$ of size $k$, gives the values
\[
\mu_{k,S}=k!\,b_k^2\,m(S)\,\nu(S)\,c^k\prod_{s\in S}s^{-\pwrco},
\]
where $m(S)$ is the number of orderings of $S$ and $\nu(S)=\E[(\m{w}^S)^2]=\prod_c(2c-1)!!$, the product running over the multiplicities $c$ of the elements of $S$. If the monomials $\m{w}^S$ were orthogonal, these would be exactly the eigenvalues of the population operator. They are orthogonal for squarefree $S$, which carry the principal term; the monomials with repeated indices couple only through lower-order corrections ($w_a^2$ with $1$, $w_a^3$ with $w_a$).

\emph{Perron's formula.} Put $Y=\varepsilon^{-1/\pwrco}$ and $N(Y)=\#\{(k,S):\ k\ge 1,\ \mu_{k,S}\ge\varepsilon\}$; the constant term $k=0$ adds the single value $b_0^2$. Then $\mu_{k,S}\ge\varepsilon$ if and only if $\prod_{s\in S}s\le a_{k,S}\,Y$ with $a_{k,S}=(k!\,b_k^2\,m(S)\,\nu(S)\,c^k)^{1/\pwrco}$, so $N$ has the Dirichlet series
\[
Z(t)=\sum_{k,S}a_{k,S}^{\,t}\prod_{s\in S}s^{-t}.
\]
Grouping the multisets by their pattern of multiplicities, $Z$ is a combination of products of $\zeta(t)$, $\zeta(2t)$ and $\zeta(3t)$. For instance, the squarefree triples contribute $\frac16\big(\zeta(t)^3-3\zeta(t)\zeta(2t)+2\zeta(3t)\big)$ and the triples $\{a,a,b\}$, $a\ne b$, contribute $\zeta(t)\zeta(2t)-\zeta(3t)$. Only $\zeta(t)$ has a pole at $t=1$, of order $\p$ coming from the squarefree $\p$-sets. Perron's formula \cite[\S5.1]{montgomery2006multiplicative}, \cite[Lemma~3.12]{titchmarsh1986theory}, with the classical bounds for the general divisor problem \cite[Ch.~XII]{titchmarsh1986theory}, gives, for $v=\infty$,
\[
N(Y)=\operatorname*{Res}_{t=1}\Big[Z(t)\,\frac{Y^t}{t}\Big]+O\big(Y^{2/3}\log Y\big)=Y\,Q(\log Y)+O\big(Y^{2/3}\log Y\big),
\]
where $Q$ is a polynomial of degree $\p-1$. To compute it, expand $\zeta(t)=(t-1)^{-1}+\gamma-\gamma_1(t-1)+O((t-1)^2)$, with $\gamma$ Euler's constant and $\gamma_1$ the first Stieltjes constant. Write $a_{\m{1}_k}$ for the constant of the squarefree $k$-sets, so $a_{\m{1}_1}=(b_1^2c)^{1/\pwrco}$, $a_{\m{1}_2}=(4b_2^2c^2)^{1/\pwrco}$ and $a_{\m{1}_3}=(36b_3^2c^3)^{1/\pwrco}$, and write $a_{(2,1)}=(54b_3^2c^3)^{1/\pwrco}$ for the triples $\{a,a,b\}$. Then
\begin{align*}
Q(L)&=a_{\m{1}_1}+\frac{a_{\m{1}_2}}{2}\big(\ell_2+2\gamma-1\big)+\zeta(2)\Big(a_{(2,1)}-\frac{a_{\m{1}_3}}{2}\Big)\\
&\quad+\frac{a_{\m{1}_3}}{6}\Big(\frac{\ell_3^2}{2}+(3\gamma-1)\,\ell_3+1-3\gamma+3\gamma^2-3\gamma_1\Big),
\end{align*}
where $\ell_k=L+\log a_{\m{1}_k}$. The top coefficient, $a_{\m{1}_\p}/(\p!\,(\p-1)!)$, is the leading order given by \cref{thm:ordered}.

\emph{Eigenvalues.} Inverting $N$, the $j$-th largest value is $\varepsilon_j=Y_j^{-\pwrco}$ with $Y_j\,Q(\log Y_j)=j$, that is,
\[
\varepsilon_j=\Big(\frac{Q(L_j)}{j}\Big)^{\pwrco},\qquad L_j+\log Q(L_j)=\log j.
\]
Since $L_j=\log j-O(\log\log j)$, this is $\asymp(\log^{\p-1}(j)/j)^{\pwrco}$, the rate in Theorem~\ref{theorem:MAIN}. The lower coefficients of $Q$ are not small, however. For $f(y)=y^3$ and $\cov\propto\operatorname{diag}(j^{-1.31})$ with $\tr\cov=1$ and $v=2\cdot 10^4$ (as in \cref{fig:gaussian_monomials}), one finds $Q(L)=0.065L^2+0.063L+3.14$. The constant term, which comes from the degree-$1$ Wick component and the triples $\{a,a,b\}$, dominates until $L\approx 7$.

\emph{Numerics.} \Cref{fig:local_slope} shows the local slope $d\log\lambda_j/d\log j$ of the population spectrum for $f(y)=y^2$, $y^3$ and $\mathrm{He}_3(y)$, computed to $j=10^6$ with $v=10^6$. It is compared with the local slope of $\varepsilon_j$ above, continued to $j=10^{10}$, and with that of the rate in Theorem~\ref{theorem:MAIN}. From $j\approx 10$ on, the local slope of $\varepsilon_j$ agrees with that of the population spectrum to within a few hundredths. For $\p=2$ the population spectrum is in the asymptotic regime from $j\approx 300$ on. For $\p=3$, its local slope approaches that of $(\log^{2}(j+1)/j)^{\pwrco}$ only for $j$ of order $10^{9}$ to $10^{10}$, far beyond any computable $d$.
}
\end{remark}

\begin{figure}[ht]
  \centering
  \includegraphics[width=0.75\textwidth]{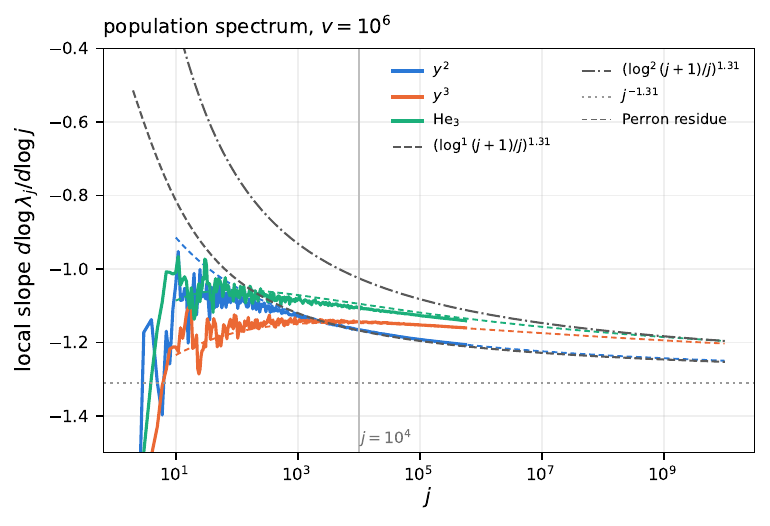}
  \caption{\ep{Local slope of the population spectrum for $\cov\propto\operatorname{diag}(j^{-1.31})$, $\tr\cov=1$, $v=10^6$. Solid: the population spectrum (top $10^6$ eigenvalues), slope over half-decade windows. Thin dashed, same colors: $\varepsilon_j$ from the full main term of Remark~\ref{rem:theory_curve}. Grey: $(\log^{\p-1}(j+1)/j)^{1.31}$ for $\p=2$ (dashed) and $\p=3$ (dash-dot); dotted: $j^{-1.31}$. Vertical line: $j=10^4$, the largest $d$ in \cref{fig:gaussian_monomials}.}}
  \label{fig:local_slope}
\end{figure}


\begin{lemma}\label{lemma:solve_X_from_N}
If all the $\pi_i$ are equal to a common value $\pi$ and $N_\downarrow(X)=N$, then
\[
X \asymp \left(\frac{N}{ \log(N)^{k-1}} \right)^{\pi}\qquad\text{as } X \to \infty.
\]
\end{lemma}
\begin{proof}
By Theorem \ref{thm:ordered}, we know that $N \asymp X^{1/\pi}(\log X)^{k-1}$. So if we take the limit of the ratio between $\log(X)$ and $\log(N)$, we see that
\begin{align}\label{eq:log(X)_log(N)}
    \lim_{X \to \infty} \frac{\log(X)}{\log(N)} &\asymp \lim_{X \to \infty } \frac{\log(X)}{\log\big(  X^{1/\pi}(\log X)^{k-1}\big) }\nonumber\\
    &= \lim_{X \to \infty } \frac{ \log(X)^{k-1}}{ \frac{1}{\pi} \log(X)^{k-1}+(k-1)\log(X)^{k-2} } = \pi.
\end{align}
Which is to say that for $X$ sufficiently large, there exist $c_1$ and $c_2$ depending only on $\pi$ such that 
\begin{equation}
    c_1 \log(N) \leq \log(X) \leq c_2 \log(N).
\end{equation}
It follows that 
\begin{equation}
    \frac{N}{\log(N)^{k-1}} \asymp \frac{N}{\log(X)^{k-1}} \asymp X^{1/\pi},
\end{equation}
as desired.
\end{proof}

\begin{corollary}\label{cor:N_lowe_bd}
    If $\pi_i$ are not all equal and $\pi^* = \max_{1\leq i\leq k} \pi_i$ then there exists constant $C_{\pi,k}>0$ such that as $X\to \infty$
    \begin{equation}
       C_{\pi,k} X^{1/\pi^*} \log(X)^{k-1} \leq  N_\downarrow(X).
    \end{equation}
\end{corollary}
\begin{proof}
    Let 
    \begin{equation}
        \overline{ N_\downarrow(X)} =  \#\{(1\leq s_1<\cdots<s_k):\ s_1^{\pi^*}\cdots s_k^{\pi^*}\leq X\},
    \end{equation}
then $ \overline{ N_\downarrow(X)} \leq N_\downarrow(X)$. By Theorem \ref{thm:ordered}, the result follows.
\end{proof}

\begin{lemma}\label{lemma:solve_X_from_N_2}
If $\pi_i$ are not identically equal and $N_\downarrow(X)=N$, then
\[
X\gtrsim \left( \frac{N}{\log(N)^{\mu-1}} \right)^{1/\theta^\star}.
\]
\end{lemma}

\begin{proof}
By Corollary \ref{cor:N_lowe_bd}, we know that $N$ is bounded from below by
\begin{equation}
    N \geq C_{\pi,k} X^{1/\pi^*}\log(X)^{k-1}.
\end{equation}
Thus, by \eqref{eq:log(X)_log(N)}
\begin{equation}
    \lim_{X \to \infty} \frac{\log(X)}{\log(N)} \leq \pi^*.
\end{equation}
That is, for $X$ sufficiently large,
$\log(X) \leq C_{\pi^*} \log(N)$ for some $C_{\pi^*}>0$. Combining this with Theorem \ref{thm:ordered} ,we see that
\begin{equation}
    N \lesssim X^{\theta^\star}\log(X)^{\mu-1} \lesssim X^{\theta_\star}\log(N)^{\mu-1}.
\end{equation}
Therefore, bounding $X$ by $N$ yields the result. 
\end{proof}

We will make use of the following Tauberian theorem of Delange \cite{delange1954ikehara} or \cite[20, Ch. III, Sec. 3]{narkiewicz1984number}.
\begin{theorem}[Delange's Tauberian theorem]\label{thm:delange}
Let $F(s)\coloneqq\sum_{n \geq 1} \theta(n) n^{-s}$ be a Dirichlet series with nonnegative coefficients and convergent and analytic for $\Re(s)>\varrho>0$. Assume that for $\beta > 0$, in a neighborhood of $\varrho$ intersected with $(\varrho,\infty)$
\[
F(s)=\frac{H(s)}{(s-\varrho)^\beta}+G(s)
\]
where $G$ and $H$ are analytic in a neighborhood of $s=\varrho$.
Then as $N \to \infty$,
\[
\sum_{n \leq N} \theta(n) \sim \frac{H(\varrho)}{\varrho \Gamma(\beta)} N^{\varrho}(\log N)^{\beta-1}.
\]
\end{theorem}
\begin{proof}[Proof of Theorem~\ref{thm:unordered}]
  Let $r(n)$ be the number of $k$-tuples $(s_1,\dots,s_k)\in\mathbb N^k$ with
  $s_1^{\pi_1}\cdots s_k^{\pi_k}=n$. Consider the Dirichlet series
  \[
  F(s)\coloneqq\sum_{n=1}^\infty \frac{r(n)}{n^s}
  =\sum_{s_1,\dots,s_k\ge1}\frac{1}{(s_1^{\pi_1}\cdots s_k^{\pi_k})^s}
  =\prod_{i=1}^k \zeta(\pi_i s),
  \qquad \Re s> \sigma_0.
  \]
  Each factor $\zeta(\pi_i s)$ has a simple pole at $s=1/\pi_i$ with principal part
  \[
  \zeta(\pi_i s)=\frac{1}{\pi_i}\cdot\frac{1}{s-1/\pi_i}+O(1).
  \]
  Therefore, $F$ has its rightmost singularity at $s=\sigma_0=1/\pi_*$, and the pole at $s=\sigma_0$ has order $m$ equal to the multiplicity of the minimum exponent. Writing the principal part near $s=\sigma_0$,
  \[
  F(s)=\frac{H(s)}{(s-\sigma_0)^m}+G(s),
  \]
  where $H$ is holomorphic near $\sigma_0$ with
  \[
  H(\sigma_0)=\Bigl(\prod_{\pi_i=\pi_*}\frac{1}{\pi_*}\Bigr)
  \prod_{\pi_i>\pi_*}\zeta\!\Bigl(\frac{\pi_i}{\pi_*}\Bigr)>0,
  \]
  and $G$ is holomorphic on a half-strip $\{\Re s>\sigma_0-\delta\}$ for some $\delta>0$ (which follows as $\zeta$ is meromorphic function with a simple pole at $s=1$).

  Since $r(n)\ge0$, we may apply Theorem \ref{thm:delange}. It yields
  \[
    \sum_{n\le X} r(n)
    \sim \frac{H(\sigma_0)}{\sigma_0\,\Gamma(m)}\,X^{\sigma_0}(\log X)^{m-1}
    \qquad (X\to\infty).
  \]
  But $\sum_{n\le X} r(n)=A_k(X)$ by definition, and $\sigma_0=1/\pi_*$. Hence
  \[
    A_k(X)\sim
    \frac{H(\sigma_0)}{(1/\pi_*)\,\Gamma(m)}\,X^{1/\pi_*}(\log X)^{m-1}
    =\frac{\pi_*^{\,1-m}}{\Gamma(m)}
    \Biggl(\prod_{\pi_i>\pi_*}\zeta\!\Bigl(\frac{\pi_i}{\pi_*}\Bigr)\Biggr)
    X^{1/\pi_*}(\log X)^{m-1},
  \]
  which is exactly \eqref{eq:unordered}.
\end{proof}
\begin{theorem}
    Let $k\ge1$ and positive integers $\pi_1,\dots,\pi_k$. Define
    $\Pi_r \coloneqq \pi_1+\cdots+\pi_r$ and $\theta_r \coloneqq r/\Pi_r$, and set
    \[
    \theta^\star \coloneqq \max_{1\le r\le k}\theta_r.
    \]
    For the ordered count
    \[
    N_\downarrow(X) \coloneqq \#\{(s_1>\cdots>s_k\ge1):\ s_1^{\pi_1}\cdots s_k^{\pi_k}\le X\},
    \]
    let
    \[
    \mu \coloneqq \#\{\,r\in\{1,\dots,k\}:\ \theta_r=\theta^\star\,\}.
    \]
    Then, as $X\to\infty$,
    \[
    N_\downarrow(X)\ \lesssim \ X^{\theta^\star}(\log X)^{\mu-1}.
    \]
    \end{theorem}
    
    \begin{proof}
    Put $M:=\lfloor\log_2 X\rfloor$. For each $i$, place $s_i$ into a dyadic
    interval,
    \[
    s_i\in [2^{m_i},2^{m_i+1})\qquad (m_i\in\mathbb Z_{\ge0}).
    \]
    The constraint $s_1^{\pi_1}\cdots s_k^{\pi_k}\le X$ implies
    \[
    \sum_{i=1}^k \pi_i m_i \ \le\ M,
    \]
    and the ordering $s_1>\cdots>s_k\ge1$ implies $m_1 \geq \cdots \geq m_k\ge0$.
    For any fixed dyadic pattern $(m_1,\dots,m_k)$, the number of integer
    $k$-tuples $(s_1,\dots,s_k)$ with $s_i\in[2^{m_i},2^{m_i+1})$ and
    $s_1>\cdots>s_k$ is at most
    \[
    \prod_{i=1}^k \#\,[2^{m_i},2^{m_i+1})\ \lesssim \ \prod_{i=1}^k 2^{m_i}
    \ =\ 2^{\sum_{i=1}^k m_i}.
    \]
    Therefore,
    \begin{equation}\label{eq:Nd-upper-dyadic}
    N_\downarrow(X)\ \lesssim \ \sum_{\substack{m_1 \geq \cdots \geq m_k\ge0\\ \sum \pi_i m_i\le M}}
    2^{\sum_{i=1}^k m_i}.
    \end{equation}
    
    \medskip
    \noindent\textbf{Step 1: Forward differences and two linear forms.}
    Introduce forward differences
    \[
    \Delta_j:=m_j-m_{j+1}\quad (j=1,\dots,k),\qquad m_{k+1}:=0.
    \]
    Then $m_j=\sum_{\ell\ge j}\Delta_\ell$, and so by using summation by parts, we have the identities
    \begin{equation}\label{eq:two-forms}
    \sum_{i=1}^k m_i=\sum_{j=1}^k j\,\Delta_j,
    \qquad
    \sum_{i=1}^k \pi_i m_i=\sum_{j=1}^k \Pi_j\,\Delta_j,
    \end{equation}
    with $\Pi_j \coloneqq \sum_{i=1}^j \pi_i$. Since $m_1 \geq \cdots \geq m_k\ge0$, we have
    \[
    \Delta_j\geq 0 \quad(1\leq j\leq k).
    \]
    Define $\theta_r \coloneqq r/\Pi_r$ and $\theta^\star \coloneqq \max_r \theta_r$, and for each $j$
    set the nonnegative ``deviation''
    \[
    \eta_j \coloneqq \theta^\star\Pi_j-j\ \ge\ 0.
    \]
    Using \eqref{eq:two-forms}, we have the exact decomposition
    \begin{equation}\label{eq:exact-identity}
    \sum_{i=1}^k m_i
    =\theta^\star\sum_{j=1}^k \Pi_j\Delta_j - \sum_{j=1}^k \eta_j\Delta_j.
    \end{equation}
    
    \medskip
    \noindent\textbf{Step 2: The exponential penalty.}
    Let
    \[
      J_{\mathrm{eq}}:=\{j:\eta_j=0\}=\{j:\theta_j=\theta^\star\},
      \qquad
      J_{\mathrm{str}}:=\{1,\dots,k\}\setminus J_{\mathrm{eq}},
    \]
    and let $\mu$ be the cardinality of $J_{\mathrm{eq}}$.
    Define the feasible set of $\{\Delta_j\}$ as 
    \[
    \mathcal{D} \coloneqq \{\Delta = (\Delta_1,\dots,\Delta_k) \in \mathbb{N}_0^k: \quad \sum_{j=1}^k \Pi_j \Delta_j \leq M\},
    \]
    and the slice $\mathcal{D}_{A,B}$ as the subset of $\mathcal{D}$ for which 
    \[
    \mathcal{D}_{A,B} \coloneqq \{\Delta \in \mathcal{D}: \sum_{j=1}^k \Pi_j \Delta_j = A, \quad \sum_{j\in J_{\mathrm{str}}} \Delta_j = B\}.
    \]
    On the feasible set $\mathcal{D}$, \eqref{eq:exact-identity} yields
    \begin{equation}\label{eq:key-penalty}
    2^{\sum_{i=1}^k m_i}
    \ \le\
    2^{\theta^\star \sum_{j=1}^k \Pi_j\Delta_j}\,
    2^{-\sum_{j=1}^k \eta_j\Delta_j}
    \leq X^{\theta^\star}\,2^{\theta^\star (\sum_{j=1}^k\Pi_j\Delta_j-M)-\sum_{j=1}^k \eta_j\Delta_j}.
    \end{equation}

    Inserting \eqref{eq:key-penalty} into \eqref{eq:Nd-upper-dyadic} and passing to
    $\Delta$-variables gives
    \begin{equation}\label{eq:split-sum}
    N_\downarrow(X)\ \lesssim \ X^{\theta^\star}
    \sum_{A,B=0}^M
    \sum_{\Delta \in \mathcal{D}_{A,B}}
    2^{\theta^\star (A-M)-\eta_* B},
    \end{equation}
    where $\eta_* = \min_{j\in J_{\mathrm{str}}} \eta_j$.

    \medskip
    \noindent\textbf{Step 3: The cardinality of $\mathcal{D}_{A,B}$.}
    
    The cardinality of $\mathcal{D}_{A,B}$ can be estimated above by specifying $(k-\mu)-1$ non-negative integers less than $B$ for the integers in $J_{\mathrm{str}}$ and $\mu-1$ non-negative integers less than $A$ for remainder (noting the sum constraints determine the last integer).  Hence,
    \[
    |\mathcal{D}_{A,B}| \leq (1+B)^{k-\mu-1} (1+A)^{\mu-1}.
    \]
    Hence using \eqref{eq:split-sum}, we have
    \[
    N_\downarrow(X)\ \lesssim \ X^{\theta^\star}
    \sum_{A,B=0}^M
    2^{\theta^\star (A-M)-\eta_* B}
    (1+A)^{\mu-1} (1+B)^{k-\mu-1}
    \lesssim X^{\theta^\star}
    M^{\mu-1}.
    \]
    Since $M=\lfloor\log_2 X\rfloor$, we have completed the proof.
  \end{proof}
  
  \begin{proof}[Proof of Theorem \ref{thm:ordered}]
    If all $\pi_i$ are equal to a common value $\pi$, then $\theta^\star=1/\pi$ and $\mu=k$. 
    Now we can connect the ordered and unordered cases, noting that the unordered count is the sum of the ordered counts for all permutations of the $\pi_i$, plus a correction term $D_k(X)$ due to the strict ordering.  This correct term can be covered by expressions $N_\downarrow(X;\pi')$ where we have fused some number of adjacent $\pi_{i}$ and $\pi_{i+1}$ to make $\pi'$.  Doing so either increases $\pi_*$ or decreases the number of $i$ so that $\pi_i = \pi_*$, and hence all such correction terms (bounding them above) are vanishing in order compared to $A_k(X)$.
    Hence
    \[
    N_\downarrow(X)\sim \frac{\pi^{\,1-k}}{\Gamma(k)\Gamma(k+1)}\,X^{1/\pi}(\log X)^{k-1}.
    \]
  \end{proof}
  \reb{
\subsection{Restricted lattice counts}
We consider a constrained variant of the lattice point counting problem, where we count only those tuples in $\N^{\p}$ whose entries are bounded by $v$. In the following section, we relate the restricted lattice count to the eigenvalue counting problem of $\cov^\pi$. 
\begin{lemma}\label{lemma:bounded_count}
    Let $\p \geq 2$ and $A_{\p}^v(X)= \# \{ (s_1,\dots,s_\p) \in [v]^{\p} : \, s_1 \cdots s_\p \leq X \}$ be the restricted unordered lattice count. If $\epsilon>0$ and $X \leq v^{\frac{\p}{1+\epsilon \p}}$ for $v$ sufficiently large, then there exist $C_{\p,\epsilon}>0$ such that $A^v_{\p}(X) \geq C_{\p,\epsilon} X \log^{\p-1}(X)$.
\end{lemma}
\begin{proof} We only need to prove the case that $X >v$, since when $X \leq v$, the restricted lattice count is the same as the unrestricted count $A_\p(X)$, so the claim follows by Theorem ~\ref{thm:unordered}. As such, we will assume that $X>v$ and that $v$ is sufficiently large.
    To estimate $A_\p^v(X)$, we consider the continuous extension of $A^v_{\p}(X)$, i.e., $\text{Vol}\left ( B^v_{\p}(X) \right)$ where $B^v_\p(X) =  \{ (s_1,\dots,s_\p)\in [1,v]^\p:\, s_1 \cdots s_\p \leq X \}$. It is clear that $A_{\p}^v(X) \sim \text{Vol}(B^v_{\p}(X)).$ Now, let 
    $0<\delta<\eta$ and 
    \begin{equation}
        \frac{1}{\p} - \frac{\epsilon}{\p-1} \leq \delta \quad \text{ and }\quad \eta \leq \min \left\{ \frac{1}{\p}+\epsilon, \frac{1}{\p-1} \right\}.
    \end{equation}
    Define
    \begin{multline*}
    \mathcal{B}_{\delta,\eta}(X) = \Big\{ (s_1,\dots,s_\p) \in [1,v]^\p: \, s_i \in \Big(X^\delta,\frac{X^\eta}{2} \Big)\, \forall i\in[\p-1],\\
    \text{ and } \frac{X}{2^{\p-1}s_1\cdots s_{\p-1}} \leq s_\p \leq  \frac{X}{s_1\cdots s_{\p-1}} \Big\}.
    \end{multline*}
    By the choices of $\delta$ and $\eta$, it is easy to check that $\left(X^\delta,\frac{X^\eta}{2} \right) \subseteq [1,v]$ and 
    \begin{align}
        s_\p &\leq X^{1 - (\p-1)\delta} \leq X^{1/\p +\epsilon} \leq v \\
        s_\p & \geq X^{1-\eta(\p-1)}\geq 1.
    \end{align}
    It follows
    \begin{align}
        \text{Vol}(B^v_{\p}(X)) &\geq \text{Vol}(\mathcal{B}_{\delta,\eta}(X)) \\
        & = \left(1- \frac{1}{2^{\p-1}} \right)X \left( \int_{X^\delta}^{\frac{X^\eta}{2}} \frac{1}{s} \,\dif s \right)^{\p-1}\\
        &\geq C_{\p,\epsilon} X \log^{\p-1}(X),
    \end{align}
    where $C_{\p,\epsilon} = \left(1 - \frac{1}{2^{\p-1}} \right) \left( \frac{\eta-\delta}{2} \right)^{\p-1}$ as desired.
\end{proof}
As with Theorem \ref{thm:ordered}, we may apply the bounds of the restricted unordered lattice count to obtain bounds for its ordered variant.
\begin{lemma} \label{lemma:restricted_count_1}
    Let $\p \geq 2$ and
    \[
    N_{\downarrow}^v(X) = \# \left\{ (1\leq s_1<\dots<s_\p\leq v): \, s_1\cdots s_\p \leq X \right\}.
    \]
    If $\epsilon>0$ and $X \leq v^{\frac{\p}{1+ \epsilon \p}}$ for $v$ sufficiently large, then there exists $C_{\p,\epsilon}>0$ such that $N_{\downarrow}^v(X)  \geq C_{\p,\epsilon} X \log^{\p-1}(X)$.
\end{lemma}
\begin{proof} It suffices to show the result for $X >v$ as the inequality is immediate by Theorem \ref{thm:ordered} in the case that $X \leq v$. If $X > v$, we may decompose $A^v_\p(X)$ into two sums; one sum is of ordered lattice counts with multiplicities and the second sum is of ordered counts without multiplicities. Observe that $N_{\downarrow}^v(X)$ is equal to the sum of ordered count without multiplicities up to some constant ($\p$ dependent) factor. For the counts with multiplicities, we may bound them from above by their corresponding unrestricted and unordered counts with the same multiplicities. By Theorem \ref{thm:unordered}, these unrestricted and unordered counts are further bounded from above by $O(X\log^{k-1}(X))$, for some $ k < \p$. However, by Lemma \ref{lemma:bounded_count}, $A^v_{\p}(X) \gtrsim_{\p,\epsilon} X \log^{\p-1}(X)$. Which is to say that the sum of ordered counts with multiplicities are dominated by $A_\p^v(X)$ when $X$ is sufficiently large. Now, writing the sum of ordered counts without multiplicities as the difference between $A_\p^v(X)$ and the sum of ordered counts with multiplicities, we see that  $N_{\downarrow}^v(X)  \gtrsim_{\p,\epsilon} X \log^{\p-1}(X)$ as desired. 
\end{proof}
\begin{corollary}\label{cor:restricted_count_2}
    Let $\p \geq 2$ and $\epsilon>0$. If $X \leq v^{\frac{\p}{1+\epsilon \p}}$ then $N_{\downarrow}^v(X) \asymp_{\p,\epsilon} X \log^{\p-1}(X)$. Moreover, if $N_{\downarrow}^v(X)= N$ then $X \asymp_{\p,\epsilon} \frac{N}{\log^{\p-1}(N+1)}$.
\end{corollary}
\begin{proof}
    Trivially, $N_\downarrow^v(X)$ is bounded from above by the unrestricted unordered lattice count $A_\p(X)$. The proof then proceeds identically to the argument of Lemma \ref{lemma:solve_X_from_N}.
\end{proof}
}
\section{The eigenvalue bounds of $\cov^\pi$}
Recall that $\cov$ has power-law spectrum with exponent $\pwrco$, $\pi = (\pi_1,\dots,\pi_\ell)$ is a composition of $\p$ and $\cov^\pi \in \R^{ {v \choose \ell}\times {v \choose \ell} }$ was defined to be 
\begin{equation}
    \cov^\pi = \text{Diag}\left( \{\cov^{\pi_1}_{i_1} \cdots \cov_{i_{\ell}}^{\pi_\ell} \}_{(i_1<\cdots<i_\ell} \right).
\end{equation}
Since $\cov_{j}\asymp j^{-\pwrco}$, it follows that there exist positive constants $\beta_1$ and $\beta_2$ depending only on $\ell$ so that for all $j\in \left[ {v \choose \ell}\right]$, there is an ordered tuple $(i_1<\cdots<i_\ell)$ so that
\begin{equation}\label{eq:H^pi_asymp}
    \beta_1 i_1^{-\pwrco \pi_1} \dots i_{\ell}^{-\pwrco \pi_\ell} \leq \lambda_j(\cov^\pi) \leq \beta_2 i_1^{-\pwrco \pi_1} \dots i_{\ell}^{-\pwrco \pi_\ell}.
\end{equation}
The goal in this section will be to derive bounds on the $j$-th eigenvalue of $\cov^\pi$ without any knowledge of the ordered tuple $(i_1<\cdots<i_\ell)$. The key to this will be the eigenvalue count, defined by
\begin{equation}
    N_\lambda(\epsilon) = \#\{ j :\,\lambda_j(\cov^\pi) \geq \epsilon \}.
\end{equation}
More specifically, if we can solve for $\epsilon$ so that $N_{\lambda}(\epsilon) = j$, then we have that $\lambda_j(\cov^\pi) \geq \epsilon$. By \eqref{eq:H^pi_asymp}, the eigenvalue count is related to a restricted ordered lattice-point count,
\begin{equation}
N_\downarrow^v(X) = \{(i_1,\dots,i_\ell) \in \N^\ell\,: \, 1\leq i_1<\cdots<i_\ell\leq v, \,i_1^{\pi_1}\cdots i_\ell^{\pi_\ell} \leq X \},
\end{equation}
as
\begin{equation}\label{eq:eig_count_bd_by_bdd_lattice}
    N_\downarrow^v (\beta_1 \epsilon^{-1/\pwrco}) \leq N_\lambda(\epsilon) \leq N_\downarrow^v(\beta_2 \epsilon^{-1/\pwrco}).
\end{equation}
Due to the additional constraint of $i_j \in [v]$, we cannot immediately apply the results of the previous section to obtain asymptotics bounds on $N_\lambda$ and $\epsilon$. Fortunately, we circumvent the bounded constraint by leveraging the following simple observation:
\begin{equation}\label{eq:simple_observe}
    N^v_\downarrow(X) \leq N_\downarrow(X)\, \text{ for all $X> 0$ and if $X\leq v$ then equality holds.}
\end{equation}
And in the case that $\pi = \m{1}_\p$, we may directly apply the bounds Lemma \ref{lemma:restricted_count_1} and Corollary \ref{cor:restricted_count_2}.
\begin{lemma}\label{lemma:Id_pi_eig_bounds}
    Let 
    $\cov^\pi \in \R^{ {v \choose \ell } \times {v \choose \ell } }$ be a diagonal matrix as described above such that all the $\pi_i$ are identical to common $\pi$. If $k_*\geq 1$ such that $k_*^\pi \leq v$, and both $k_*$ and $v$ are sufficiently large, then for all $1 \leq j \leq k_*$, 
    \begin{equation}\label{eq:eig_uppr_lower_count1}
        \lambda_j(\cov^\pi) \asymp \left( \frac{\log^{\ell-1}(j+1)}{j} \right)^{\pwrco \pi}.
    \end{equation}
    Moreover, if $j > k_*$ then
    \begin{equation}\label{eq:eig_uppr_lower_count2}
        \lambda_j(\cov^\pi) \lesssim \left( \frac{\log^{\ell-1}(j+1)}{j} \right)^{\pwrco \pi}.
    \end{equation}
\end{lemma}

\begin{proof}
We start by choosing for each $j\in \left[ {v \choose \ell}\right]$, a $\delta_j$ and $\epsilon_j$ so that 

\noindent $N_\downarrow^v\left( \beta_1 \delta_j^{-1/\pwrco} \right) = N_\downarrow^v( \beta_2 \epsilon_j^{-1/\pwrco}) = j$. We will assume that $\delta_j \leq \epsilon_j$ for all $j$, as we may choose $\delta_j = \left(\frac{\beta_1}{\beta_2}\right)^\pwrco \epsilon_j$. It is easy to check that both 
$\epsilon_j$ and $\delta_j$ are decreasing monotonically. By \eqref{eq:eig_count_bd_by_bdd_lattice}, we see that 
\begin{equation}
   j \leq N_\lambda(\delta_j)\quad \text{and} \quad  N_\lambda(\epsilon_j) \leq j .
\end{equation}
In particular, this implies that there are at least $j$-eigenvalues of $\cov^\pi$ that are greater than or equal to $\delta_j$ and there are at most $j-1$ eigenvalues greater than or equal to $\epsilon_{j-1}$. Given that $\delta_j\leq \epsilon_j<\epsilon_{j-1}$, this means the $j$-th eigenvalue of $\cov^\pi$ must lie in between $\delta_j$ and $\epsilon_{j-1}$. That is,
\begin{equation}\label{eq:j-th_eig_delta_eps_bdd}
    \delta_j \leq \lambda_j(\cov^\pi) <\epsilon_{j-1}.
\end{equation}
It remains to estimate $\delta_j$ and $\epsilon_j$, for which the unbounded lattice count will prove useful. Let $\epsilon>0$ and $N = N_\downarrow(\epsilon^{-1/\pwrco})$. By Lemma \ref{lemma:solve_X_from_N}, there exist positive constants $c_1$ and $c_2$ depending only on $\pi$ so that 
\begin{equation}\label{eq:count_uppr_lower_count}
    c_1\left(\frac{N}{\log^{\ell-1}(N)} \right)^\pi \leq \epsilon^{-1/\pwrco} \leq c_2 \left( \frac{N}{\log^{\ell-1}(N)} \right)^\pi.
\end{equation}
Notice that if $N$ is sufficiently large and $N^\pi\leq v$, then $\epsilon^{-1/\pwrco}\leq v$. Let $k_*$ be such an $N$ and
and $\epsilon_*>0$ such that $N_\downarrow( \epsilon_*^{-1/\pwrco} )=k_*$. Since the choice of $k_*$ and $\epsilon_*$ are independent to the choice of $\delta_j$ and $\epsilon_j$ in the steps prior, we may assume that $\delta_{k_*}$ and $\epsilon_{k_*}$ were chosen so that $\delta_{k_*}=\beta_1^\pwrco \epsilon_*$ and $\epsilon_{k_*}=\beta_2^\pwrco \epsilon_*$. Thus by \eqref{eq:simple_observe}, we see that for all $j \leq k_*$
\begin{equation}
    N_\downarrow ( \beta_1 \delta_j^{-1/\pwrco}) =  N_\downarrow^v(\beta_1 \delta_j^{-1/\pwrco}) = j = N_\downarrow^v(\beta_2 \epsilon_j^{-1/\pwrco}) = N_\downarrow(\beta_2 \epsilon_j^{-1/\pwrco}).
\end{equation}
Using \eqref{eq:count_uppr_lower_count}, we may bound $\delta_j$ from below and $\epsilon_j$ from above by
\begin{equation}
    \frac{\beta_1^\pwrco}{c_2^{\pwrco}}\left( \frac{\log^{\ell-1}(j)}{j} \right)^{\pwrco \pi} \leq \delta_j \quad \text{and} \quad \epsilon_j \leq \frac{\beta_2^\pwrco}{c_1^\pwrco}\left( \frac{\log^{\ell-1}(j)}{j} \right)^{\pwrco \pi}.
\end{equation}
Combining this with \eqref{eq:j-th_eig_delta_eps_bdd} shows \eqref{eq:eig_uppr_lower_count1} for all $j\leq k_*$. If $j > k_*$, we have the trivial bound
\begin{equation}  
j -1 =  N^v_\downarrow(\beta_2 \epsilon_{j-1}^{-1/\pwrco}) \leq N_\downarrow(\beta_2 \epsilon_{j-1}^{-1/\pwrco}) \eqqcolon N_{j-1}.
\end{equation}
By \eqref{eq:count_uppr_lower_count},
\begin{equation}
    \epsilon_{j-1} \leq \frac{\beta_2^\pwrco}{c_1^\pwrco}\left( \frac{\log^{\ell-1}(N_{j-1})}{N_{j-1}} \right)^{\pwrco \pi}.
\end{equation}
Notice the function $f: t \mapsto \frac{\log^{\ell-1}(t)}{t}$ is decreasing if $t>e^{\ell-1}$. Thus, if $v$ is sufficiently large so that $k_*>e^{\ell-1}$ and $k_*^\pi \leq v$, then $f(N_{j-1}) \leq f(j-1)$. Therefore, 
\begin{equation}
    \epsilon_{j-1} \leq \frac{\beta_2^\pwrco}{c_1^\pwrco}\left( \frac{\log^{\ell-1}(j-1)}{j-1} \right)^{\pwrco \pi}. 
\end{equation}
Plugging this into \eqref{eq:j-th_eig_delta_eps_bdd} concludes \eqref{eq:eig_uppr_lower_count2} for all $j>k_*$.
\end{proof}

\begin{lemma}\label{lemma:principal_eig_bounds}
    Let $\pi = \m{1}_\p$, $\epsilon >0$ and $v \geq d^{\frac{1}{\p}+\epsilon }$. If $d$ is sufficiently large,
    then for all $j \in [d]$,
\begin{equation}\label{eq:eig_v_lessthan_d_regime}
        \lambda_j(\cov^{\m{1}_{\p}} ) \asymp  \left( \frac{\log^{\p-1}(j+1)}{j}\right)^{\pwrco}.
    \end{equation}
    Moreover, if $j > d$ then\begin{equation}\label{eq:eig_uppr_lower_count2}
        \lambda_j(\cov^{\m{1}_\p}) \lesssim \left( \frac{\log^{\p-1}(j+1)}{j} \right)^{\pwrco }.
    \end{equation}
\end{lemma}
\begin{proof}
    We first verify that $d \leq {v \choose \p}$. This is an easy check by simply expanding 
    \begin{align}
        {v \choose \p} = \frac{v(v-1)\cdots (v-\p+1) }{\p!} \geq \frac{v^{p}}{2^{\p-1}\p!} \geq \frac{d^{1+\epsilon \p}}{d^{\p-1}\p!} \geq d.
    \end{align}
    To show \eqref{eq:eig_v_lessthan_d_regime}, we observe that we only need to prove the lower bound as the upper bound holds for all $j$ by Lemma \ref{lemma:Id_pi_eig_bounds}. We proceed in a similar fashion as the proof of Lemma \ref{lemma:Id_pi_eig_bounds}. Let $\delta_j \geq \frac{\beta_1}{d^\pwrco}$ so that $N_{\downarrow}^v \left(\left( \beta_1/\delta_j \right)^{1/\pwrco} \right) = j$.
   It is easy to check that $\left(\beta_1/\delta_j \right)^{1/\pwrco} \leq d \leq v^{\frac{\p}{1+\p \epsilon}} $. Thus, by Corollary \ref{cor:restricted_count_2} there exists constant $C_{\p,\epsilon}>0$ so that 
    \begin{equation}
          C_{\p,\epsilon} \left( \frac{\log^{\p-1}(j+1)}{j} \right)^\pwrco \leq \delta_j.
    \end{equation}
     Since $N_\lambda(\delta_j) \geq N_{\downarrow}^v \left(\left( \beta_1/\delta_j \right)^{1/\pwrco} \right)=j$, it follows that
       $ \lambda_j(\cov^{\m{1}_\p}) \geq \delta_j$ which concludes the proof.
\end{proof}

\reb{
\begin{lemma}\label{lemma:tail_eig_bound}
    Let $\mathcal{I}$ be the set of all ordered tuple pairs $(\m{i},\m{s})$ such that $\m{i} = (  i_1<\cdots<i_\ell )$, $\m{s} = (  s_1<\cdots<s_r )$, $\m{i}\cap \m{s} = \emptyset$ and $i_j,s_t \in [v]$. Let $(\pi^1,\pi^2) = (\pi^1_1,\cdots, \pi^1_\ell,\pi_1^2,\cdots,\pi^2_t)$,
    \begin{equation}
        \pi_{\min} = \min\{ \pi^1_1,\cdots, \pi^1_\ell,\pi_1^2,\cdots,\pi^2_t \} \quad \text{and} \quad m = \# \{ i: \pi^1_i = \pi_{\min} \}+\# \{ i: \pi^2_i = \pi_{\min} \}.
    \end{equation}
    If $\cov^{(\pi^1,\pi^2)} \in \R^{\abs{\mathcal{I}} \times \abs{\mathcal{I}} }$ defined by
    \begin{equation}
        \cov^{(\pi^1,\pi^2)} = \text{Diag}\left( \left\{ \prod_{j=1}^\ell \cov_{i_j}^{\pi^1_j} \prod_{t=1}^r \cov_{s_t}^{\pi^2_t} \right\}_{(\m{i},\m{s})} \right),
    \end{equation}
    then for all $j \geq 1$,
    \begin{equation}
        \lambda_j\left(\cov^{(\pi^1,\pi^2)}  \right) \lesssim \left( \frac{\log^{m-1}(j+1)}{j} \right)^{\pi_{\min} \pwrco}.
    \end{equation}
\end{lemma}
\begin{proof}
    In contrast to Lemma \ref{lemma:Id_pi_eig_bounds}, the double indices of $(\m{i},\m{s})$ are ordered internally within their respective tuples but no ordering is imposed between them. Since we only need an upper bound, we may work around this by simply bounding the eigenvalue count by the unordered lattice count,
    \begin{equation}
        A_{\ell+r}(X) = \# \left \{  (i_1,\cdots,i_\ell)\in \N^{\ell}, \, (s_1,\cdots,s_r)\in \N^r: \prod_{j=1}^\ell i_j^{\pi_j^1} \prod_{t=1}^r s_t^{\pi^2_t} \leq X \right\}.
    \end{equation}
    It follows that $N_\lambda(\epsilon) \leq N_{\downarrow}^v\left((\beta_2/\epsilon)^{1/\pwrco}\right) \leq A_{\ell+r}\left((\beta_2/\epsilon)^{1/\pwrco}\right)$. Choose $\epsilon_j$ such that $A_{\ell+r}\left((\beta_2/\epsilon)^{1/\pwrco}\right) = j$,  then by Lemma \ref{lemma:solve_X_from_N_2}, 
    \begin{equation}
        \epsilon_j \lesssim \left( \frac{\log^{m-1}(j+1)}{j} \right)^{\pi_{\min}\pwrco}.
    \end{equation}
    Combining this with $\lambda_j\left( \cov^{(\pi^1,\pi^2)} \right) \leq \epsilon_{j-1}$ concludes the proof.
\end{proof}
}

\reb{
\section{Covariance Matrix vs. Second-Moment Matrix}\label{Appendix:cov_vs_2nd}
In this section, we compare the spectrum of the covariance matrix against the spectrum of its second moment. Recall,
 the covariance matrix of $\frac{1}{\sqrt{d}}f(\W^\top \x)$ with respect to $\x$ is given by
\begin{equation}\label{eq:cov_f}
    \text{Cov}\left( \frac{1}{\sqrt{d}}f(\W^\top \x) \right) = \frac{1}{d}\E_\x \left[ f(\W^\top \x)^{\otimes 2} \right] - \frac{1}{d}\E_\x \left[f(\W^\top \x) \right]^{\otimes 2},
\end{equation}
where the latter term is a rank-$1$ perturbation. We will first show that this rank-$1$ perturbation is $O(1)$ in operator norm, or equivalently, its trace is $O(1)$. Computing the trace we get
\begin{align}
    \frac{1}{d}\tr \left(  \E_\x \left[f(\W^\top \x) \right]^{\otimes 2} \right) &= \frac{1}{d}\sum_{i=1}^d \ep{\left(\E_{\x}\left[ f\left(\inner{\W_{\col(i)}}{\x} \right)\right]\right)^2}\\
    &= \frac{1}{d}\sum_{i=1}^d \ep{\left(\E \left[ f\left( \inner{\W_{\col(i)}}{\x} \right)  \vert \W_{\col(i)}\right ]\right)^2}.
\end{align}
Since the columns of $\W$ are \ep{i.i.d., the random variables $\left(\E \left[ f\left( \inner{\W_{\col(i)}}{\x} \right)  \vert \W_{\col(i)}\right ]\right)^2$, $1\leq i\leq d$, are i.i.d.; they have finite variance because $f$ is a polynomial and $\W$ is Gaussian. By Chebyshev's inequality}, for all $\epsilon>0$,
\begin{equation}
    \prob \left( \abs{\frac{1}{d}\tr\left( \E_\x \left[f(\W^\top \x) \right]^{\otimes 2} \right)- \ep{\E\left[ \left(\E_\x\left[ f(\W^\top \x)_1 \right]\right)^2 \right]}  } > \epsilon \right) \leq O_\epsilon \left( \frac{1}{d}\right).
\end{equation}
Which is to say that $\frac{1}{d}\E_\x \left[f(\W^\top \x) \right]^{\otimes 2}$ is $O(1)$ with high probability. We now compare the spectrum between the covariance \eqref{eq:cov_f} and the second moment matrix.  Since the rank-$1$ perturbation is negative semidefinite, it is immediate that
\begin{equation}
    \lambda_j\left( \text{Cov}\left( \frac{1}{\sqrt{d}}f(\W^\top \x) \right) \right) \leq \lambda_j \left( \E_{\x}\left[ \frac{1}{d} f(\W^\top \ep{\x})^{\otimes 2}   \right]\right).
\end{equation}
As for the lower bound, upon noting that a rank-$1$ matrix only has a single non-zero eigenvalue, we see that by Weyl's inequality, 
\begin{align}
   \lambda_{j+1} \left( \E_{\x}\left[ \frac{1}{d} f(\W^\top \ep{\x})^{\otimes 2}   \right]\right) &\leq \lambda_j\left( \text{Cov}\left( \frac{1}{\sqrt{d}}f(\W^\top \x) \right) \right)  + \lambda_2 \left( \frac{1}{d}\E_\x \left[ f(\W^\top \x) \right]^{\otimes 2} \right)\\
   &= \lambda_j\left( \text{Cov}\left( \frac{1}{\sqrt{d}}f(\W^\top \x) \right) \right).
\end{align}
By Theorem \ref{theorem:MAIN}, the first $O(d \log^{-\p-1}(d))$ eigenvalues of the second moment matrix follows a weak power-law with exponent $\pwrco$, thus the weak power-law structure extends to the covariance matrix of $\frac{1}{\sqrt{d}}f(\W^\top \x)$ as well.
}

\section{\ep{Power-law persistence across layers}}\label{sec:across_layers}
\ep{
This section follows the covariance spectrum of CIFAR-10 \cite{krizhevsky2009learning} through the layers of neural networks, at random initialization and after training. Theorem~\ref{theorem:MAIN} describes a single random-feature layer, so these experiments go beyond it: they show what happens to the exponent, rather than testing a prediction.

\paragraph{Setup}
For each network and each layer, we compute the covariance of the layer's features over the $50{,}000$ training images, flattening the feature maps of convolutional layers. Every feature dimension (at most $16{,}384$) is below the number of images, so the whole spectrum is estimated. The exponent of a layer is the slope of an ordinary least-squares fit of $\log\lambda_j$ against $\log j$ over $11\le j\le 200$, the same window for every layer. The window starts at $j=11$ because in trained networks the leading eigenvalues carry the ten class means of CIFAR-10, which would otherwise dominate the fit. For the input covariance the slope is $-1.35$, with $R^2=0.999$. Slopes depend on the fit window, so only comparisons within this one window are meaningful; slopes over other windows are reported with the code \cite{paquette2026code}. We report the median and range over $5$ seeds for random networks and $3$ seeds for trained ones.

\paragraph{Networks}
\begin{itemize}[leftmargin=*,itemsep=2pt]
\item \emph{Random MLPs} of depth $4$ and width $1024$, at a controlled scale. The input is centered CIFAR-10, divided by one constant so that the mean of $|\x|^2$ is $1$. Layer $\ell$ maps $\m{h}_{\ell-1}$ to $f(c_\ell \W_\ell^\top \m{h}_{\ell-1})$, with $\W_\ell$ i.i.d.\ $N(0,1)$, no biases, and $c_\ell$ fixed once from the data so that the pre-activations have mean square $1$. Each layer is then the random-feature model of the main text, applied to the features of the previous layer. The activations are ReLU and $\tanh$, and for ReLU also ReLU followed by RMSNorm or by LayerNorm.
\item \emph{A VGG-style CNN} \cite{simonyan2014very} with four blocks (two convolution--batch-norm--ReLU layers and a max-pool each, with $64$, $128$, $256$ and $512$ channels), with features taken after each block.
\item \emph{ResNet-20} \cite{he2016deep}, with features taken after the first convolution and after each of the nine residual blocks.
\item \emph{Trained networks}: a ReLU MLP of the same shape (with PyTorch's default initialization, biases, and inputs scaled to $[-1,1]$), the CNN and ResNet-20. Each is trained with SGD (learning rate $0.1$, momentum $0.9$, weight decay $10^{-4}$, cosine schedule, batch size $256$, $50$ epochs, no augmentation). The test accuracies are about $58\%$, $87\%$ and $84\%$.
\end{itemize}

\paragraph{Results}
\begin{enumerate}[leftmargin=*,itemsep=2pt]
\item Every layer of every network has a spectrum close to a power law over $11\le j\le 200$ (\cref{fig:layer_spectra}), with $R^2\ge 0.977$.
\item Random MLPs flatten the exponent layer by layer (\cref{fig:mlp_activation}). After four layers the slope is $-0.95$ for ReLU and $-1.07$ for $\tanh$. LayerNorm flattens it faster, to $-0.70$; RMSNorm only slightly, to $-0.92$.
\item Among random networks, the slope of the CNN varies between $-1.05$ and $-1.39$ and ends at $-1.05$, while that of ResNet-20 stays between $-1.13$ and $-1.34$, approaching the input's $-1.35$ in its last three blocks (\cref{fig:across_layers}).
\item Training flattens the intermediate layers, to as little as $-0.33$ (ResNet-20), $-0.56$ (CNN) and $-0.89$ (MLP), and steepens the last one. The last-layer slope is $-1.40$ for ResNet-20, $-1.46$ for the CNN and $-1.53$ for the MLP, against the input's $-1.35$.
\end{enumerate}

\paragraph{The scale of the pre-activations}
For activations that are not positively homogeneous, the scale of the pre-activations matters. With PyTorch's default initialization (weights and biases uniform on $\pm 1/\sqrt{\text{fan-in}}$) and inputs in $[-1,1]$, the pre-activations of the random $\tanh$ MLP have root mean square $0.29$, $0.16$, $0.09$ and $0.055$ in layers $1$ to $4$. At that scale $\tanh$ is close to linear, so the network is nearly linear and its layers hardly probe the nonlinearity; this is why the random MLPs above use the controlled scale. ReLU is positively homogeneous, and the two setups give ReLU slopes within $0.03$ of each other.
}

\begin{figure}[t]
  \centering
  \includegraphics[width=0.8\textwidth]{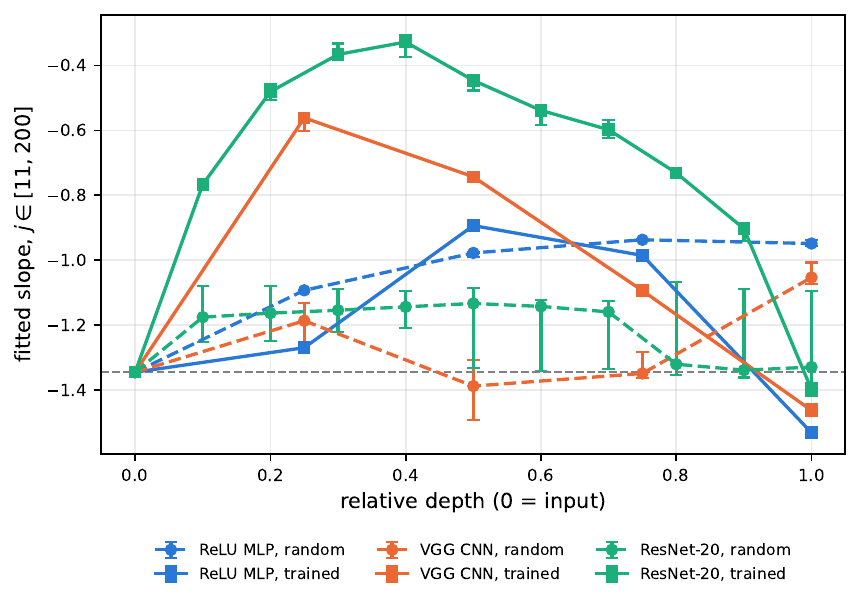}
  \caption{\ep{Fitted slope of the feature covariance spectrum ($11\le j\le 200$) against relative depth ($0$: input, $1$: last layer) on CIFAR-10, for a ReLU MLP, a VGG-style CNN and ResNet-20, at random initialization (dashed, circles) and after training (solid, squares). The random MLP is at the controlled scale. Median and range over seeds. Dashed grey: the input slope $-1.35$.}}
  \label{fig:across_layers}
\end{figure}

\begin{figure}[t]
  \centering
  \includegraphics[width=\textwidth]{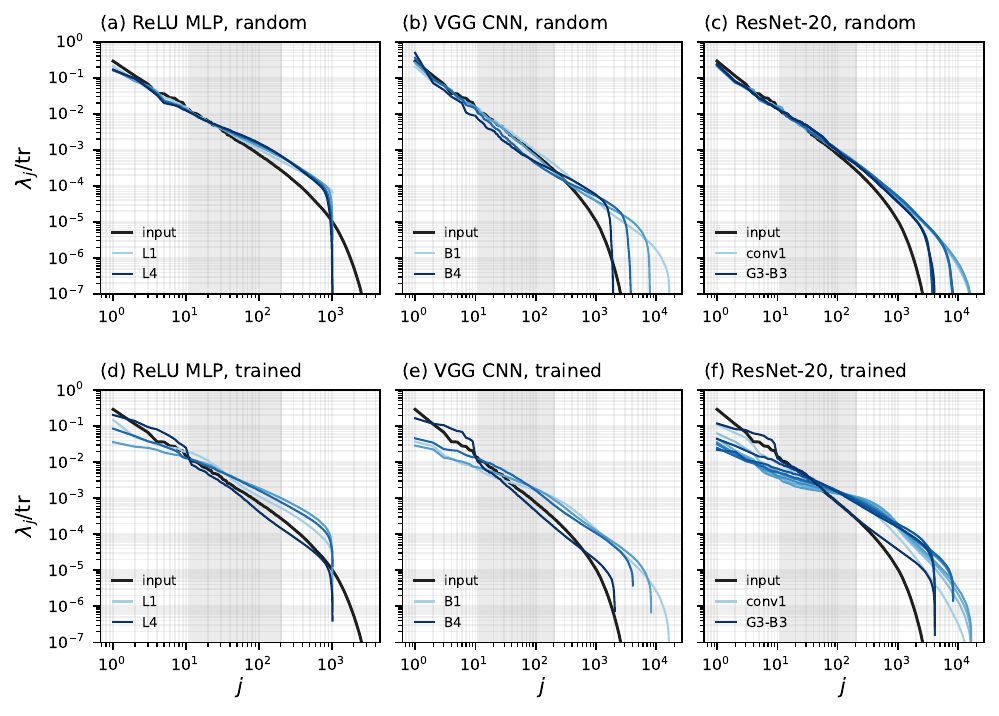}
  \caption{\ep{Covariance spectra of the features at every layer, normalized by the trace, on CIFAR-10: the input (black) and the layers, light to dark with depth (the first and last layers are labelled). Top row: random initialization (the MLP at the controlled scale); bottom row: after training. Median over seeds. Shaded: the fit window $11\le j\le 200$.}}
  \label{fig:layer_spectra}
\end{figure}

\begin{figure}[t]
  \centering
  \includegraphics[width=\textwidth]{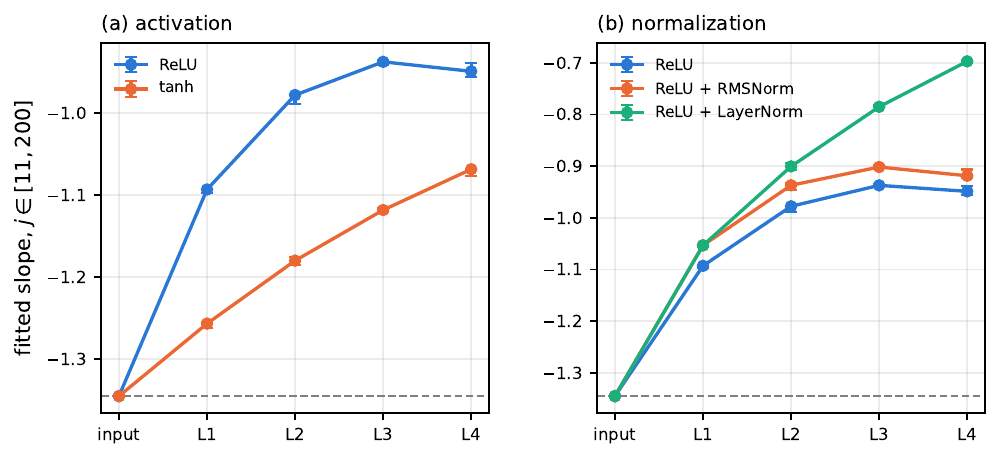}
  \caption{\ep{Random MLPs (depth $4$, width $1024$) at the controlled scale on CIFAR-10: fitted slope per layer ($11\le j\le 200$). \textbf{(a)} ReLU and $\tanh$. \textbf{(b)} ReLU without normalization, and with RMSNorm or LayerNorm after the activation. Median and range over $5$ seeds. Dashed grey: the input slope $-1.35$.}}
  \label{fig:mlp_activation}
\end{figure}

\clearpage
\section{\ep{Spectrum of random features when $\pwrco<1$}}\label{sec:alpha_lt_1}
\ep{
Theorem~\ref{theorem:MAIN} assumes $\pwrco>1$, which keeps $\tr\cov$ bounded as $v$ grows. For $\pwrco<1$, $\tr\cov$ grows like $v^{1-\pwrco}$. This section examines numerically what then happens to the spectrum of $\K_\W$. We take $\cov=\operatorname{diag}(j^{-\pwrco})$, not normalized, since the growth of the trace is the point; $\pwrco\in\{0.5,0.7\}$; monomials $f(y)=y^\p$; and $v=d^{1/\p+0.3}$, the dimension condition of Theorem~\ref{theorem:MAIN} with $\epsilon=0.3$, for $d$ up to $10^4$. As in \cref{fig:gaussian_monomials}, $\K_\W$ is computed exactly, over $20$ draws of $\W$, and compared with the population spectrum for $\p\le 3$.

\paragraph{The exponent}
The deterministic spectrum in Theorem~\ref{theorem:MAIN} is built from the products $\cov_{s_1}\cdots\cov_{s_\p}$, $s_1<\dots<s_\p$. Counting them does not use $\pwrco>1$, and for $\pwrco<1$ they still follow $(\log^{\p-1}(j+1)/j)^{\pwrco}$ (\cref{fig:alpha_principal}).

\paragraph{The staircase}
What changes is how the Wick components of $f$ are stacked. Expand $f$ at the scale $s=\m{w}^\top\cov\m{w}\approx\tr\cov$ of the pre-activations; for instance,
\[
y^3={:}y^3{:}+3s\,{:}y{:},\qquad y^4={:}y^4{:}+6s\,{:}y^2{:}+3s^2 .
\]
The component of degree $k$ then enters $\K_\W$ with a weight of order $(\tr\cov)^{\p-k}$.
\begin{itemize}[leftmargin=*,itemsep=2pt]
\item For $\pwrco>1$ these weights are bounded, and the components interleave into a single spectrum.
\item For $\pwrco<1$ they grow with $v$, and the lower-degree components lie above the others. The spectrum lists the lowest-degree component, about $\binom{v}{k}$ eigenvalues decaying like the products in \cref{fig:alpha_principal}; then it drops, and continues with the next component.
\end{itemize}
The steps sit near $j=v$ for $y^3$, where the linear component runs out, and near $j=v+\binom v2$ for $y^4$, where the quadratic one does (\cref{fig:staircase}). They are absent for $\pwrco=1.31$.

For $y^3$, the step separates the last linear eigenvalue, about $9(\tr\cov)^2\cov_v$, from the largest cubic one, about $36\,\cov_1\cov_2\cov_3$. Their ratio is of order $v^{2-3\pwrco}$, so the step grows with $v$ for $\pwrco<2/3$ and fades for $\pwrco>2/3$, as the exact population spectrum confirms. The finite matrix $\K_\W$ follows its population spectrum up to the step.

\paragraph{Relation to \cite{wortsman2025kernel}}
For inner-product kernels $\sum_j h_j\langle\x,\x'\rangle^j$ on Gaussian data with covariance eigenvalues proportional to $j^{-\pwrco}$, $\pwrco<1$, Wortsman and Loureiro show that the spectrum of the kernel operator splits into sectors, one per degree. Within each sector the eigenvalues decay like $(m-m_j)^{-\pwrco}$, where $m_j$ is the sector's first index, and the first few sectors are separated by spectral gaps \cite[Proposition~2]{wortsman2025kernel}. Our population operator is of this type. Writing $\m{z}=\cov^{1/2}\m{w}$, its kernel is
\[
\sum_k k!\,b_k^2\,\big(|\m{z}|^2|\m{z}'|^2\big)^{(\p-k)/2}\langle \m{z},\m{z}'\rangle^k,
\]
an inner-product kernel up to the norm factors, which concentrate when $\pwrco<1$. The steps above are the boundaries of their sectors, and the size of a step follows the same scaling. For $y^3$, whose Wick components have degrees $1$ and $3$, it is $v^{2-3\pwrco}$. The random-feature setting adds that the $d\times d$ matrix $\K_\W$ itself follows these sectors.

\paragraph{Conclusion}
Between consecutive steps, each Wick component decays with the exponent of the data, up to the logarithmic correction of its degree. In this sense the exponent is preserved for $\pwrco<1$ as well. What the single formula $(\log^{\p-1}(j+1)/j)^{\pwrco}$ no longer describes is how the components are stacked, which depends on $\tr\cov$ and therefore on $v$.
}

\begin{figure}[ht]
  \centering
  \includegraphics[width=\textwidth]{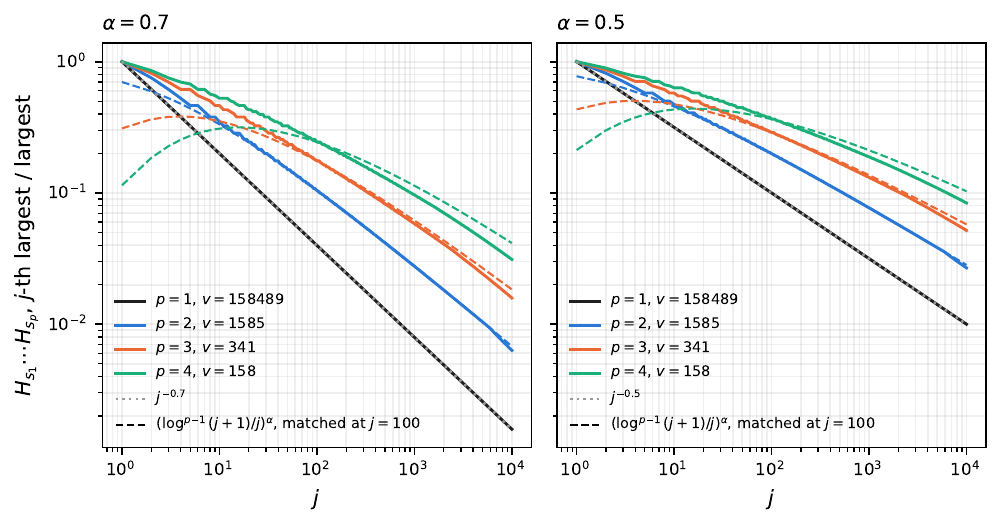}
  \caption{\ep{The products $\cov_{s_1}\cdots\cov_{s_\p}$, $s_1<\dots<s_\p\le v$, in decreasing order and divided by the largest, for $\cov=\operatorname{diag}(j^{-\pwrco})$ with $\pwrco=0.7$ (left) and $0.5$ (right), $\p=1,\dots,4$ and $v=d^{1/\p+0.3}$ with $d=10^4$ (solid). Dashed: $(\log^{\p-1}(j+1)/j)^{\pwrco}$, matched at $j=100$. Dotted grey: $j^{-\pwrco}$.}}
  \label{fig:alpha_principal}
\end{figure}

\begin{figure}[ht]
  \centering
  \includegraphics[width=\textwidth]{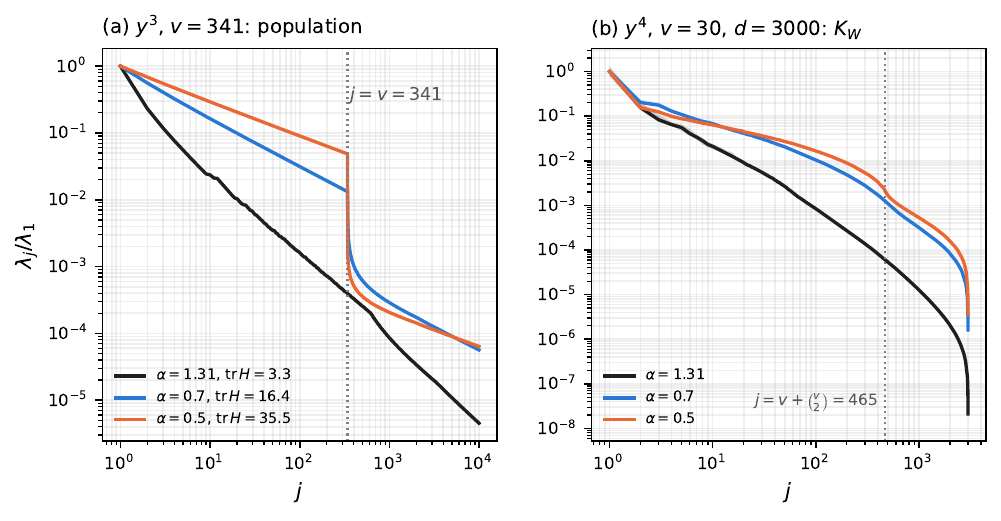}
  \caption{\ep{The staircase, for $\cov=\operatorname{diag}(j^{-\pwrco})$ with $\pwrco=1.31$ (black), $0.7$ (blue) and $0.5$ (orange); each spectrum is divided by its largest eigenvalue. \textbf{(a)} $f(y)=y^3$, $v=341$: the population spectrum (legend: $\tr\cov$). \textbf{(b)} $f(y)=y^4$, $v=30$, $d=3000$: the exact spectrum of $\K_\W$, median over $20$ draws of $\W$. Dotted lines: the indices where a Wick component runs out.}}
  \label{fig:staircase}
\end{figure}

\clearpage

\section{\ep{The floor of non-polynomial activations}}\label{sec:floor}
\ep{
This section supports the paragraph ``Beyond polynomial activations'' of the main text. The setup is that of \cref{fig:gaussian_monomials}: Gaussian data, $\lambda_j(\cov)\propto j^{-1.31}$, $\tr\cov=1$ (so the pre-activations have mean square $1$), $v=2d$, and 20 draws of $\W$.

Write $s_i=\m{w}_i^\top\cov\m{w}_i$ and $\rho_{ij}=\m{w}_i^\top\cov\m{w}_j/\sqrt{s_is_j}$, and expand $f$ in normalized Hermite polynomials,
\[
f(\sqrt{s}\,z)=\sum_{k\geq0}a_k(s)\,\frac{\mathrm{He}_k(z)}{\sqrt{k!}} .
\]
Mehler's formula gives
\[
\K_\W=\sum_{k\geq 0}\K^{(k)},\qquad \big(\K^{(k)}\big)_{ij}=\frac1d\,a_k(s_i)\,a_k(s_j)\,\rho_{ij}^k ,
\]
and $\E f(\m{w}_i^\top\x)^2=\sum_k a_k(s_i)^2$. Let
\[
E_{\geq 3}=\frac1d\sum_{i=1}^d\sum_{k\geq 3}a_k(s_i)^2
\]
be the part of $\E f^2$ carried by the Hermite degrees $k\geq 3$. Heuristically, the high degrees contribute to $\K_\W$ approximately a multiple of the identity, which puts a floor of height about $E_{\geq3}/d$ under the spectrum. This prediction uses only the Hermite coefficients of $f$; nothing is fitted.

For the Heaviside function and $y^2e^{-y^2}$, the degrees $k\geq3$ carry $18\%$ and $30\%$ of $\E f^2$, against $2\%$ for ReLU and $7\%$ for $\tanh$ (\cref{tab:floor}). For these two activations the predicted floor matches the end of the spectrum reasonably well: $d\,\lambda_{d/2}(\K_\W)$ is within $13\%$ of $E_{\geq3}$ at $d=10^3$ and within $1\%$ at $d=10^4$ (\cref{fig:floor_heaviside,fig:floor}). The floor accounts for the flat ends of these spectra, which are their main departure from power-law decay, and it falls like $1/d$. For ReLU and $\tanh$ it lies below the spectrum.

$\K_\W$ is computed in closed form for ReLU, the Heaviside function and $y^2e^{-y^2}$, and from Mehler's formula for $\tanh$; $\K^{(\leq2)}=\sum_{k\leq2}\K^{(k)}$ is computed from the first three terms of Mehler's formula, with Hermite coefficients computed by Gauss--Hermite quadrature.
}

\begin{table}[ht]
  \centering
  \caption{\ep{Mass of the Hermite degrees $k\geq3$ and the measured floor, in the setup of \cref{fig:gaussian_monomials}. $\E f^2$ and $E_{\geq3}$ are averaged over the features (median over 20 draws of $\W$ at $d=3000$). The last three columns give $d\,\lambda_{d/2}(\K_\W)/E_{\geq3}$; a value near $1$ means that the end of the spectrum sits on the predicted floor.}}
  \label{tab:floor}
  \begin{tabular}{lcccccc}
    \toprule
    & & & & \multicolumn{3}{c}{$d\,\lambda_{d/2}(\K_\W)/E_{\geq3}$} \\
    \cmidrule(l){5-7}
    activation & $\E f^2$ & $E_{\geq3}$ & $E_{\geq3}/\E f^2$ & $d=10^3$ & $d=3\cdot10^3$ & $d=10^4$ \\
    \midrule
    ReLU & $0.497$ & $0.0112$ & $2.3\%$ & $6.84$ & $5.60$ & $4.47$ \\
    $\tanh$ & $0.380$ & $0.0267$ & $7.0\%$ & $2.03$ & $1.64$ & $1.30$ \\
    Heaviside & $0.500$ & $0.0908$ & $18.2\%$ & $1.13$ & $1.06$ & $1.00$ \\
    $y^2e^{-y^2}$ & $0.0521$ & $0.0157$ & $30.2\%$ & $1.03$ & $1.01$ & $0.99$ \\
    \bottomrule
  \end{tabular}
\end{table}

\begin{figure}[ht]
  \centering
  \includegraphics[width=\textwidth]{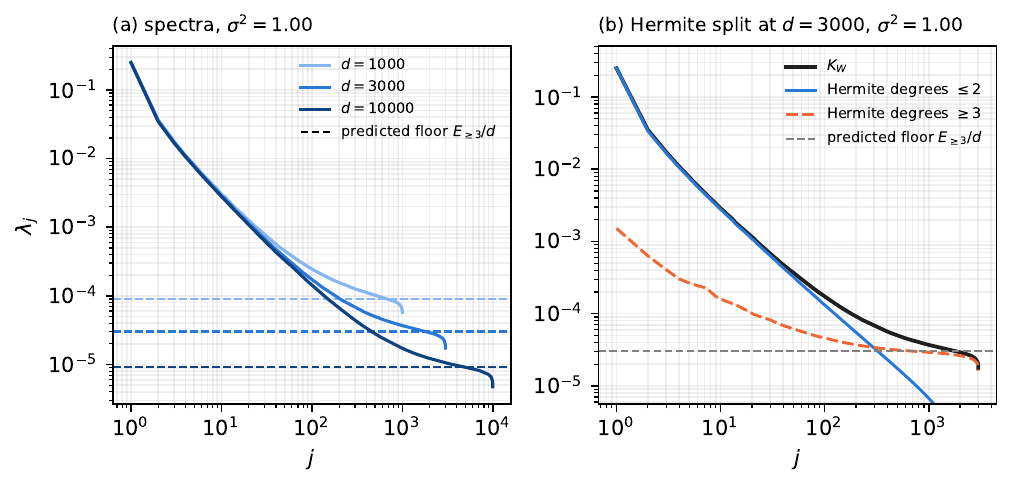}
  \caption{\ep{The floor of the Heaviside function, in the setup of \cref{fig:gaussian_monomials}. \textbf{(a)} $\lambda_j(\K_\W)$ for $d=10^3$, $3\cdot10^3$ and $10^4$, with the predicted floors $E_{\geq3}/d$ (dashed). \textbf{(b)} At $d=3000$: the spectra of $\K_\W$ and of its components of Hermite degree at most $2$ and at least $3$, with the predicted floor. Median and 10--90\% band over 20 draws of $\W$.}}
  \label{fig:floor_heaviside}
\end{figure}
\begin{figure}[ht]
  \centering
  \includegraphics[width=\textwidth]{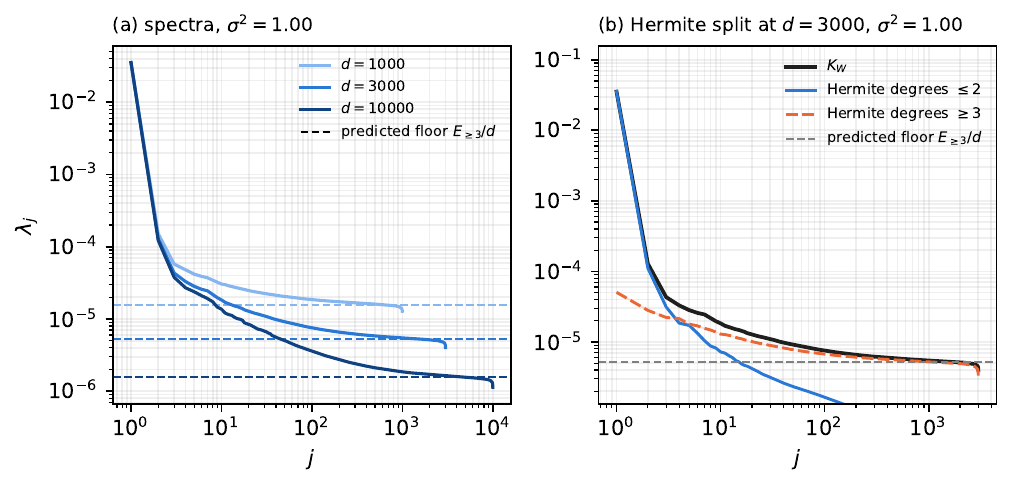}
  \caption{\ep{The floor of $f(y)=y^2e^{-y^2}$, as in \cref{fig:floor_heaviside}.}}
  \label{fig:floor}
\end{figure}

\end{document}